\documentclass[10pt,twocolumn,letterpaper]{article}

\usepackage{cvpr}              
\definecolor{cvprblue}{rgb}{0.21,0.49,0.74}
\usepackage[pagebackref,breaklinks,colorlinks,allcolors=cvprblue]{hyperref}

\usepackage{amsmath}
\usepackage{amsthm}
\newtheorem{proposition}{Proposition}
\newtheorem{corollary}{Corollary}[proposition] 

\usepackage[accsupp]{axessibility}

\usepackage{graphicx}
\usepackage{amsmath}
\usepackage{amssymb}
\usepackage{booktabs}
\usepackage{indentfirst}

\usepackage[capitalize]{cleveref}
\crefname{section}{Sec.}{Secs.}
\Crefname{section}{Section}{Sections}
\Crefname{table}{Table}{Tables}
\crefname{table}{Tab.}{Tabs.}
\usepackage{epsfig}
\usepackage{multirow}
\usepackage{makecell}
\usepackage{amsfonts}
\usepackage{enumitem}
\usepackage{array}
\usepackage{algorithm}
\usepackage{algpseudocode}
\usepackage[symbol]{footmisc}
\usepackage[accsupp]{axessibility}
\usepackage{blindtext}
\usepackage{graphicx}
\usepackage{mathtools}
\usepackage{arydshln}

\usepackage{tabularx}
\usepackage{adjustbox}
\usepackage{multirow}
\usepackage{booktabs}
\usepackage{enumitem}
\usepackage{float} 

\usepackage{colortbl}
\usepackage{amssymb}
\usepackage{pifont}
\newcommand{\cmark}{\ding{51}} 
\newcommand{\xmark}{\ding{55}} 

\definecolor{darkgreen}{RGB}{0,150,0}

\algrenewcommand\algorithmicindent{0.75em}
\def\httilde{\mbox{\tt\raisebox{-.5ex}{\symbol{126}}}}

\def\eqref#1{(\ref{eq:#1})}
\def\eqlabel#1{\label{eq:#1}}
\def\figref#1{\ref{fig:#1}}
\def\figlabel#1{\label{fig:#1}}
\def\pparagraph#1{\par{\bf #1}~~}

\def\x{{\mathbf x}}
\def\L{{\cal L}}

\def\NoNumber#1{{\def\alglinenumber##1{}\State #1}\addtocounter{ALG@line}{-1}}

\def\eqref#1{(\ref{eq:#1})}
\def\eqlabel#1{\label{eq:#1}}
\def\figref#1{\ref{fig:#1}}
\def\figlabel#1{\label{fig:#1}}
\def\pparagraph#1{\par{\bf #1}~~}
\def\httilde{\mbox{\tt\raisebox{-.5ex}{\symbol{126}}}}
\def\xcomment#1{\textcolor[rgb]{.3,.3,.1}{\text{$/\!\!/$ {\em #1}}}}
\def\comment#1{\kern-1cm\xcomment{#1}}
\def\eqcomment#1{\kern-1cm\xcomment{#1}}

\def\xx#1{\textcolor{red}{ #1}}

\def\m#1{\ensuremath{\mathtt{#1}}}
\def\mt#1{\ensuremath{\mathtt{\tilde{#1}}}}
\def\v#1{\ensuremath{\mathbf{#1}}}

\def\colspace#1{\mathrm{col}(#1)}
\def\nullspace#1{\mathrm{null}(#1)}
\def\localmin{\widehat{\min} }
\def\localargmin{{\arg\widehat{\min}}}

\DeclareMathOperator*{\argmin}{arg\,min}
\def\symmx#1{{\mathbb S}^{#1}}
\def\Rmx#1#2{{\mathbb R}^{{#1}\times{#2}}}

\def\real{\mathbb{R}}
\def\tr{^\top}
\def\trinv{^{-\top}}
\def\inv#1{#1^\mathsf{-1}}
\def\pinv#1{#1^\mathsf{\dagger}}
\def\mupinvof#1{{{#1}^{-\lambda}}}
\def\expm{\mathrm{expm}}
\def\logm{\mathrm{logm}}

\def\hadamard{\odot}
\def\kron{\otimes}
\def\vec{\operatorname{vec}}
\def\unvec{\operatorname{unvec}}

\def\sym{\operatorname{sym}}

\def\norm#1{\left\lVert#1\right\rVert}
\def\fro#1{\norm{#1}_F}
\def\l2#1{\norm{#1}_2}
\def\nuclear#1{\norm{#1}_*}

\def\err{f}
\def\erru{f^*}

\def\vDu{{\v\Delta \vu}}
\def\vDv{{\v\Delta \vv}}
\def\vDx{{\v\Delta \v x}}
\def\vDy{{\v\Delta \v y}}

\def\vx{\v x}
\def\vp{\v p}

\def\vhx{\hat{\v x}}
\def\vtx{\hat {\v x}}
\def\vy{\v y}
\def\vf{\v f}
\def\vp{\v p}
\def\vdx{\v {\Delta x}}
\def\vg{\v g}
\def\vgu{\vg^*}
\def\vgutr{\vg^{*\top}}
\def\vgup{\vgu_p}
\def\vgur{\vgu_r}
\def\vw{\v w}
\def\vm{\v m}
\def\vr{\v r}
\def\vru{\vr^*}
\def\vu{\v u}
\def\vup{\vu_p}
\def\vuperp{\vu_\perp}
\def\vuv{\v u^*}
\def\vv{\v v}
\def\vvu{\v v^*}
\def\hatvvu{{\hat{\v v}}^*}
\def\vx{\v x}
\def\vs{\v s}
\def\verr{{\boldsymbol\varepsilon}}
\def\verru{{\boldsymbol\varepsilon}^*}
\def\vdu{\partial \vu}
\def\vdv{\partial \vv}
\def\vdvu{\partial \vvu}
\def\dell{{\boldsymbol\delta}}

\def\set#1{\mathcal{#1}}

\def\mJu{\mJ_\vu}
\def\mJuu{\mJ_\vu^*}
\def\mJv{\m J_\vv}
\def\tmJv{\tilde{\m J}_\vv}
\def\mPv{\m P_\vv}
\def\mQv{\m Q_\vv}

\def\mA {\m A}
\def\tildemJu{\tilde\mJ^*}
\def\mJutr{\mJ^{*\top}}
\def\mU{\m U}
\def\mUperp{\mU_{\perp}}
\def\mUp{\mU_p}
\def\mUr{\mU_r}
\def\mUu{\mU^*}
\def\mVu{\mV^*}
\def\mVutr{\mV^{*\top}}
\def\mW{\m W}
\def\mM{\m M}
\def\mS{\m S}
\def\dvdu{\frac{d\vvu}{d\vu}}
\def\mdU{\partial\mU}
\def\mdV{\partial\mV}
\def\mdVu{\partial\mVu}

\def\cU{{\cal U}}
\def\cV{{\cal V}}

\def\twiddle#1{{\tilde{#1}}}
\newcommand{\centered}[1]{\begin{tabular}{l} #1 \end{tabular}}

\def\paperID{3334} 
\def\confName{CVPR}
\def\confYear{2026}

\title{Revisiting Geometric Obfuscation with Dual Convergent Lines for Privacy-Preserving Image Queries in Visual Localization}

\author{Jeonggon Kim$^1$ \quad Heejoon Moon$^{2}$ \quad Je Hyeong Hong$^{1,2}$\footnotemark[2]\\
$^1$Dept. Electronic Engineering, Hanyang University$ \quad ^2$Dept. Artificial Intelligence, Hanyang University}

\begin{document}
\maketitle


\begin{abstract} 
Privacy-Preserving Image Queries (PPIQ) are an emerging mechanism for cloud-based visual localization, enabling pose estimation from obfuscated features instead of private images or raw keypoints.
However, the main approaches for PPIQ, primarily geometry-based and segmentation-based obfuscation, both suffer from vulnerabilities to recent privacy attacks.
In particular, a fundamental limitation of geometry-based obfuscation is that the spatial distribution of obfuscated neighboring lines still effectively surrounds the original keypoint location, providing exploitable cues for recovering the original points.
We revisit this geometric paradigm and introduce Dual Convergent Lines (DCL), a novel keypoint obfuscation method demonstrating strong resilience against such attack.
DCL places two fixed anchors on a central partition line and lifts each keypoint to a line originating from one of them, with the active anchor determined by the keypoint's location.
This arrangement invalidates the geometry-recovery attack~\cite{chelani2024obfuscation} by making its optimization ill-posed:
Neighboring lines either misleadingly converge to one anchor, yielding a trivial solution, or become near-parallel at the partition boundary, yielding an unstable high-variance solution. Both outcomes thwart point recovery.
DCL is also compatible with an existing line-based solver, enabling deployment in traditional localization pipelines.
Experiments on both indoor and large-scale outdoor datasets demonstrate DCL's robustness against privacy attacks, efficiency, and scalability, while achieving practical localization performance. 
\end{abstract}
\footnotetext[2]{Corresponding author}
\footnotetext{E-mail: \{drgon22, wilko97, jhh37\}$@$hanyang.ac.kr}


\vspace{-6mm}
\section{Introduction}
\label{sec:intro}

Visual Localization (VL) is the task of estimating the 6-DoF (Degrees of Freedom) camera pose for a given query image and is a key component for applications, such as Augmented Reality (AR)~\cite{azure_spatial_anchors, google_vps, niantic_vps_docs}, autonomous driving, and robotics~\cite{lynen2015get, mur2015orb, davison2007monoslam}. 
Modern VL systems widely employ a cloud-based architecture~\cite{azure_spatial_anchors, google_vps, niantic_vps_docs}, where a client device sends an ``image query", either the raw image or a set of extracted local features to the server for localization.
This benefits the client by offloading both the storage of the large-scale 3D map and the computation process for the pose estimation.
However, privacy concerns have been raised in cloud-based systems, where man-in-the-middle attackers can intercept private image queries during transmission.
While sending a raw image is obviously prone to the privacy attack, Pittaluga~\etal~\cite{pittaluga2019revealing} demonstrated that even local features can be inverted to reconstruct high-fidelity RGB images, thereby highlighting the privacy concerns known as an \emph{``Inversion attack"}.
    
\begin{figure}[t]
    \center{
    \includegraphics[width=0.90\linewidth]{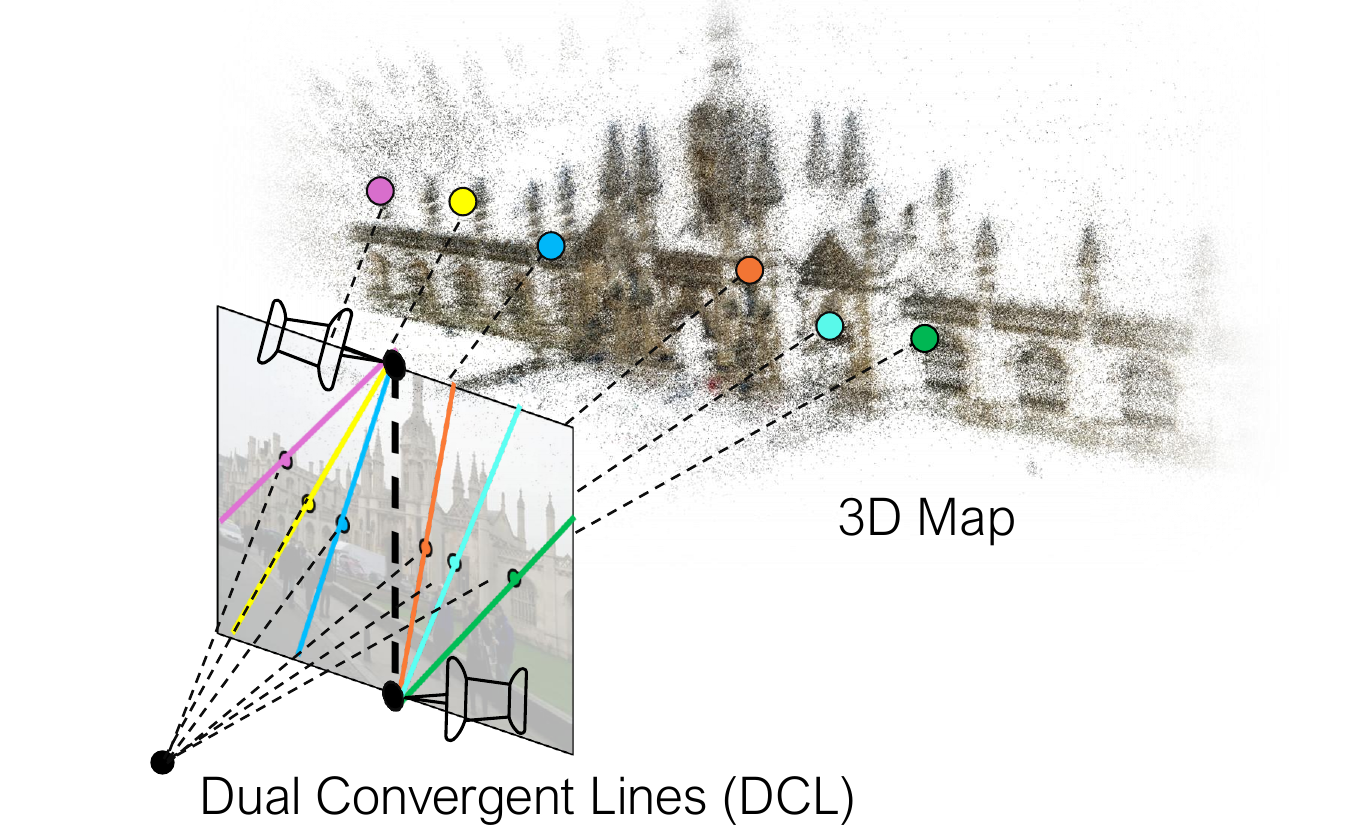}\\
    \scriptsize{~~~~~Original~~~~~~~~~~~~~~~~~~~~~~~~~~~~~~~~Recovered via privacy attack~\cite{chelani2024obfuscation, pittaluga2019revealing}~~~~}\\
    \includegraphics[width=0.32\linewidth]{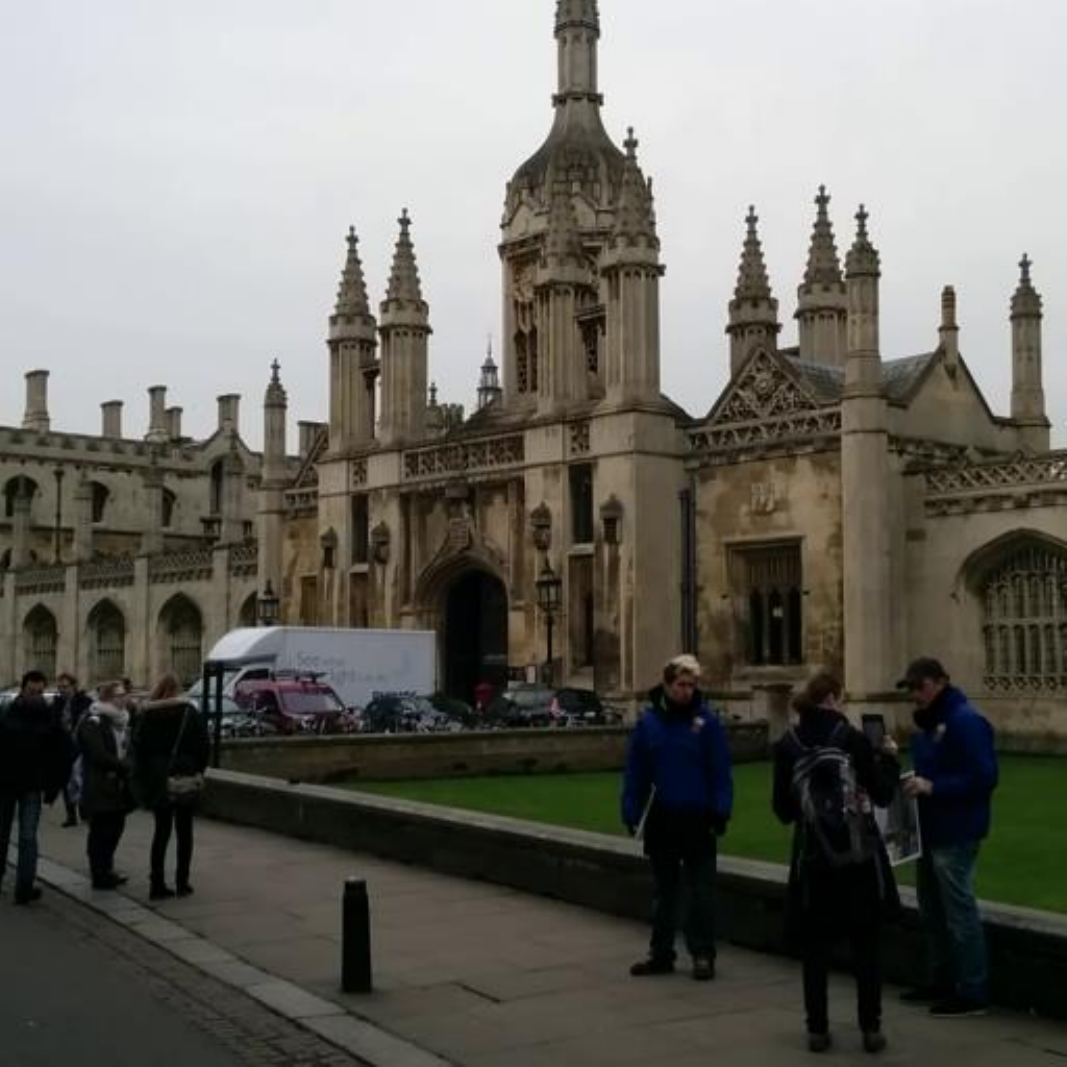}
    ~~
    \includegraphics[width=0.32\linewidth]{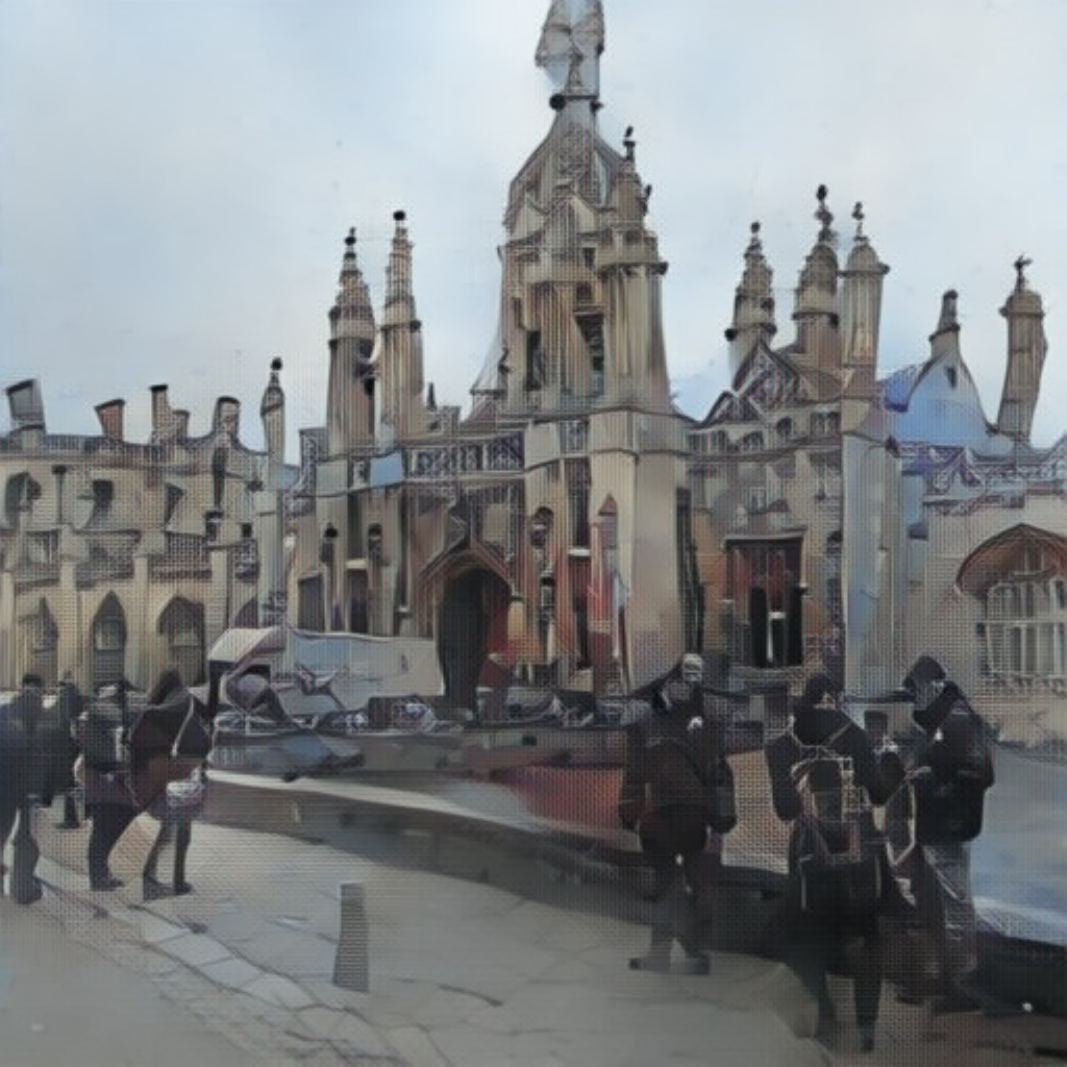}
    \includegraphics[width=0.32\linewidth]{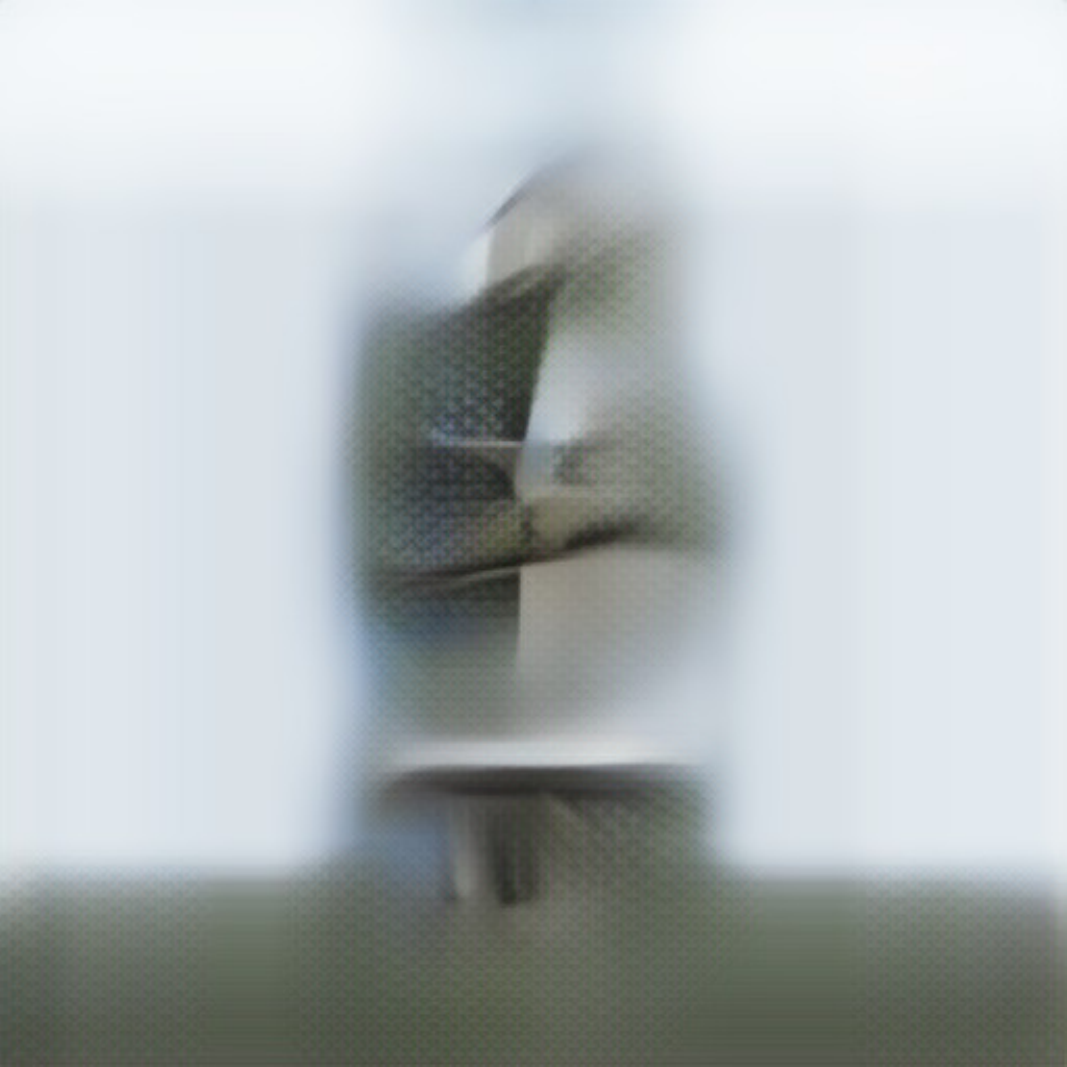}\\ \scriptsize{~~~~~~~~~~~~~~~~~~~~~~~~~~~~~~~~~~~~~~~~Random Lines~\cite{speciale2019queries}~~~~~~~~~~~~~~~\textbf{DCL (ours)}}
    \\}
    \vspace{-2mm}
    \caption{
    Existing privacy-preserving schemes such as Random Lines~\cite{speciale2019queries}, are shown to be prone to a recent attack that approximately recovers the keypoint positions~\cite{chelani2024obfuscation}, leading to successful image recovery. In contrast, our \emph{Dual Convergent Lines} show robustness to the attack and allow practical localization.
    }
    \figlabel{teaser}
    \vspace{-2mm}
\end{figure}

This concern spurred the research community to develop methods to ensure the privacy preservation of image queries against inversion attack while enabling localization, known as \emph{Privacy-Preserving Image Queries (PPIQ)}.
Out of many PPIQ strategies~\cite{speciale2019queries, dusmanu2021privacy, pittaluga2023ldp, ng2022ninjadesc, pan2023permut, pietrantoni2023segloc, pietrantoni2025gaussian}, two main paradigms have been remarked for allowing accurate localization: i) geometry obfuscation-based and ii) segmentation-based approaches.
Geometry obfuscation-based methods aim to obfuscate the 2D keypoint locations (\eg 2D random lines~\cite{speciale2019queries}, coordinate permutation~\cite{pan2023permut}) with minimal obfuscation overhead in the client side. 
On the other hand, segmentation-based methods~\cite{pietrantoni2023segloc, pietrantoni2025gaussian} use semantic maps as feature descriptors, which contain fewer visual details and thereby mitigate the risk of direct inversion attacks. However, their drawbacks include slow localization and incompatibility with traditional pipelines~\cite{schoenberger2016sfm, sarlin2019coarse}.
        
While both representations were considered as privacy-preserving, recent studies~\cite{chelani2023privacy, anonymous2025vulnerability} have demonstrated the limitations in each approach and pose privacy risks again.
Chelani~\etal~\cite{chelani2024obfuscation} showed 
that keypoints can be recovered by identifying the neighboring obfuscations through a learning-based search over their descriptors~\cite{lowe2004sift, detone2018superpoint}, followed by optimization that minimizes the distance between neighboring obfuscations and the estimated point location.
Very recently, \cite{anonymous2025vulnerability} has shown that the diffusion model is capable of reconstructing detailed images even from the high-level semantic representations.
Consequently, guaranteeing privacy in image queries remains an open question again.

In this work, considering the efficiency in localization and minimal obfuscation overhead on the client-side, we revisit the geometric obfuscation approach and introduce a novel line obfuscation method called \emph{Dual Convergent Lines (DCL)}.
Specifically, DCL constructs obfuscated lines based on a spatial partitioning of the image relative to two fixed anchors.
This specific arrangement invalidates the core assumption of the geometry-recovery attack~\cite{chelani2024obfuscation} 
and leads to two distinct failure results:
1) If neighboring lines originate from the same anchor, the recovered points collapse to a trivial solution at that anchor, and
\noindent 2) as shown in Fig.~\ref{fig:main_pipeline}, for keypoints located near the boundary region, neighboring lines originating from the opposing anchor become near-parallel. 
This geometry yields an unstable, high-variance solution, far from the original location.
Beyond its robustness against privacy attacks, DCL can be efficiently integrated into conventional localization pipelines, where the line-based solver~\cite{speciale2019queries} offers favorable runtime and accuracy.
To summarize, our key contributions are as follows:
\begin{itemize}
    \item We introduce \textbf{Dual Convergent Lines (DCL)}, a novel line obfuscation method that, to the best of our knowledge, is the only method resilient to the privacy attack~\cite{chelani2024obfuscation}.
    \item We provide a theoretical analysis of DCL’s robustness, showing how its dual-anchor based line lifting induces ill-posed solutions in geometry-recovery attacks, causing either trivial collapse or unstable, high-variance solutions.
\end{itemize}
Finally, experimental results on both indoor~\cite{shotton2013scene} and two large-scale outdoor~\cite{kendall2015posenet, sattler2018aachen} datasets demonstrate the DCL's robustness against the advanced privacy attack~\cite{chelani2024obfuscation} while achieving practical localization performance.

\begin{table}[t]
\centering
\begin{tabular}{c}
    \includegraphics[width=0.97\linewidth]{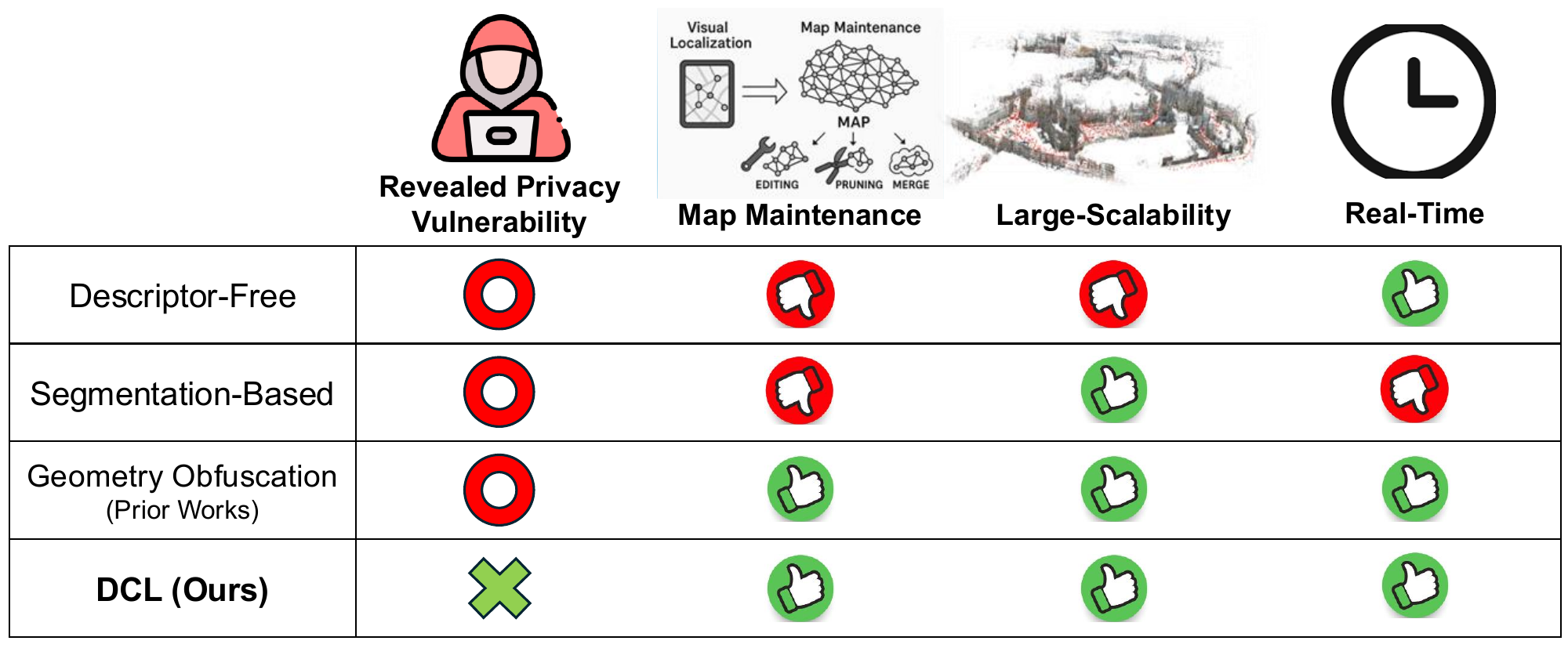}
\end{tabular}
\vspace{-4mm}
\caption{
Comparison of privacy-preserving image queries methods.
Recently, elevated privacy attacks~\cite{chelani2024obfuscation, anonymous2025vulnerability, pittaluga2019revealing} have highlighted the vulnerabilities of existing schemes.
In contrast, our method offers advantages in privacy preservation, map maintenance, scalability, and pose estimation runtime over other schemes.
}
\vspace{-4mm}
\label{tab:ppiq-comparison}
\end{table}

\section{Related Work}
\label{sec:related_work}

\paragraph{Privacy threats in visual localization.}
It was Pittaluga \etal\cite{pittaluga2019revealing}, who  first demonstrated that high-fidelity scene images can be recovered from sparse keypoints and their descriptors, raising real privacy concern in the research community.
Their work (\emph{InvSfM}) employs a cascaded U-Net~\cite{ronneberger2015u} to generate RGB reconstructions of the scene using only 2D keypoints or projected 3D point cloud and respective descriptors.
This soon prompted research efforts focused on obfuscating either image queries~\cite{speciale2019queries, pan2023permut, dusmanu2021privacy,ng2022ninjadesc,pittaluga2023ldp} or the 3D points maps~\cite{speciale2019privacy, geppert2021privacy, geppert2022privacy, lee2023paired, moon2024raycloud, moon2024sphere, pietrantoni2023segloc, pietrantoni2025gaussian} to prevent sensitive information from being revealed without overly sacrificing localization performance.

\vspace{-4mm}
\paragraph{Approaches for privacy-preserving image queries.}
Privacy-Preserving Image Queries (PPIQ) can be largely divided into two categories of methods: i) methods obfuscating the 2D keypoint locations~\cite{speciale2019queries, pan2023permut}, and  ii) those obfuscating the keypoint descriptors~\cite{dusmanu2021privacy,ng2022ninjadesc,pittaluga2023ldp, pietrantoni2023segloc, pietrantoni2025gaussian, zhou2022geometry, wang2024dgc}.

\begin{figure*}[t]
\centering
\includegraphics[width=0.99\linewidth]{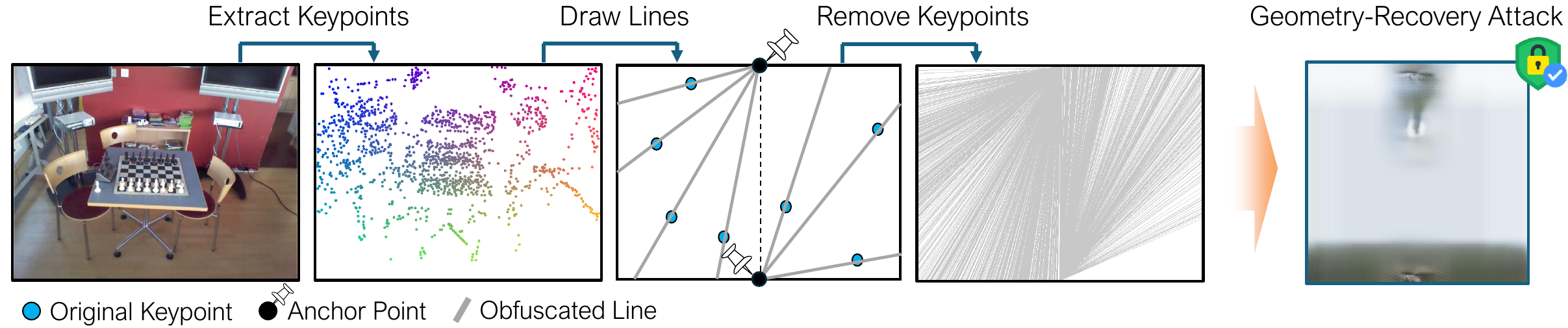}
\vspace{-2mm}
\caption{
Illustration of our DCL construction procedure. We begin by extracting keypoints~\cite{lowe2004sift, detone2018superpoint} from the query image. Next, lines are generated by connecting each keypoint to one of the two predefined anchor points located at the top and bottom of the central midline. Finally, the positions of the original keypoints are removed. We show in Sec.~\ref{sec:robustness_analysis} that DCL effectively prevents the geometry-recovery attack \cite{chelani2024obfuscation} by encouraging erroneous recovery of keypoint locations, subsequently successfully thwarting the inversion attack \cite{pittaluga2019revealing}.
}
\label{fig:main_pipeline}
\vspace{-4mm}
\end{figure*}

i) \emph{Geometric obfuscation}: Speciale~\etal~\cite{speciale2019privacy} first introduced geometric obfuscation to conceal 2D keypoint locations, where each point is replaced by a randomly oriented line passing through it. 
They also proposed the \texttt{l6p} solver~\cite{speciale2019queries}, leveraging line-point distances for robust pose estimation. 
Pan~\etal~\cite{pan2023permut} introduced a coordinate permutation scheme, randomly pairing keypoints and permuting one of their components. 
While geometric approaches offer accurate pose estimation and flexibility with conventional localization pipelines, they have been shown to be vulnerable to a neighborhood-based geometry-recovery attack~\cite{chelani2024obfuscation}. 

ii) \emph{Descriptor obfuscation}: 
Approaches for obfuscating descriptors include affine-subspace lifting~\cite{dusmanu2021privacy}, adversarial descriptor learning~\cite{ng2022ninjadesc} and local differential privacy protocol~\cite{pittaluga2023ldp}, but these suffer from reduced localization accuracy.
A more recent line of research is to use semantic maps for extracting feature descriptors~\cite{pietrantoni2023segloc, pietrantoni2025gaussian}, which is based on the intuition that they contain less private information compared to original features derived from RGB query images.
Unfortunately, this type of approaches is shown to be vulnerable to a new type of privacy attack based on diffusion models~\cite{anonymous2025vulnerability}, raising privacy concerns once more.
Other related research includes ``descriptor-free" approaches~\cite{zhou2022geometry, wang2024dgc}, which proposes localization schemes just involving geometric correspondences without image-based descriptors~\cite{lowe2004sift, detone2018superpoint}, but they show poor localization accuracy compared to traditional pipelines.
Most importantly, descriptor obfuscation-based approaches commonly do not conceal keypoint positions, leaving the 2D keypoint layout vulnerable to an inversion attack~\cite{pittaluga2019learning}.
Also, these methods require the same obfuscation on the server-side, which can add additional (repetitive) burden whenever the map is updated for maintenance.

In this paper, among many strategies for PPIQ, we revisit the geometric obfuscation scheme due to its superiority in map maintenance, scalability to large scenes, and favorable localization runtime, as shown in Table~\ref{tab:ppiq-comparison}. 
While RayCloud~\cite{moon2024raycloud} is closely related to our study in its use of dual-anchor line obfuscation, \cite{moon2024raycloud} determines anchor positions via $K$-means clustering and assigns 3D points to anchors randomly, thereby neighboring lines remain close to original 3D points and remain susceptible to~\cite{chelani2024obfuscation}.
In contrast, DCL's strategically optimized anchor positions induce ill-posed optimization, thereby effectively thwarting~\cite{chelani2024obfuscation}.

    
\section{Preliminaries}
\label{sec:preliminaries}

\paragraph{Geometry obfuscation via 2D line lifting.} 
To prevent the direct image inversion~\cite{pittaluga2019revealing} from the keypoints in the image query, line-based obfuscation~\cite{speciale2019queries} conceals their positions by replacing each 2D keypoint $\mathbf{x} \in \mathbb{R}^2$ with a randomly oriented 2D line $\mathbf{l} \in \mathbb{P}^2$ passing through the point.
Note, a point $\hat{\v x}_i \in \real^2$ on the obfuscated line $\mathbf{l}_i$ can be parameterized as $\mathbf{\hat x}_i = \v x_i + t_i \v v_i$, where $t_i \in \real$ is the offset along the line and $\v v_i\in\real^2$ is the direction vector.

\vspace{-3mm}
\paragraph{Camera pose estimation with lifted 2D lines.}
Given 2D line-3D point correspondences via descriptor matching, the underlying camera pose can be estimated by employing the $\texttt{l6P}$ minimal solver~\cite{kukelova2016efficient, speciale2019queries} equipped with RANSAC~\cite{fischler1981random}. 

The \texttt{l6P} solver estimates the 6-DoF camera pose $(\mathbf{R} \in SO(3), \mathbf{t} \in \mathbb{R}^3$) by using the geometric constraint \cite{ramalingam2010p, ramalingam2013theory}:
\vspace{-2mm}
\begin{equation}
    \mathbf{n}_i^\top (\mathbf{R} \mathbf{X}_i + \mathbf{t}) = 0,
    \eqlabel{constraint}
\end{equation}
where $\v n_i\in\real^3$ is the normal vector of the 3D plane $\Pi_i$ created by back-projecting the line $\v l_i$ (from the camera center), and $\v X_i\in\real^3$ is the respective 3D point.
Intuitively, the above constraint dictates that the normal of the unprojected plane ($\Pi_i$) must be orthogonal to the 3D point $\mathbf{X}_i$ in camera-centered coordinates.
Since each 2D line-3D point correspondence yields only one constraint, at least six independent correspondences are required for pose estimation. 


\vspace{-3mm}
\paragraph{Geometry-recovery attack.} 
Chelani~\etal~\cite{chelani2024obfuscation} proposed to recover the original keypoint $\v x$ from its obfuscated representation $\mathcal O(\v x)$. The core assumption of this attack is that obfuscated neighbors still remain geometrically close to the original point and recovered by finding $\mathbf{x}^*$ that minimizes the sum of squared distances to its neighbors as follows:
\vspace{-3mm}
\begin{align}
\mathbf{x}^* &= \argmin_{\hat {\v x} \in \mathcal O(\v x)} \sum_{\mathcal O(\v x_j) \in \mathcal{N}(\mathcal O(\v x))} d(\mathcal O(\v x_j), \hat{\v x})^2,
\eqlabel{new_chelani}
\end{align}
where $\mathcal{N}(\mathcal O(\v x))$ is the set of $K$ obfuscated neighbors of $\v x$ and $d(\cdot,\cdot)$ is the shortest Euclidean distance between the two inputs.
Since the true neighborhood $\mathcal{N}(\mathcal O(\v x))$ is inaccessible, \cite{chelani2024obfuscation} approximates it using an attention-based network that learns co-occurrence patterns from the local descriptors~\cite{lowe2004sift, detone2018superpoint}.
They also employed RANSAC-like loops to filter false-positive neighborhood predictions.
    

For our case of line-based geometric obfuscation, $\mathcal O(\v x_i)$ directly turns into $\mathbf{l}_i$ and the attack~\cite{chelani2024obfuscation} estimates the keypoint location ($\mathbf{x}_i^*$) along line $\mathbf{l}_i$ that minimizes the sum of squared distances to its neighborhood lines $\{\v l_j\}$.
In terms of equation, we seek to find $\mathbf{x}_i^* = \argmin_{\mathbf{\hat x}_i \in \mathbf{l}_i} f(\mathbf{\hat{x}}_i)$, where
\begin{equation}
    f(\mathbf{\hat{x}}_i) = \sum_{\mathbf{l}_j \in  \mathcal{N}(\mathbf{l}_i)} d(\mathbf{l}_j, \mathbf{\hat{x}}_i)^2
    \eqlabel{cost_function}
.
\end{equation}
If the neighboring lines are uniformly distributed as in previous work~\cite{speciale2019queries},
the recovered point $\mathbf{x}_i^*$ converges to the  true point $\mathbf{x}_i$ as more lines surround the original point. 
In contrast, our method (DCL) avoids this attack by essentially deviating from uniformly distributed lines, preventing the recovered point $\mathbf{x}_i^*$  from converging to the true point $\mathbf{x}_i$.


\section{Proposed Method}
\label{sec:method}
We first introduce the motivation of DCL and its construction process in Sec.~\ref{sec:dcl_definition}, followed by an analysis of its resilience against the geometry-recovery attack in Sec.~\ref{sec:robustness_analysis}, and finally, the localization methodology in Sec.~\ref{sec:localization}.

\subsection{Dual Convergent Lines}
\label{sec:dcl_definition}

\paragraph{Motivation.}
Our key insight for preventing the geometry-recovery attack \cite{chelani2024obfuscation} is to deliberately induce ill-posed optimization, ensuring the recovered point deviates significantly from its true location.
The simplest way to achieve this is by making all obfuscated lines intersect at a single point which ensures the solution of Eq.~\eqref{cost_function} collapse to that intersection.
Unfortunately, this introduces a degeneracy for pose estimation (details in Sec.~\ref{sec:localization}), and
consequently, motivates the design of our Dual Convergent Lines, maintaining resilience to privacy attacks while enabling localization.
   
\begin{figure}[t]
\centering
    \vspace{-2mm}
    \subfloat{
    \includegraphics[width=1.0\linewidth]{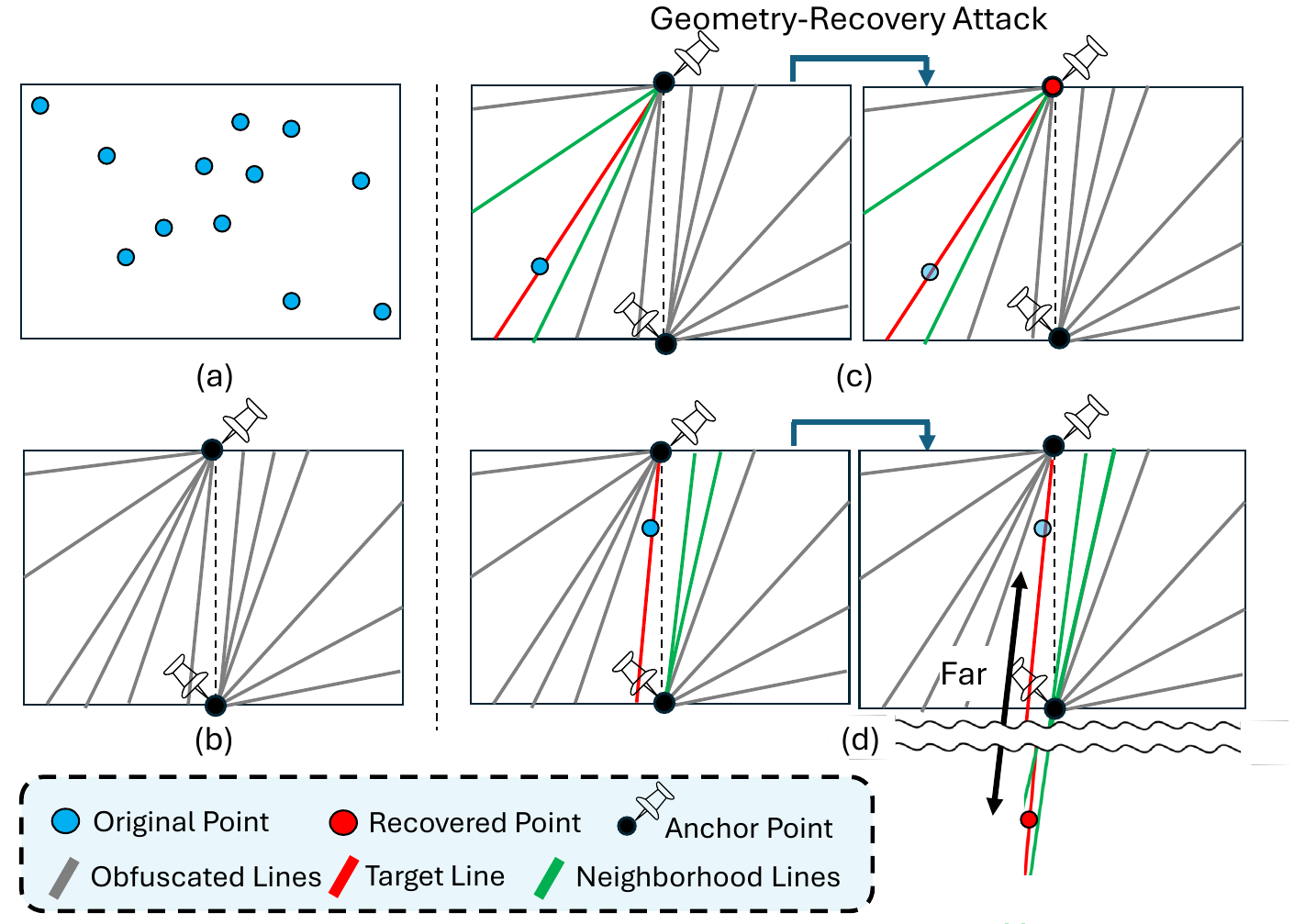}
    }

    \vspace{-2mm}
    \caption{
    Visualization of DCL's two failure modes against the geometry-recovery attack. (a) Original keypoints. (b) Our DCL obfuscation of the keypoints in (a). (c) Mode 1 (intra-anchor convergence): Keypoints trivially converge to the anchor location. (d) Mode 2 (inter-anchor instability): Near-parallel geometry forces optimization to high variance instability.
    }
    \figlabel{attack_scenario}
    \vspace{-4mm}
    \label{fig:invMode}
\end{figure}

\vspace{-4mm}
\paragraph{Construction procedure.}
For the detailed construction of DCL, we first partition the image space into two disjoint vertical regions, $\mathcal{R}_1$ and $\mathcal{R}_2$, separated by a central dividing line: $\mathcal{R}_1$$=$$\{(u,v)\mid0<u<W/2\}$ and $\mathcal{R}_2$$=$$\{ (u, v) \mid W/2 < u < W \}$. Note that while we use a vertical partition, the scheme is applicable to any dividing line (\eg, horizontal or diagonal).
We then define two fixed anchor points ($\mathbf{a}_j \in \real^2, j = 1,2$) as the top and bottom of the image center given an image with width $W$ and height $H$: $\mathbf{a}_1$$=$$(W/2, 0)$ and $\mathbf{a}_2$$=$$(W/2, H)$.
Each region is then assigned one of these anchors.
After setting the anchors, we extract keypoints from the query image and obfuscate each keypoint by connecting it with the assigned anchor, as shown in Fig.~\ref{fig:main_pipeline}. Specifically, for any given keypoint $\mathbf{x}_i = (u_i, v_i)$, its obfuscated line $\mathbf{l}_i$ is generated as:
\begin{align}
    \mathbf{l}_i = \begin{cases}
    \text{line}(\mathbf{x}_i, \mathbf{a}_1) & \text{if } \mathbf{x}_i \in \mathcal{R}_1 \\
    \text{line}(\mathbf{x}_i, \mathbf{a}_2) & \text{if } \mathbf{x}_i \in \mathcal{R}_2
    \end{cases}
\end{align}
This particular line distribution ensures that neighboring lines either intersect at the same anchor point or become parallel (near the boundary). 
        
\subsection{Robustness to Point-Recovery Attacks}
\label{sec:robustness_analysis}

We now analyze DCL's robustness against the known geometry-recovery attack~\cite{chelani2024obfuscation}, which uses descriptor-based neighbor identification followed by geometric optimization. Our DCL framework is designed to make this optimization fundamentally ill-posed. We show that even with perfect neighbor identification, the Dual Convergent Line distribution induces two distinct and unavoidable failure modes.

\begin{figure}[t]
\centering
    \setlength{\tabcolsep}{0.1pt}
    
\centered{\includegraphics[width=1.0\linewidth]
    {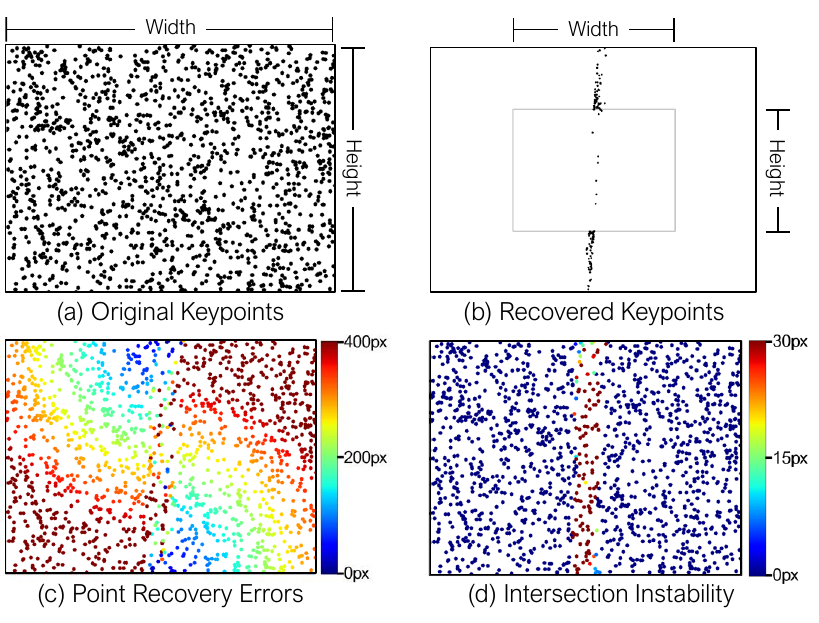}}
    \hfil
    \vspace{-4mm}
    \caption{
        Overview of our DCL method and its robustness in our synthetic environment. 
        We generated 1,500 uniformly distributed keypoints within a 640x480 resolution.
        (b) shows the recovered keypoints, while the error map (c) confirms the attack's overall failure, with only 8 keypoints having error below 30 pixels.
        (d) visualizes the standard deviation of the estimated parameter $t_{i,j}$ in Eq.~\eqref{t_star_weighted}, showing high instability at the partition boundary.
        }
    \figlabel{2d_rep}
    \vspace{-3mm}
    \label{fig:DCLsynthesis}
\end{figure}

\begin{figure*}[t]
\centering
    \subfloat[][Original Keypoints~\cite{detone2018superpoint}]{
    \setlength{\fboxsep}{0pt}
    \fbox{\includegraphics[width=0.23\linewidth]{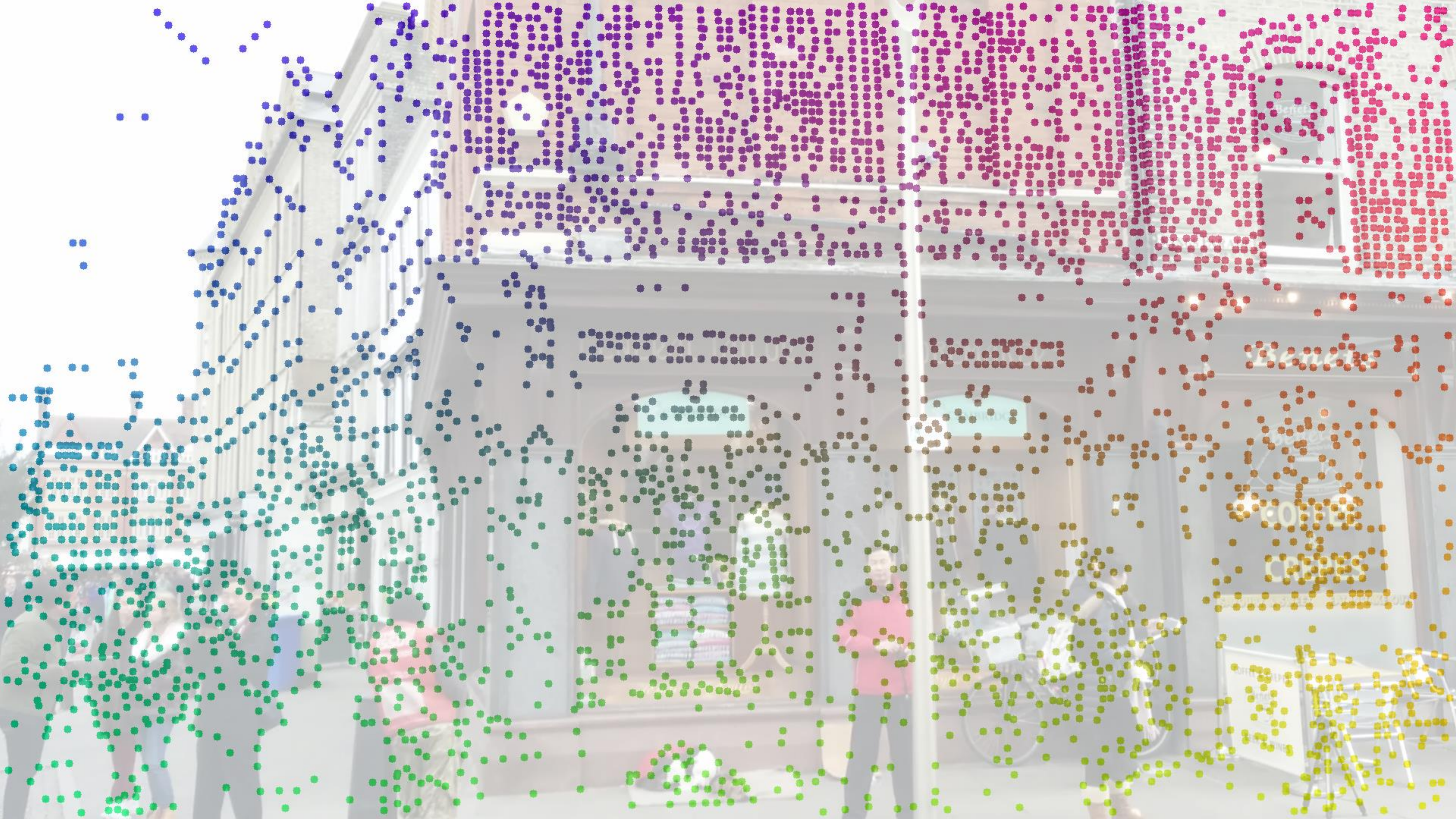}}
    }
    \subfloat[][Random Lines~\cite{speciale2019queries}]{
    \setlength{\fboxsep}{0pt}
    \fbox{\includegraphics[width=0.23\linewidth]{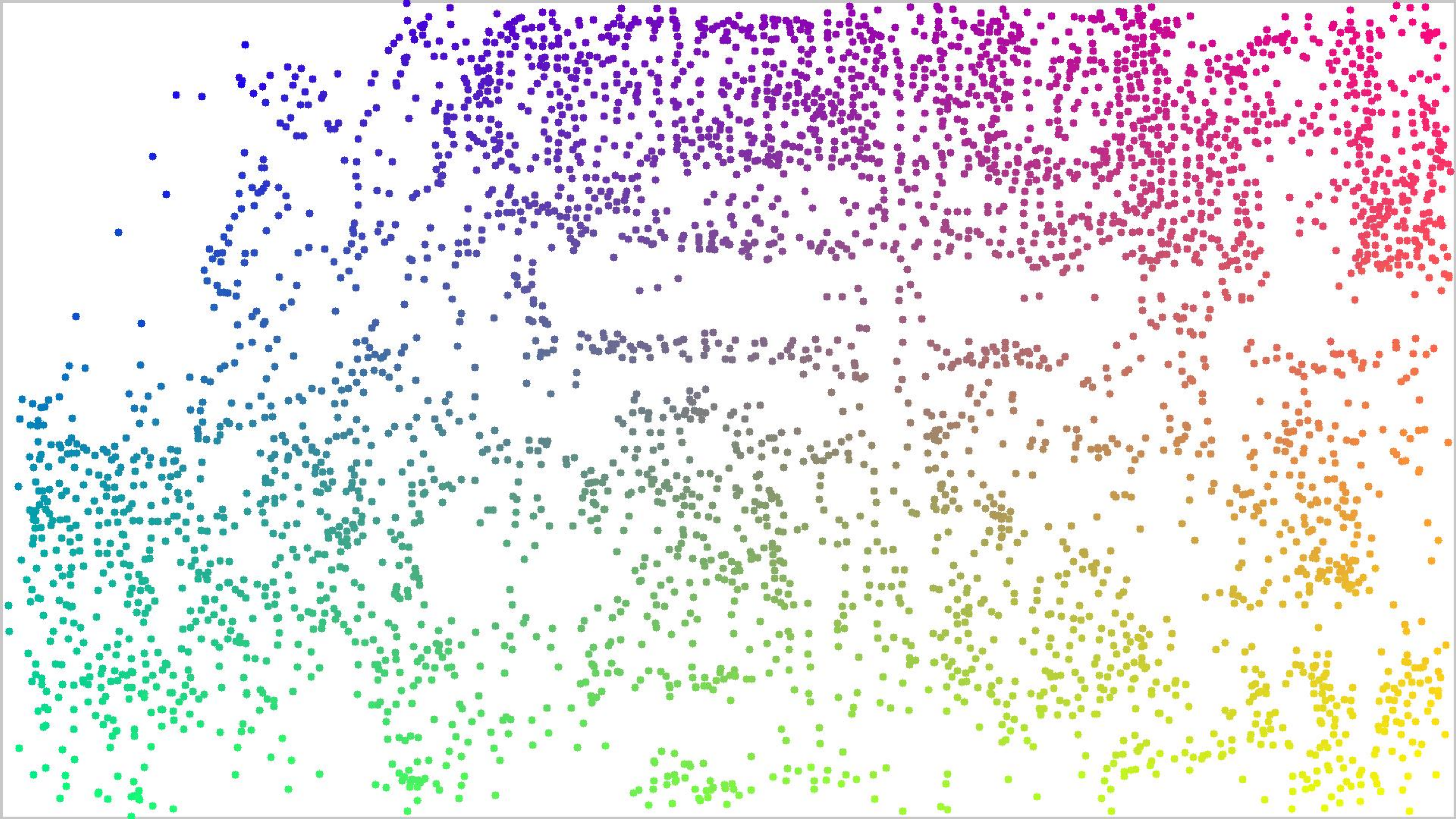}}
    } 
    \subfloat[][Coordinate Permutation~\cite{pan2023permut}]{
    \setlength{\fboxsep}{0pt}
    \fbox{\includegraphics[width=0.23\linewidth]{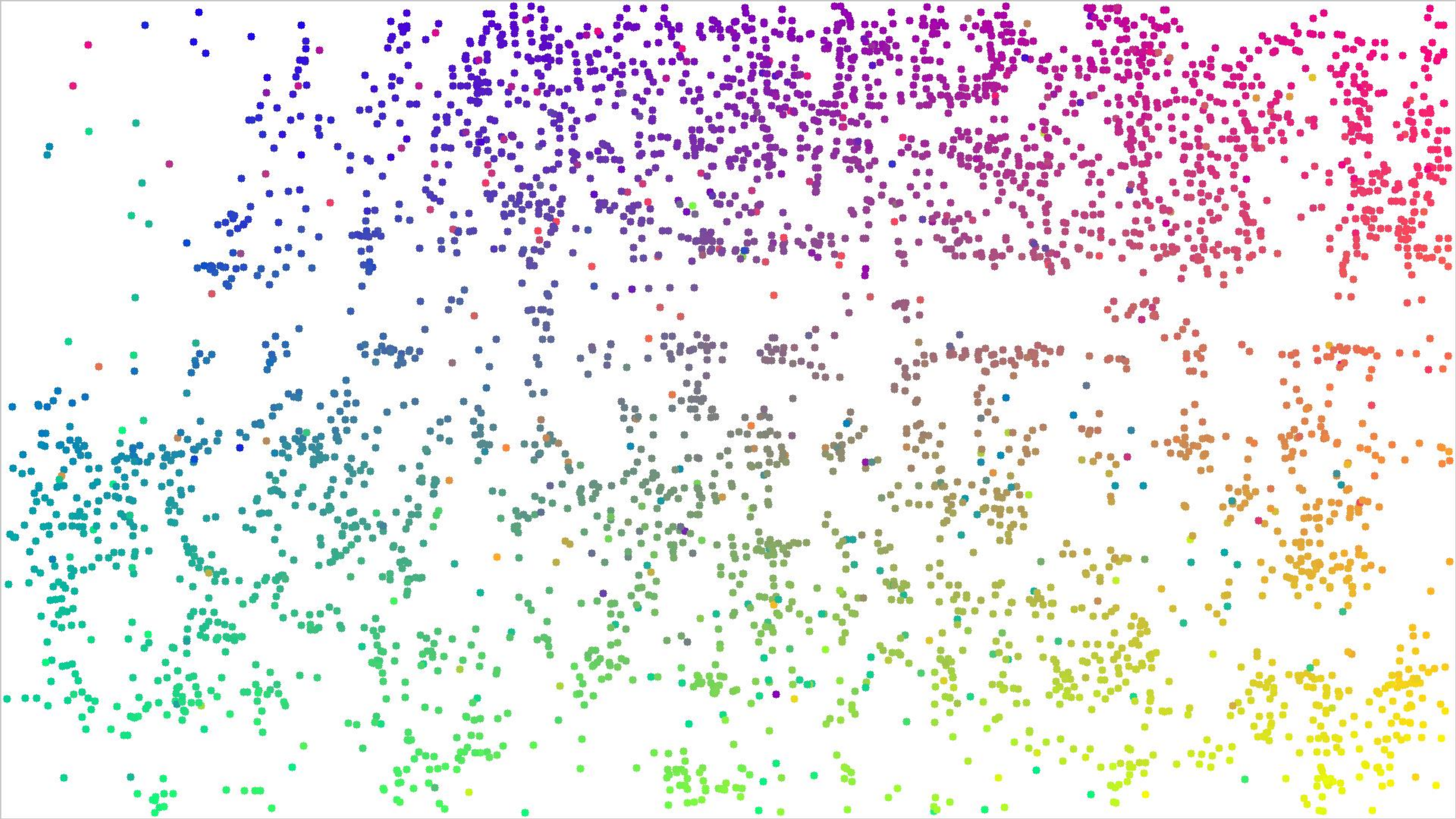}}
    } 
    \subfloat[][DCL(ours)~\figlabel{ibv_brute}]{
    \setlength{\fboxsep}{0pt}
    \fbox{\includegraphics[width=0.23\linewidth]{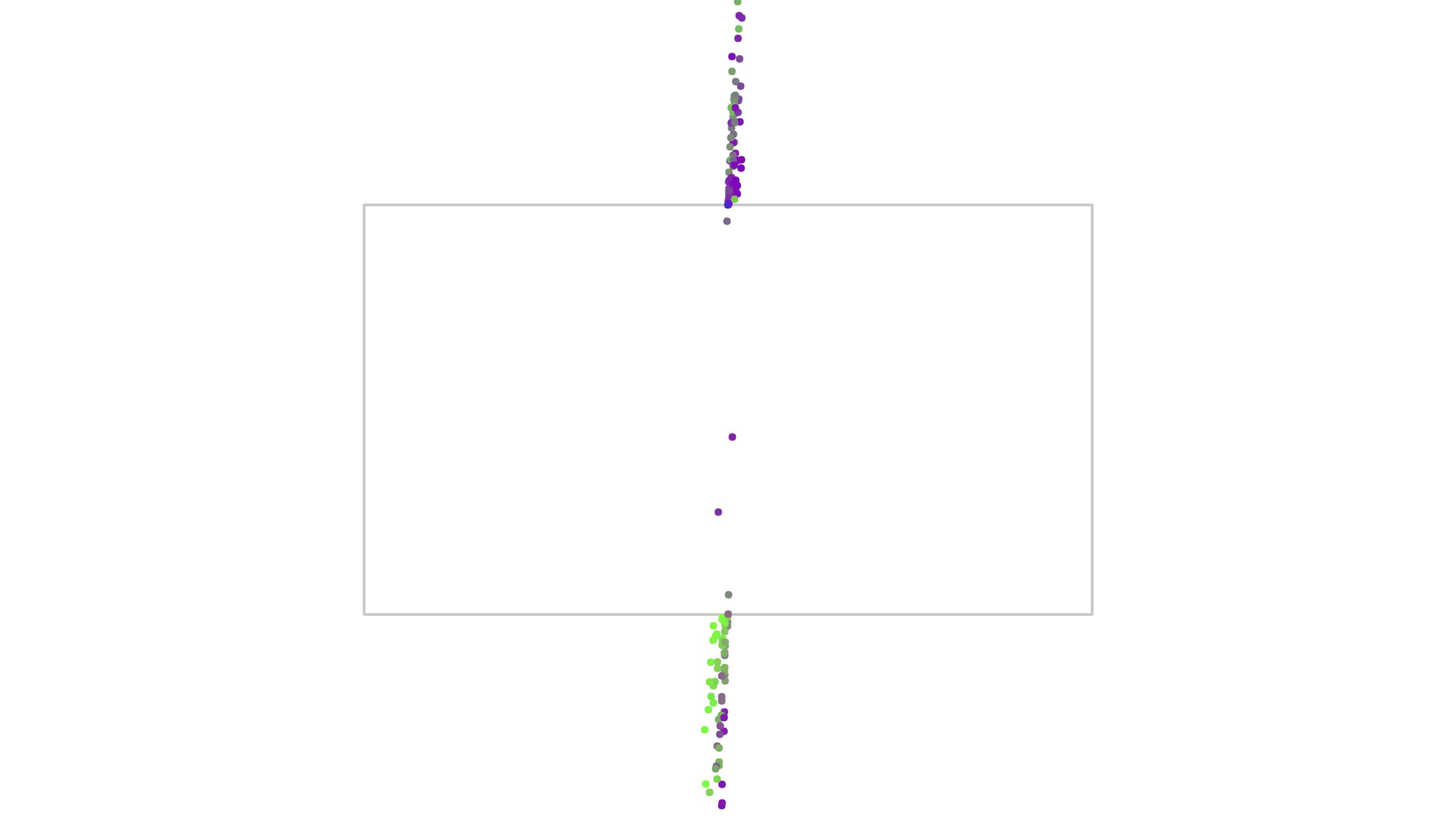}}
    } 
    \vspace{-2mm}
    \caption{
    Results of revealed Superpoint~\cite{detone2018superpoint} keypoints from the geometry-recovery attck~\cite{chelani2024obfuscation}. In (d), the view is zoomed out by $\times$2 to illustrate that the recovered points from DCL largely fall outside the original image region.
    }
    \vspace{-4mm}
    \label{fig:recovered_points}
\end{figure*}

\vspace{-3mm}
\paragraph{Mode 1: Intra-anchor convergence.}
As shown in Fig.~\ref{fig:invMode}, consider a keypoint $\mathbf{x}_i$ located in $\mathcal{R}_1$ far from the boundary such that all neighboring lines $\{\mathbf{l}_j\} = \mathcal{N}(\mathbf{l}_i)$ originate from the same anchor $\mathbf{a}_1$ of the obfuscated line $\mathbf{l}_i$. 
As all neighboring lines $\{\mathbf{l}_j\}$ intersect at $\mathbf{a}_1$, the cost function $f(\mathbf{\hat x}_i)$ in Eq.~\eqref{cost_function} has a global minimum of 0 at $\mathbf{\hat x}_i = \mathbf{a}_1$. 
Therefore, the optimization converges to the anchor $\mathbf{a}_1$, which is far from the true point $\mathbf{x}_i$ (see Fig.~\ref{fig:invMode}~(c)).

\vspace{-3mm}
\paragraph{Mode 2: Inter-anchor instability.}
This failure mode arises when keypoints lie near the partition boundary.
In such cases, a keypoint's true neighbors may reside in the opposite region, $\mathcal{R}_2$, and only these neighboring lines avoid trivially converging at the anchor $\mathbf{a}_1$. However, as keypoints from $\mathcal{R}_1$ and $\mathcal{R}_2$ approach the central boundary, their corresponding DCL lines (pointing to $\mathbf{a}_1$ and $\mathbf{a}_2$, respectively) become near-parallel. 
As shown below, attempting to find a minimal-distance intersection from a set of near-parallel lines is a classic example of an ill-posed problem.

Suppose we have a single line $\v l_j$ being anticipated to be in the vicinity of the (obfuscated) $i$-th keypoint. Then, assuming $\v l_i$ and $\v l_j$ are non-parallel, minimizing $d(\cdot,\cdot)^2$ in Eq.~\eqref{cost_function} for a single neighbor trivially leads to their intersection point.
We term the offset from $\v x_i$ to intersection point along $\v l_i$ as $t^*_{i,j}$.
Suppose now, we have multiple lines $\{\v l_j \}$ in the vicinity of $\v x_i$. Then, we yield the following results, whose full derivations are in the supplementary material.
%


\begin{proposition}
\label{prop:weighted_average}
For the inter-anchor scenario, the minimized cost function $f(\mathbf{\hat{x}}_i(t_i))$ is convex quadratic. The recovered point parameter $t_i^*$ is the weighted average of the individual parameters $\{t_{i,j}^*\}$ of the intersection points between the target line $\mathbf{l}_i$ and each neighbor line $\mathbf{l}_j$:
\begin{equation}
    t_i^* = \frac{\sum_j (w_{i,j} t_{i,j}^*)}{\sum_j w_{i,j}}.
    \label{eq:t_star_weighted}
\end{equation}
\end{proposition}
\begin{corollary}
\label{cor:weights}
The weight $w_{i,j}$ for each offset contribution $t_{i,j}^*$ in Proposition~\ref{prop:weighted_average} is determined \textbf{only} by the angle $\theta_{i,j}$ between the target line $\mathbf{l}_i$ (with direction $\mathbf{v}_i$) and its neighboring line $\mathbf{l}_j$ (with direction $\mathbf{v}_j$) as follows:
\begin{equation}
    w_{i,j} = \| \mathbf{v}_i \times \mathbf{v}_j \|^2 = \sin^2(\theta_{i,j}).
    \label{eq:weights}
\end{equation}
\end{corollary}

\noindent Established by Proposition~\ref{prop:weighted_average} and Corollary~\ref{cor:weights}, the estimated parameter $t_i^*$ is a weighted average of the individual intersection parameters $\{t_{i,j}^*\}$. The weight for each neighboring line $l_j$ is determined by $w_{i,j} = \sin^2(\theta_{i,j})$, where $\theta_{i,j}$ is the angle between the line $l_i$ and the neighbor $l_j$.



For the inter-anchor scenario (mode 2), when the target line $l_i$ and one of its neighboring lines, $l_j$, become nearly parallel, the angle $\theta_{i,j}$ and the corresponding weight $w_{i,j}$ both tends to zero, and subsequently the estimated point parameter $t_{i,j}^*$ grows very large. 
As a result, the weighted average $t_i^* = {\sum (w_{i,j} t_{i,j}^*)}/{\sum w_{i,j}}$ becomes numerically unstable and highly sensitive to noise and keypoint distribution, leading to unreliable point estimations as observed in synthetic experiments (see Fig.~\figref{DCLsynthesis} (d)).

\subsection{Camera Pose Estimation}
\label{sec:localization}

\begin{figure*}[t]
\centering

    \subfloat{
        \includegraphics[width=0.19\linewidth]{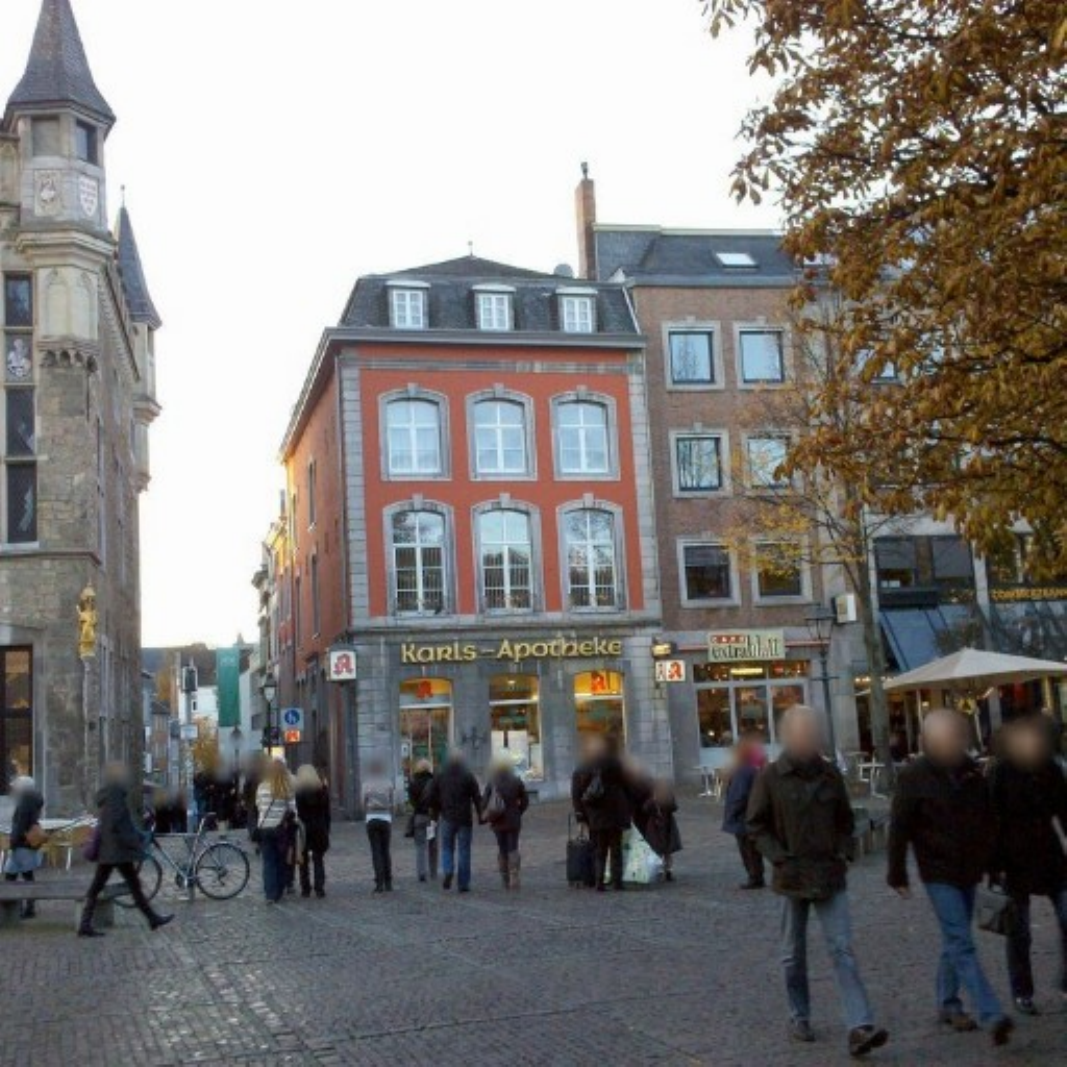}
    }
    \subfloat{
        \includegraphics[width=0.19\linewidth]{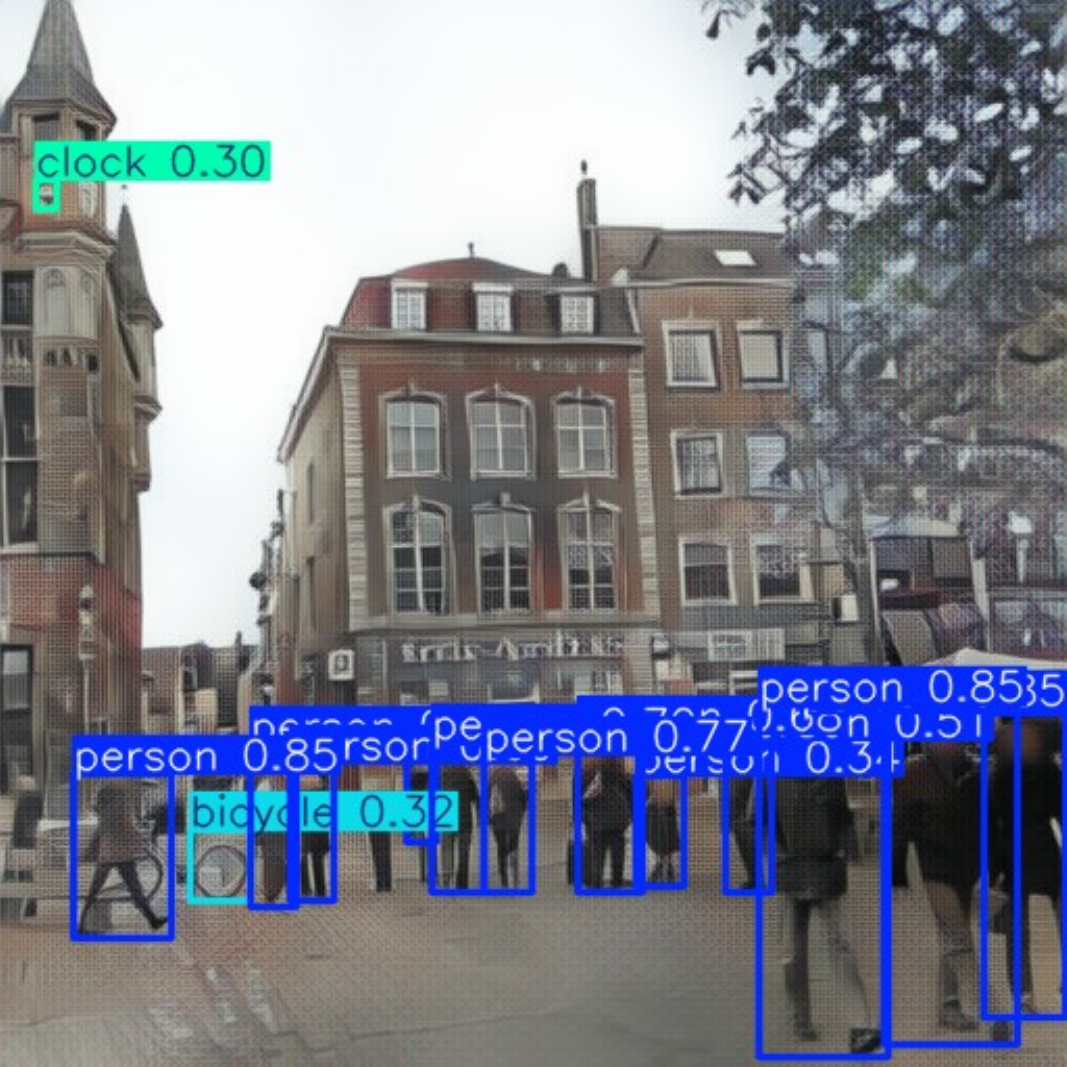}
    }
    \subfloat{
        \includegraphics[width=0.19\linewidth]{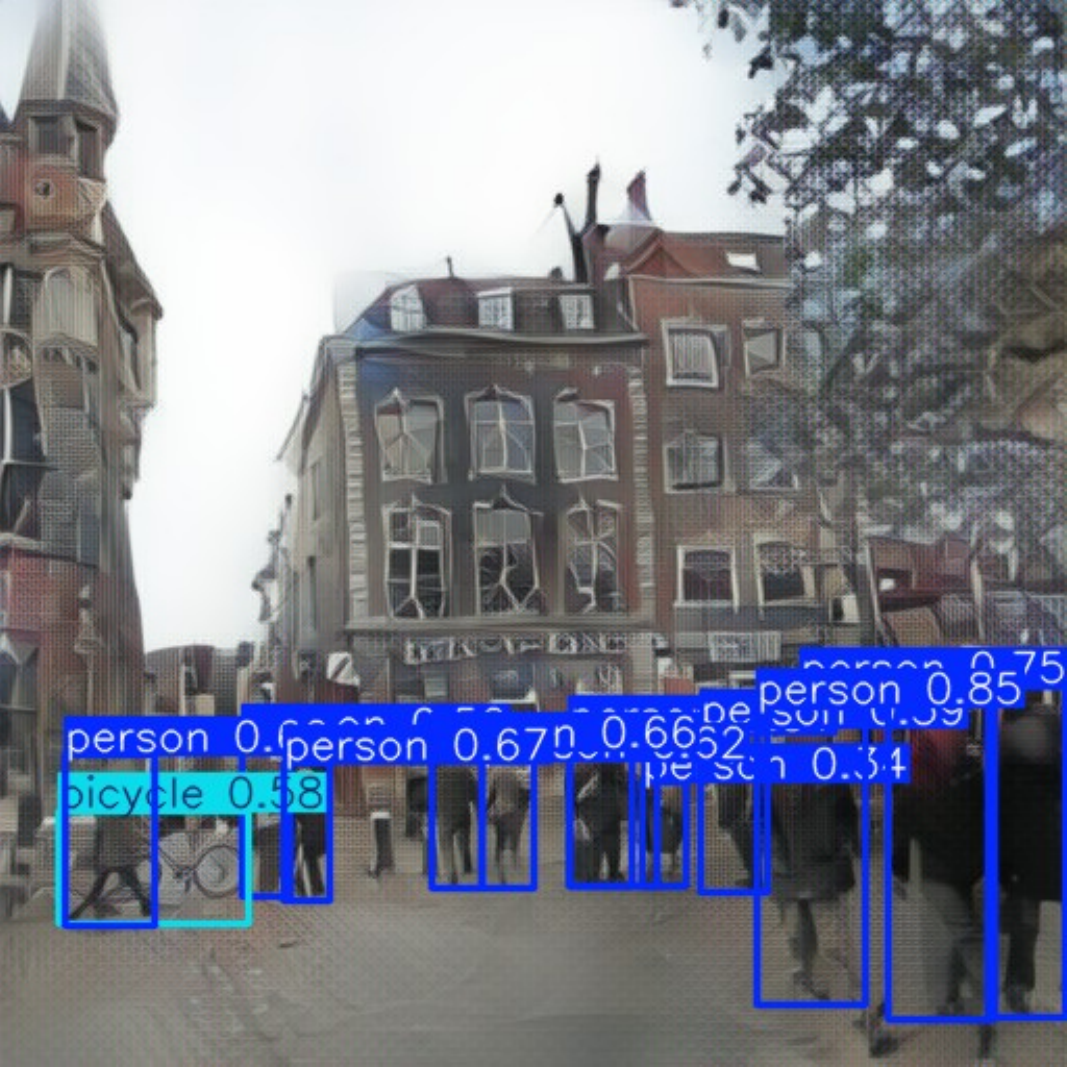}
    }
    \subfloat{
        \includegraphics[width=0.19\linewidth]{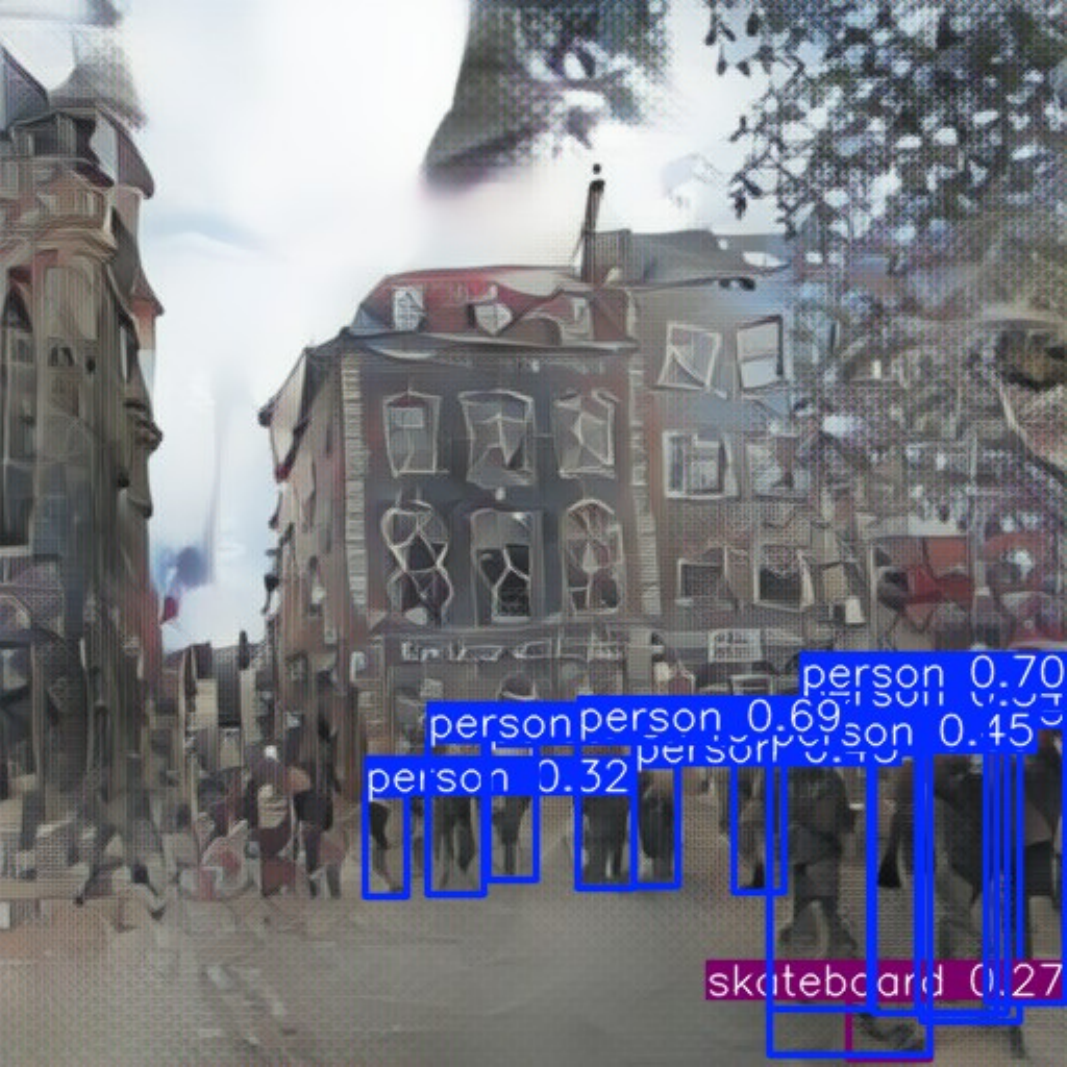}
    } 
    \subfloat{
        \includegraphics[width=0.19\linewidth]{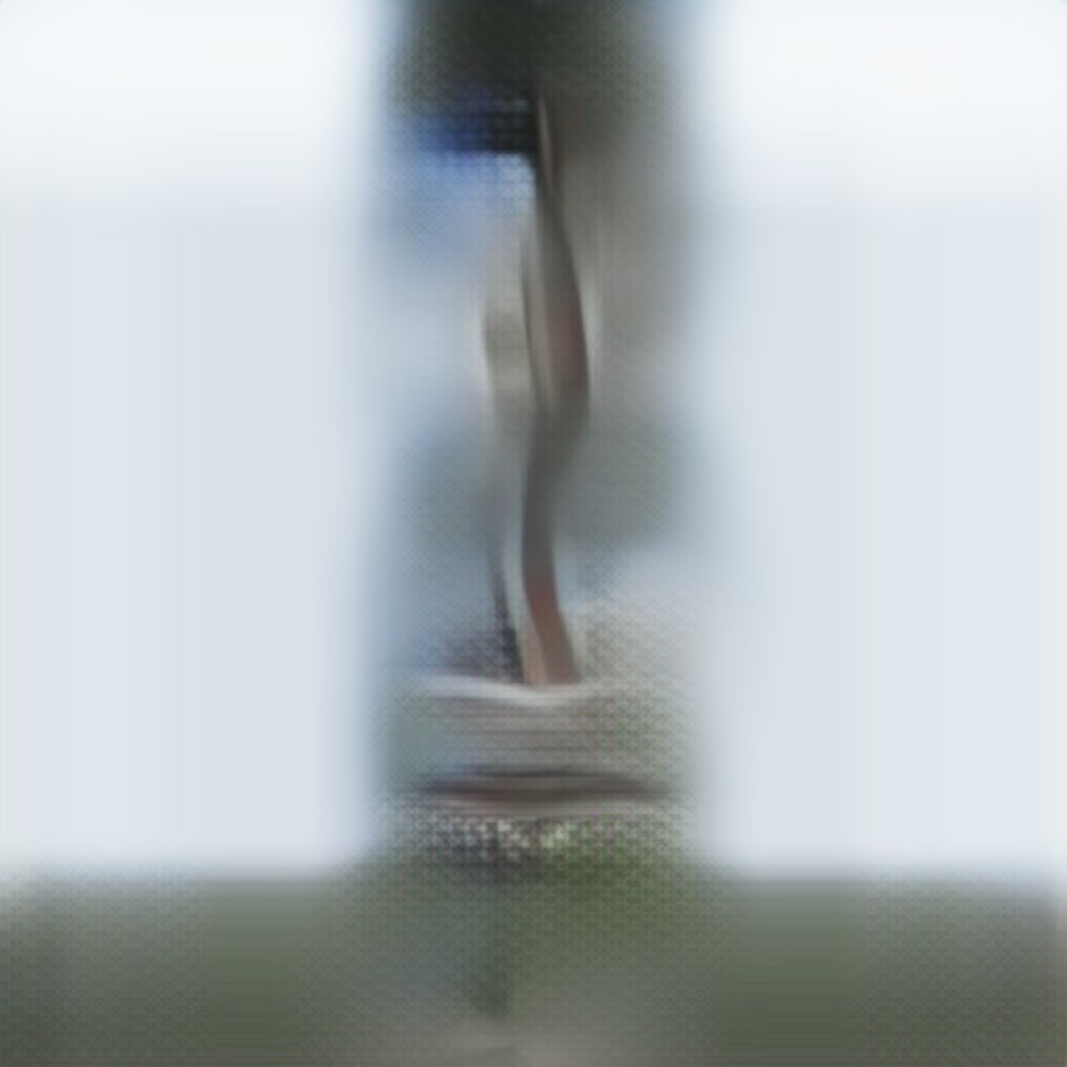}
    } 
    \\
    \vspace{1mm} 

    \setcounter{subfigure}{0}
    \subfloat{
        \includegraphics[width=0.19\linewidth]{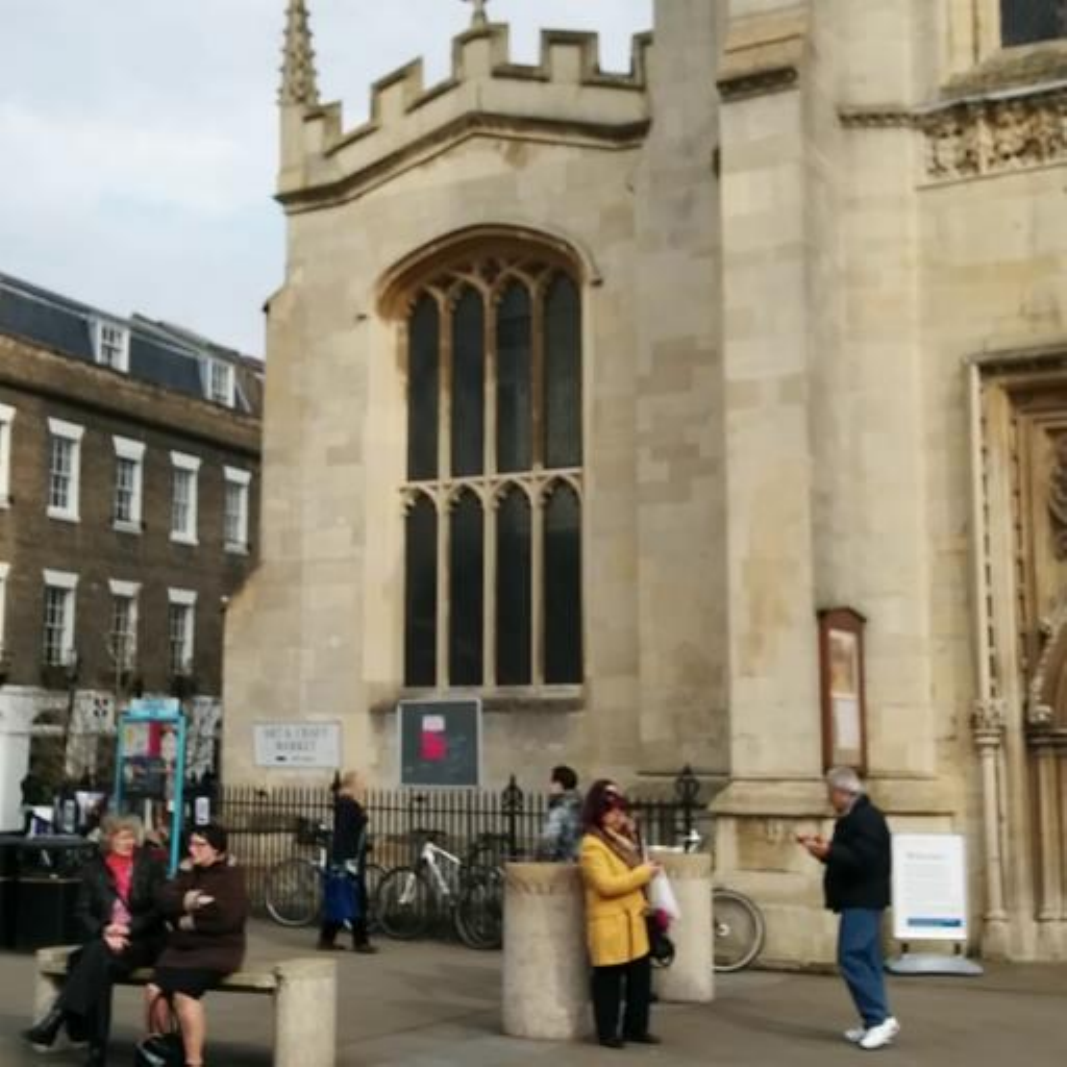}
    }
    \subfloat{
        \includegraphics[width=0.19\linewidth]{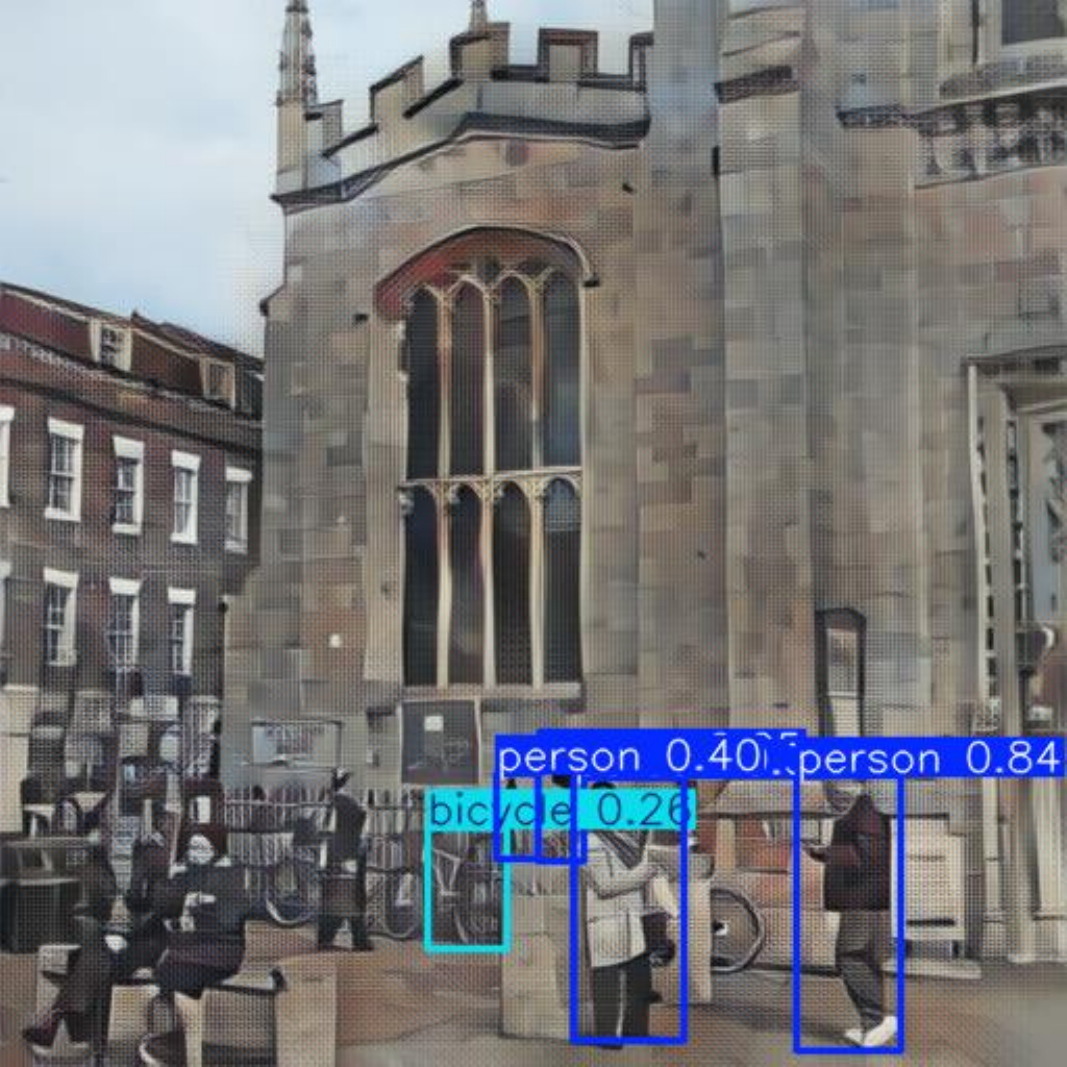}
    }
    \subfloat{
        \includegraphics[width=0.19\linewidth]{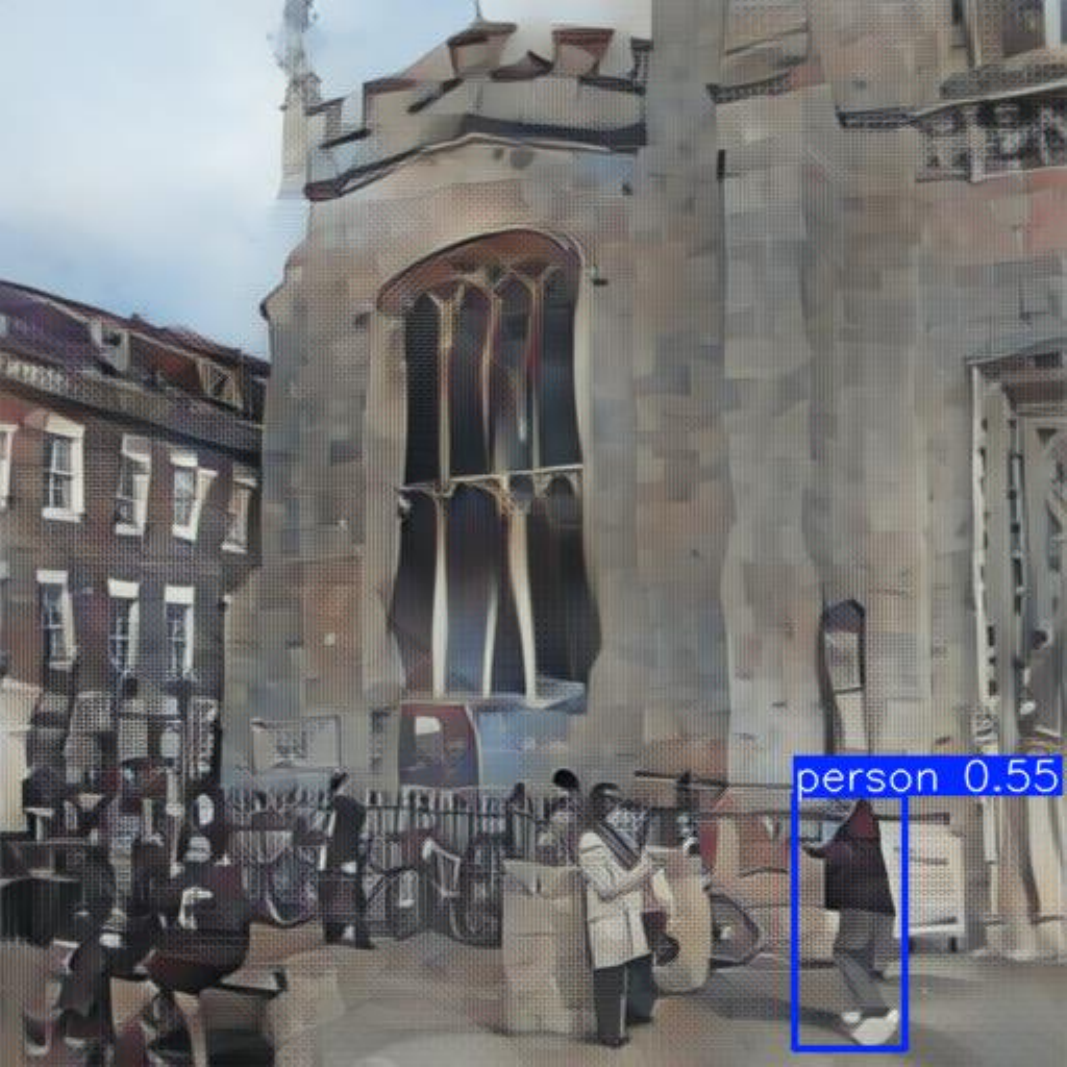}
    }
    \subfloat{
        \includegraphics[width=0.19\linewidth]{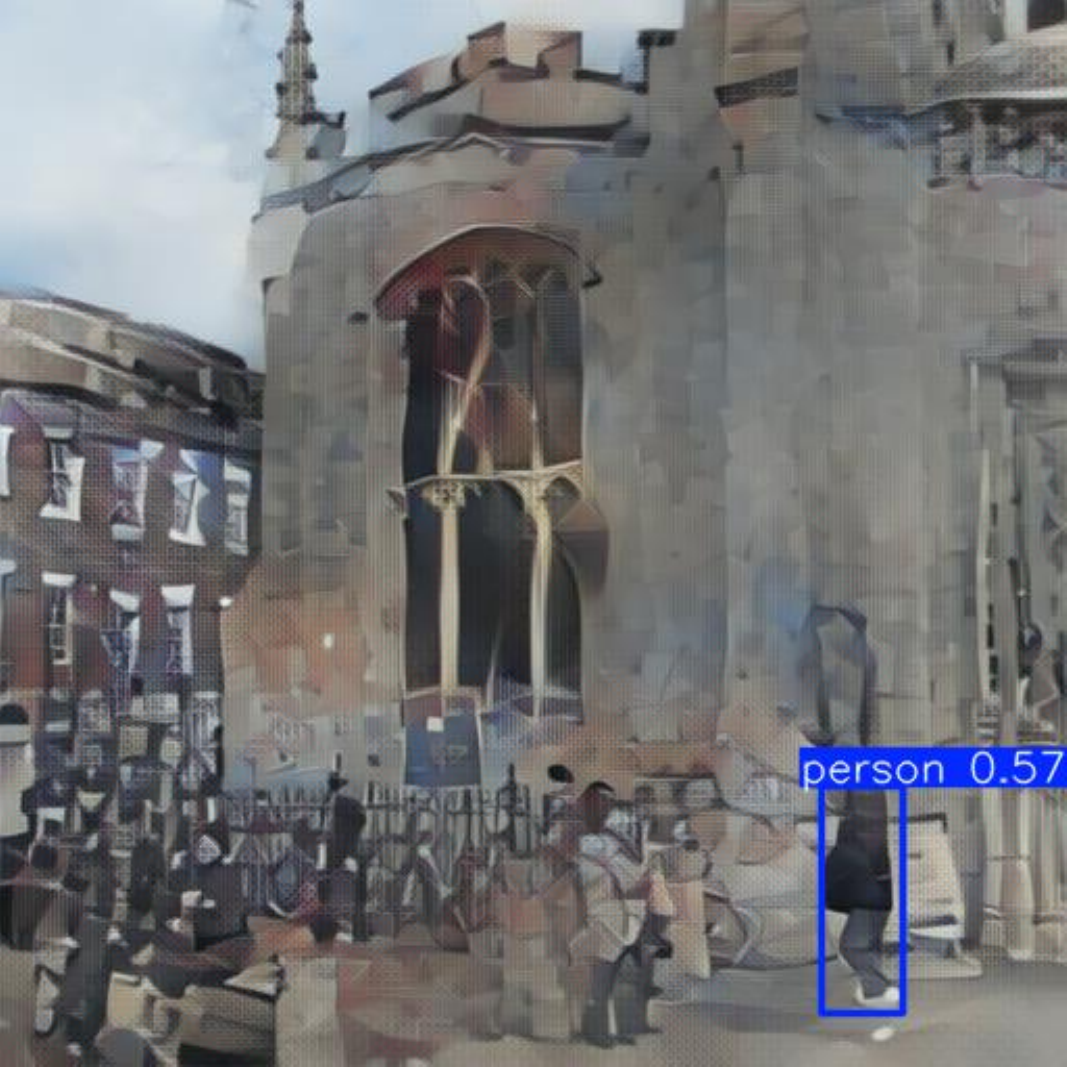}
    }
    \subfloat{
        \includegraphics[width=0.19\linewidth]{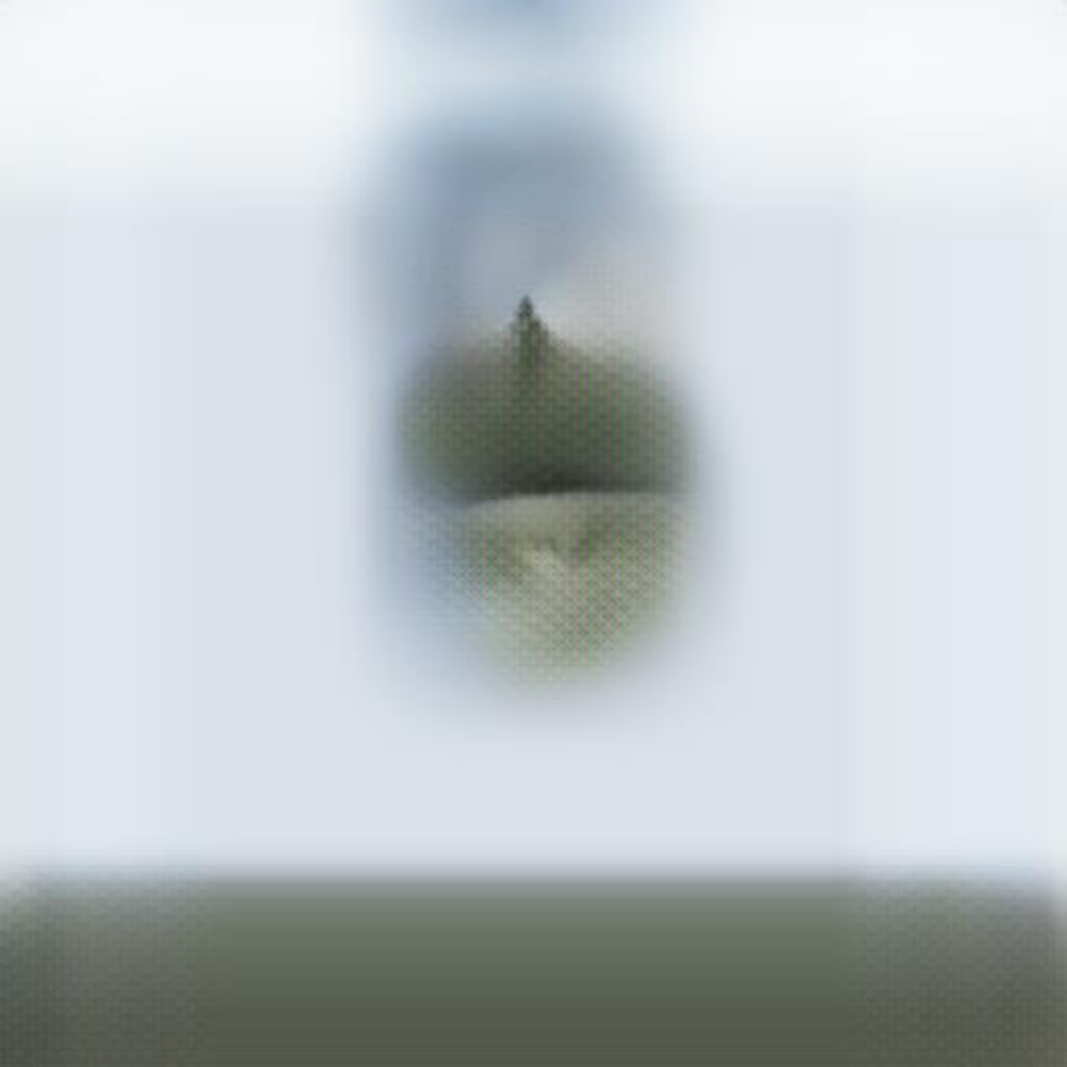}
    }
    \\
    \vspace{1mm}

    \setcounter{subfigure}{0}
    \subfloat[][Original Image]{
        \includegraphics[width=0.19\linewidth]{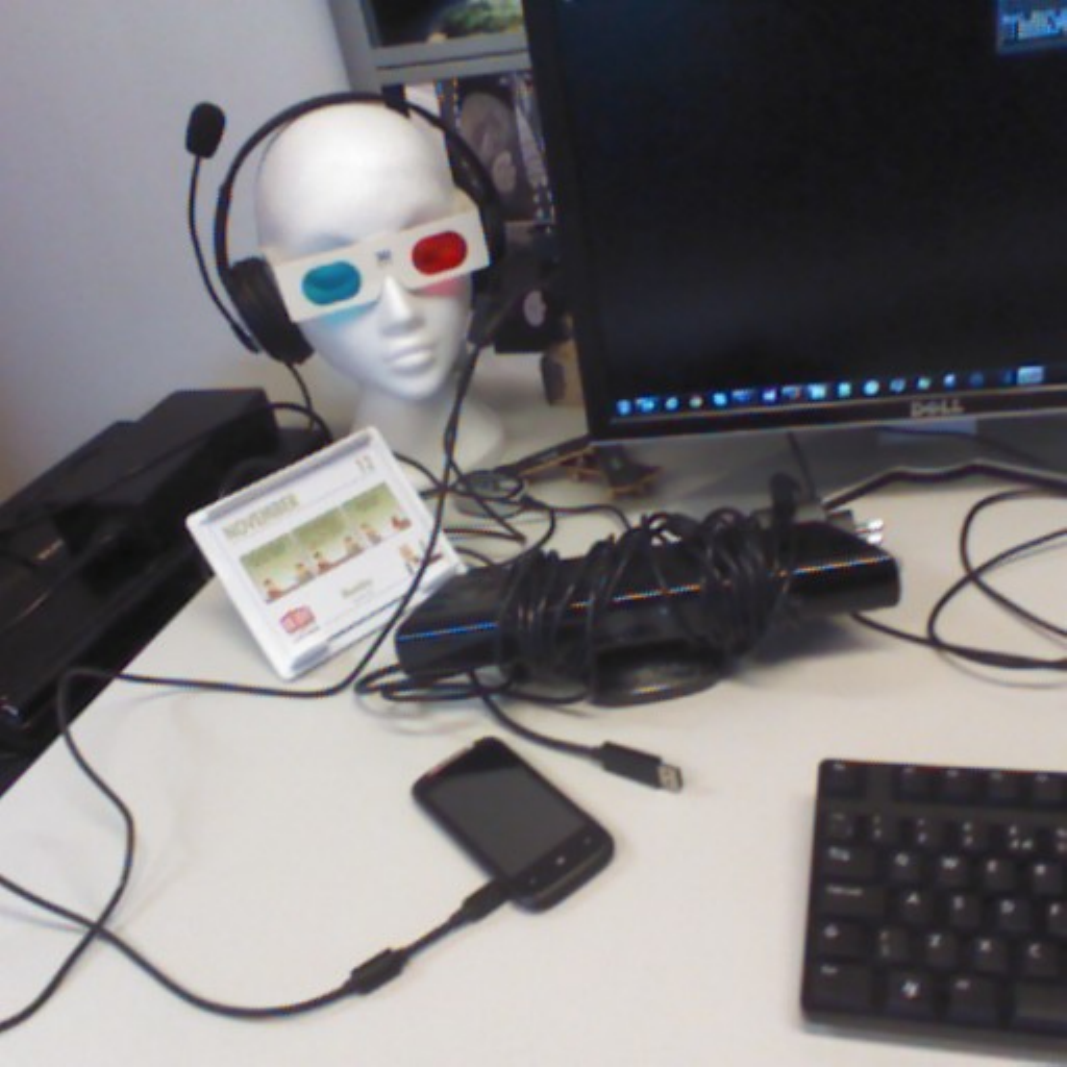}
    }
    \subfloat[][Feature Points~\cite{lowe2004sift}]{
        \includegraphics[width=0.19\linewidth]{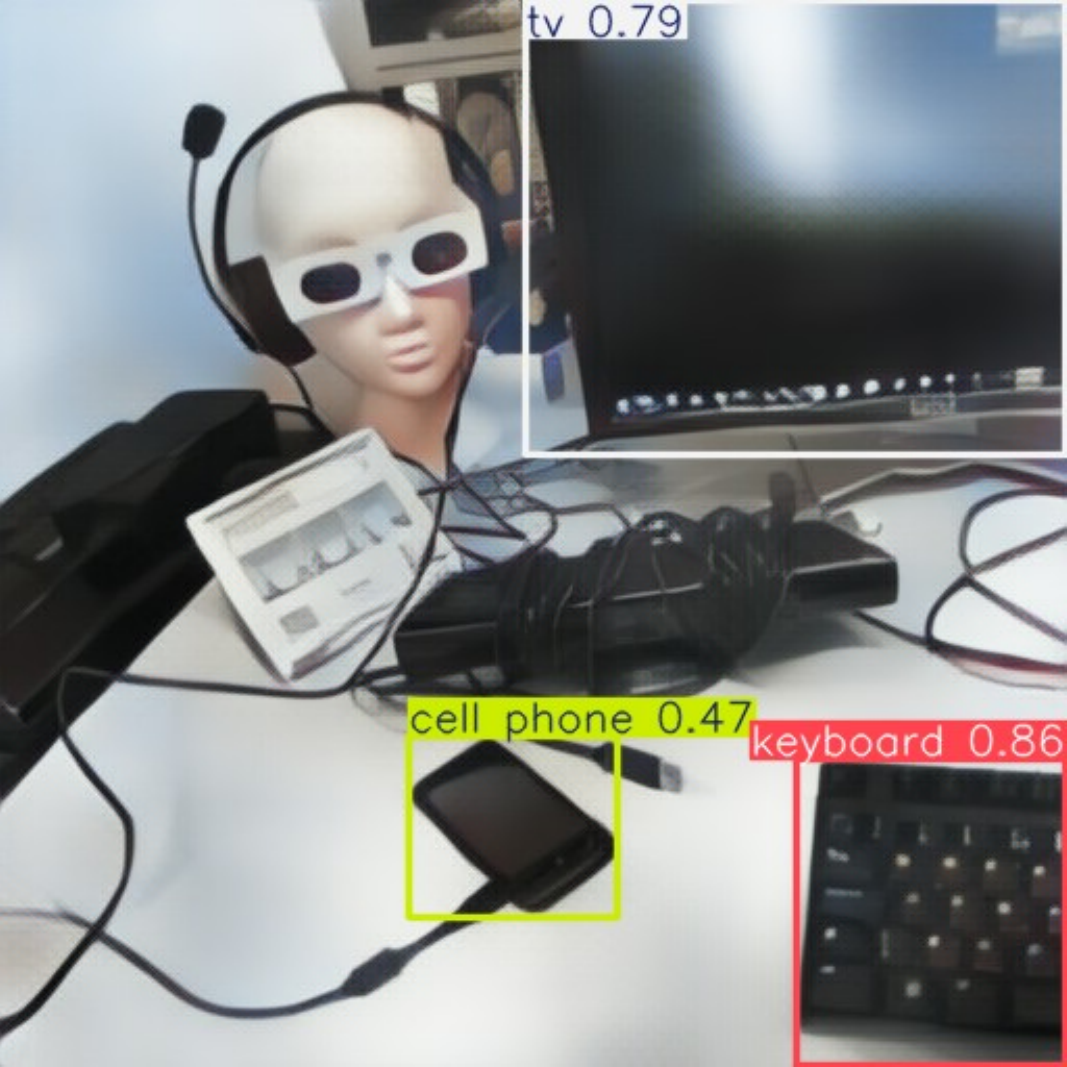}
    }
    \subfloat[][Random Lines~\cite{speciale2019queries}\figlabel{ulc_inv2d}]{
        \includegraphics[width=0.19\linewidth]{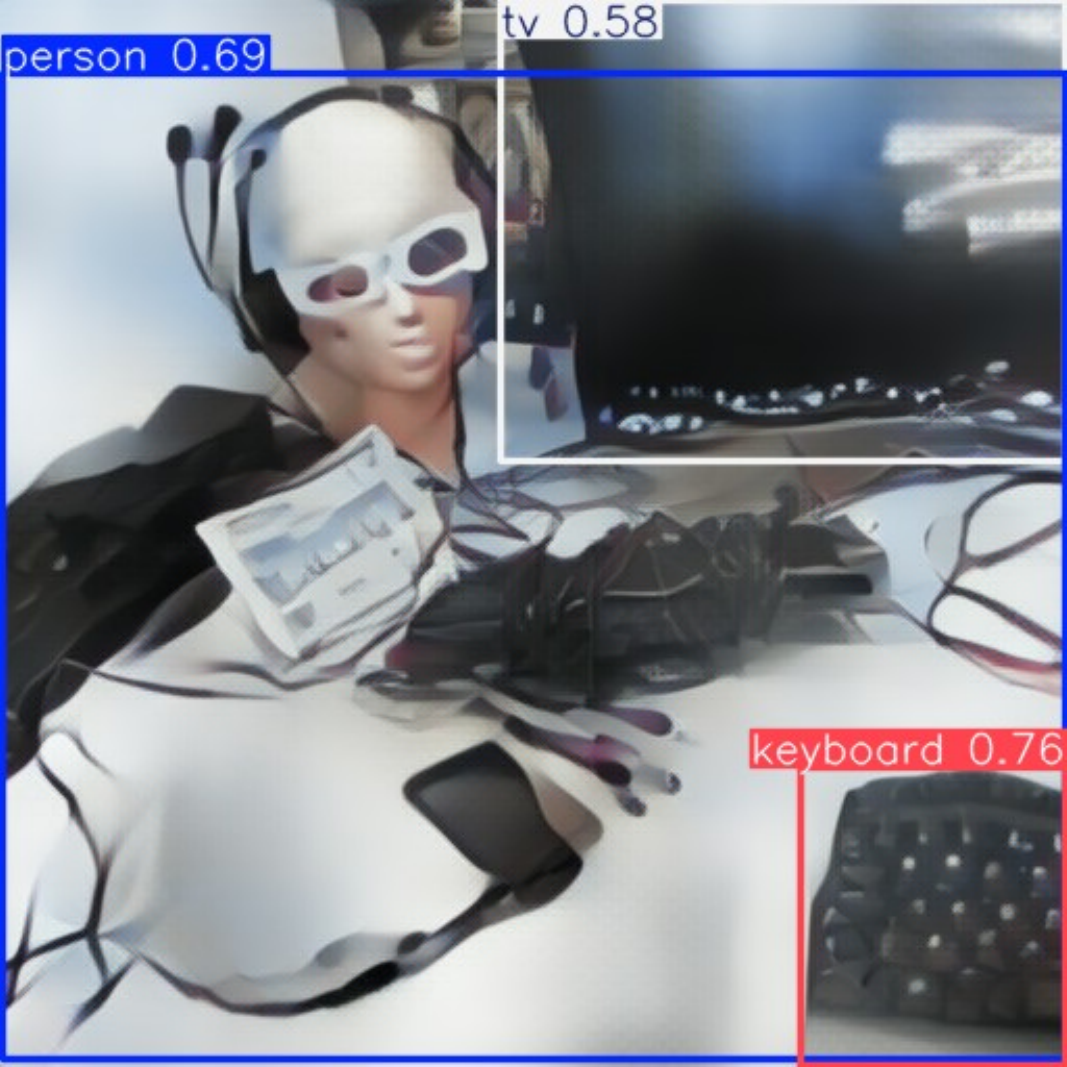}
    }
    \subfloat[][Coord. Perm.~\cite{pan2023permut}\figlabel{cp_inv2d}]{
        \includegraphics[width=0.19\linewidth]{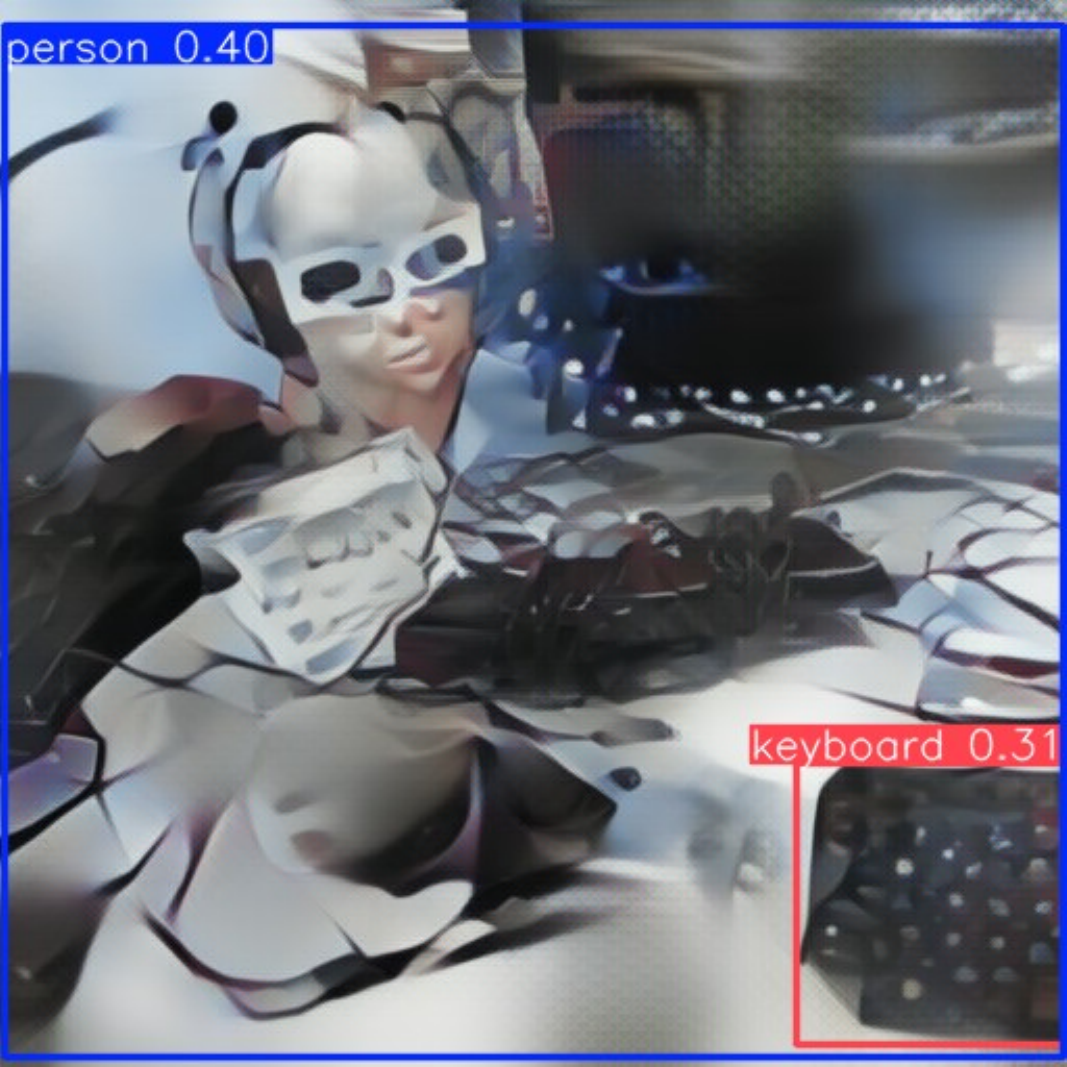}
    }
    \subfloat[][\textbf{DCL (ours)}\figlabel{dcl_inv2d}]{
        \includegraphics[width=0.19\linewidth]{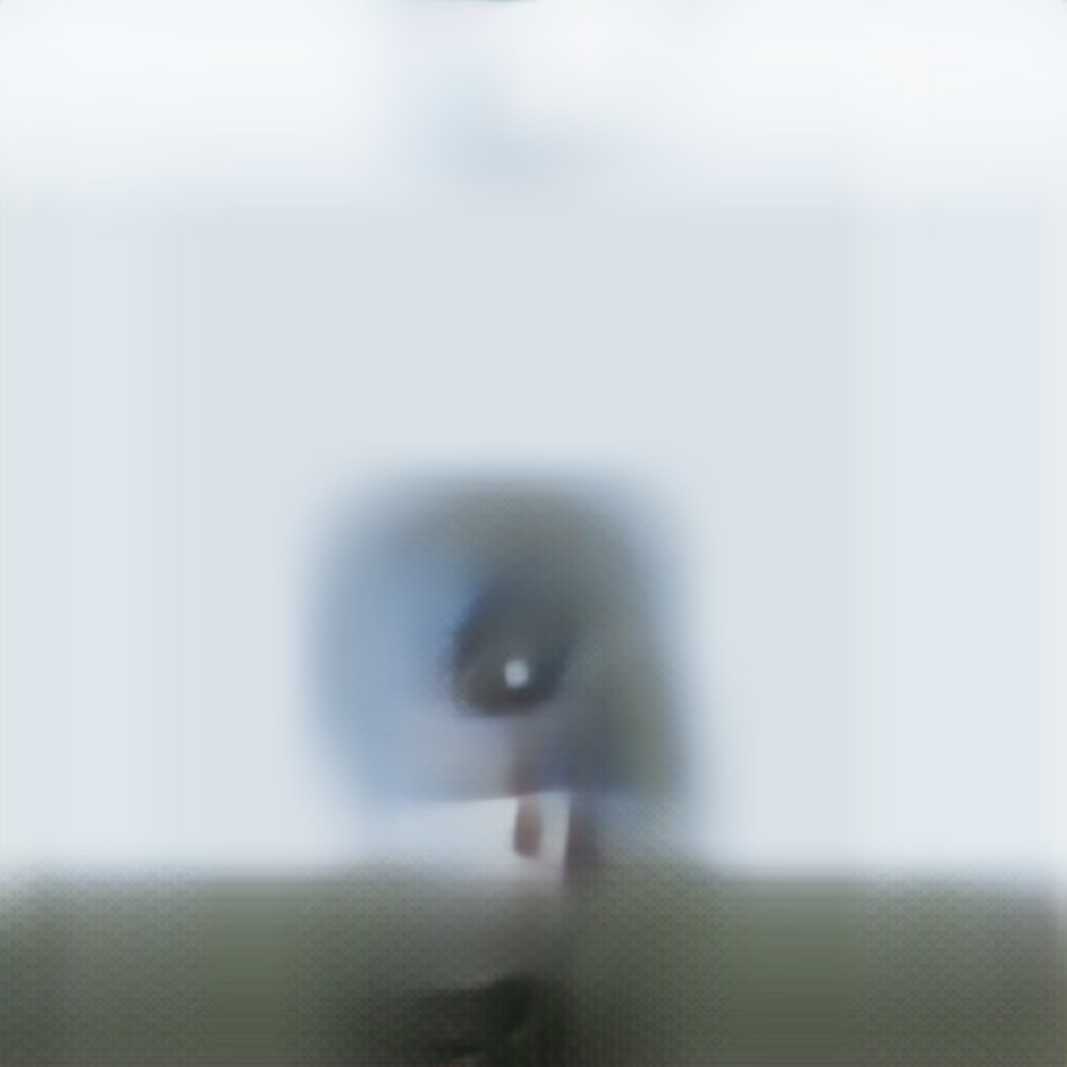}
    }

    \vspace{-2mm}
    \caption{Qualitative inversion results with the Superpoint~\cite{detone2018superpoint}, (top) Aachen (bottom) 7Scenes. Images are reconstructed after feature positions are recovered via geometry-recovery attack~\cite{chelani2024obfuscation} on various methods: (b) Feature Points~\cite{detone2018superpoint}, (c) Random Lines~\cite{speciale2019queries}, (d) Coordinate Permutation~\cite{pan2023permut}, and (e) Dual Convergent Lines (ours). We also visualize the object detection results with a pretrained YOLO-v11 model.
    Additional qualitative results and inversion results with SIFT~\cite{lowe2004sift} descriptors are in the supplementary material.
    }
    \vspace{-4mm}
    \figlabel{feat_inv}
    \label{fig:privacy_qualitative}
\end{figure*}

Since DCL is still based on 2D lines, we can still employ the \texttt{l6P} solver for performing RANSAC-based localization.
To understand a degenerate configuration when naively adopting this solver, we first review its internal linear system and illustrate how it can be practically avoided.


\vspace{-2mm}
\paragraph{Reviewing the linear system of equations in l6P solver.} 
To estimate the unknowns ($\mathbf{R}$, $\mathbf{t}$), the \texttt{l6P} solver first isolates the variables by separating the linear translation $\mathbf{t}$ from the non-linear rotation $\mathbf{R}$ in Eq.~\eqref{constraint} for each constraint:
\begin{equation}
    \mathbf{n}_i^\top \mathbf{t} = - \mathbf{n}_i^\top \mathbf{R} \mathbf{X}_i,~~ i \in \{1,...,6\}.
    \label{eq:separated_constraint}
\end{equation}
Let the non-linear rotational component be denoted by the scalar function $f_i(\mathbf{R}) = - \mathbf{n}_i^\top \mathbf{R} \mathbf{X}_i$.
The solver then partitions the six constraints into two systems of three equations each.
The first system is stacked into matrix form as $\mathbf{N}_1 \mathbf{t} = \mathbf{f}_1(\mathbf{R})$, where $\mathbf{N}_1 = [\mathbf{n}_1, \mathbf{n}_2, \mathbf{n}_3]^\top \in \mathbb{R}^{3 \times 3}$ and $\mathbf{f}_1(\mathbf{R})$ is the vector of non-linear functions. Similarly, the second system forms $\mathbf{N}_2 \mathbf{t} = \mathbf{f}_2(\mathbf{R})$, where $\mathbf{N}_2 = [\mathbf{n}_4, \mathbf{n}_5, \mathbf{n}_6]^\top$.

To solve for $\mathbf{R}$, the solver has to eliminate $\mathbf{t}$. By inverting $\mathbf{N}_1$ from the first system, the translation vector $\mathbf{t}$ is expressed in terms of $\mathbf{R}$:
\begin{equation}
    \mathbf{t} = \mathbf{N}_1^{-1} \mathbf{f}_1(\mathbf{R}).
    \label{eq:solve_for_t}
\end{equation}
Substituting this into the second system yields an equation purely in terms of $\mathbf{R}$:
\begin{equation}
    \mathbf{N}_2 \left( \mathbf{N}_1^{-1} \mathbf{f}_1(\mathbf{R}) \right) = \mathbf{f}_2(\mathbf{R}).
    \label{eq:final_system_R}
\end{equation}
Eq.~\eqref{final_system_R} consists of three nonlinear polynomial equations in $\mathbf{R}$, which the minimal solver solves. Note the overall solving process fails at Eq.~\eqref{solve_for_t} if the matrix $\mathbf{N}_1$ is non-invertible.

\vspace{-3mm}
\paragraph{Avoiding degenerate configuration.}
As noted in Sec.~\ref{sec:preliminaries}, each 2D line $\v l_i$ back-projects to a 3D plane $\Pi_i$ whose normal $\v n_i$ is used by the solver. 
Note, during RANSAC, there is a roughly $2\times0.5^3=0.25$ probability of sampling three obfuscated lines passing through the same 2D anchor point ($\v a_1$ or $\v a_2$). In such case, all their corresponding 3D planes $\{\Pi_i\}$ must intersect along a camera ray passing through the camera's optical center and the anchor point.

By definition, the normal $\v n_i$ of any plane $\Pi_i$ in this set must be orthogonal to this common ray, forcing all normal vectors to lie in the 2D subspace.
When these coplanar normals are stacked to form the matrix $\v N_1 = [\v n_1, \v n_2, \v n_3]\tr$, the matrix is rank-deficient and thus non-invertible. This makes $\v N_1$ non-invertible ($\mathrm{det}(\v N_1) = 0$), causing the solver to fail when attempting to isolate $t$ (Eq.~\eqref{solve_for_t}).

We mitigate this by enforcing that the minimal set comprises lines from both anchors (e.g., two from $a_1$, one from $a_2$).
This can be efficiently implemented by monitoring the intersection points of the three sampled lines, \ie evaluating whether the intersection point of pairs of lines is identical (degenerate) or not (non-degenerate).
This guarantees that at least one normal vector does not lie in the 2D subspace of other two normal vectors, effectively breaking the coplanarity and allowing the set of normals to span the full 3D space.
Since $N_1$ is now full-rank, the \texttt{l6P} solver can be proceeded without further numerical issues.

\section{Experimental Results}
\label{sec:experiments}



\noindent\textbf{Datasets.} We evaluate our method on three standard visual localization benchmarks: 7Scenes \cite{shotton2013scene}, Cambridge Landmarks \cite{kendall2015posenet}, and the original Aachen Day-Night dataset \cite{sattler2012image,sattler2018aachen}. 7Scenes provides dense indoor RGB-D sequences, Cambridge Landmarks offers outdoor scenes of varying scales, and Aachen Day-Night presents a challenging large-scale scenario with significant appearance changes. 
For all datasets, we build our 3D map using the triangulation pipeline from HLoc~\cite{sarlin2019coarse}.
We initialize the map by adopting the camera poses from the pre-existing SfM model~\cite{schoenberger2016sfm} provided with the dataset. We then extract local features using SuperPoint~\cite{detone2018superpoint} and match them with a nearest neighbor (NN) matcher. Finally, we triangulate a new 3D point cloud based on these features and the fixed camera poses.

\noindent\textbf{Evaluation metrics.} To quantify the resilience against geometry recovery attacks, we measure the geometric error of recovered keypoints in pixel scale and the quality of images reconstructed by an inversion network using PSNR~\cite{huynh2008scope}, SSIM~\cite{wang2004image}, and LPIPS~\cite{zhang2018unreasonable}. Lower PSNR/SSIM and higher LPIPS indicate better privacy protection.
We also visualize the object detection results with the reconstructed images using the pretrained YOLO-v11 model following~\cite{chelani2024obfuscation}.
For localization accuracy, we compute the rotational and translational errors as $\Delta R = \arccos\left(\frac{\text{Tr}(\mathbf{R}^T \hat{\mathbf{R}})-1}{2}\right)$ and $\Delta T = \|\mathbf{R}^T \mathbf{T} - \hat{\mathbf{R}}^T \hat{\mathbf{T}}\|_2$. 
We report the median rotation error (degrees, $\downarrow$) and median translation error ($\downarrow$), along with the recall (\% $\uparrow$) within specific accuracy threshold. 

\noindent\textbf{Implementation details.} Our localization pipeline uses SuperPoint \cite{detone2018superpoint} features extracted from the query image. We then perform image retrieval~\cite{arandjelovic2016netvlad, torii201524} to identify potential candidate database images.
Subsequently, considering privacy, we establish 2D-3D matches using a NN search on the feature descriptors instead of keypoint position-dependent methods like SuperGlue~\cite{sarlin2020superglue}.
For pose estimation, we utilized the $\texttt{l6P}$ minimal solver~\cite{speciale2019privacy} within a Lo-RANSAC~\cite{chum2003loransac} loop using PoseLib \cite{PoseLib}, followed by refinement via Levenberg-Marquardt optimization~\cite{levenberg1944method}.

For the geometry-inversion attack~\cite{chelani2024obfuscation}, we follow the same experimental setup as~\cite{chelani2024obfuscation}, using 20 neighborhood obfuscations per keypoint and assuming perfect identification of these neighbors to simulate the \emph{worst-case} (upper-bound) scenario.
In image inversion, we used the inversion network from~\cite{dangwal2021mitigating} with Superpoint features and InvSfM~\cite{pittaluga2019revealing} with SIFT features~\cite{lowe2004sift}. Further details on the localization and inversion model are in the supplementary material.
All experiments were run on a desktop PC with a AMD Ryzen9 7900X CPU, 64GB RAM, and an NVIDIA RTX 4090 GPU.

\begin{table}[t]
\centering
\scriptsize

\setlength{\tabcolsep}{5pt}
\renewcommand{\arraystretch}{1.2}
\begin{tabular}{c|c|c|c|c|c}
\hline
        &         & Feature  & Random & Coord. & DCL\\
Dataset & Metrics &   Points~\cite{detone2018superpoint} & Lines~\cite{speciale2019queries} & Perm.~\cite{pan2023permut} & (ours)\\
\hline

\multirow{4}{*}{7-scenes} 
& $e_{recon}$(\textcolor{darkgreen}{↑}) & - & 6.137 & 10.56 & \textbf{330.4} \\
& PSNR(\textcolor{blue}{↓}) & 15.886 & 13.899 & 13.358 & \textbf{7.040} \\
& SSIM(\textcolor{blue}{↓}) & 0.681 & 0.503 & \textbf{0.444} & \textbf{0.444} \\
& LPIPS(\textcolor{darkgreen}{↑})& 0.477 & 0.604 & 0.643 & \textbf{0.754} \\
\hline

\multirow{4}{*}{Cambridge} 
& $e_{recon}$(\textcolor{darkgreen}{↑}) & - & 6.386 & 11.81 & \textbf{800.2} \\
& PSNR(\textcolor{blue}{↓}) & 15.667 & 14.991 & 14.309 & \textbf{6.746} \\
& SSIM(\textcolor{blue}{↓}) & 0.433 & 0.376 & 0.335 & \textbf{0.276} \\
& LPIPS(\textcolor{darkgreen}{↑})& 0.495 & 0.517 & 0.558 & \textbf{0.740} \\
\hline

\multirow{4}{*}{Aachen} 
& $e_{recon}$(\textcolor{darkgreen}{↑}) & - & 5.381 & 9.855 & \textbf{713.0} \\
& PSNR(\textcolor{blue}{↓}) & 16.103 & 15.386 & 14.281 & \textbf{7.021} \\
& SSIM(\textcolor{blue}{↓}) & 0.430 & 0.358 & 0.291 & \textbf{0.257}\\
& LPIPS(\textcolor{darkgreen}{↑})& 0.433 & 0.476 & 0.539 & \textbf{0.736} \\
\hline

\end{tabular}
\vspace{-2mm}
\caption[Result-inversion-attacks]
{Mean geometric error ($e_{recon}$) of points recovered by the geometry-recovery attack~\cite{chelani2024obfuscation} (in pixels) and mean quality metrics for images reconstructed from those points with Superpoint~\cite{detone2018superpoint}.}
\label{tab:privacy_quantitative}
\vspace{-5mm}
\end{table}

\subsection{Robustness Against Privacy Attacks}
\label{sec:privacy_robustness}

We conducted both quantitative and qualitative evaluations of DCL's robustness against the neighborhood-based geometry-recovery attack~\cite{chelani2024obfuscation}, also comparing results against prior geometric-obfuscation schemes~\cite{speciale2019queries, pan2023permut}.


As shown in Fig.~\ref{fig:recovered_points}, DCL achieves the worst point recovery, with recovered points deviating significantly from the true positions.
In contrast, prior schemes like Random Lines \cite{speciale2019queries} and Coordinate Permutation \cite{pan2023permut} still resemble the ground truth, indicating vulnerability to the attack.
Consequently, DCL demonstrates degraded image reconstruction quality compared to prior methods, as shown in Fig.~\ref{fig:privacy_qualitative}.

The quantitative inversion results in Table~\ref{tab:privacy_quantitative} also support DCL's enhanced privacy capability, as it consistently shows the worst point and image recovery quality across all datasets and metrics.
Although the SSIM score for Coordinate Permutation~\cite{pan2023permut} on the 7Scenes is on par with DCL, DCL consistently achieves worse recovery quality, including point recovery error, PSNR/LPIPS, and qualitative results in Fig.~\ref{fig:privacy_qualitative}, SIFT-based inversion results in supplement.
\begin{table}[t!]
\centering
\renewcommand{\arraystretch}{1.2}
\setlength{\tabcolsep}{3pt}
\scriptsize
\begin{adjustbox}{max width=\textwidth}
\begin{tabular}{l||cc|cccc}
\textbf{Method} & \textbf{\makecell{P.V.}} &\textbf{\makecell{R.T.}}  & \textbf{King's} & \textbf{Hospital} & \textbf{Shop} & \textbf{St. Mary's} \\
\specialrule{.1em}{.05em}{.05em} 

HLoc(SP+SG)~\cite{sarlin2019coarse,sarlin2020superglue}  & N/A & \cmark & 12/0.2 & 15/0.3 & 4/0.2 & 7/0.2 \\
HLoc(SP+NN)~\cite{sarlin2019coarse} & N/A  & \cmark & 12/0.2 & 15/0.3 & 4/0.2 & 8/0.2 \\
DSAC*~\cite{brachmann2021dsacstar} & N/A & \cmark& 18/0.3 & 21/0.4 & 5/0.3 & 15/0.6 \\
ACE~\cite{brachmann2023ace} & N/A & \cmark& 28/0.4 & 31/0.6 & 5/0.3 & 18/0.6 \\
GLACE~\cite{GLACE2024CVPR} & N/A & \cmark & 19/0.3 & 17/0.4 & 4/0.2 & 9/0.3 \\
\midrule

GoMatch~\cite{zhou2022geometry} & \cmark & \cmark& 25/0.64 & 283/8.14 & 48/4.77 & 335/9.94 \\
DGC-GNN~\cite{wang2024dgc} &\cmark & \cmark& 18/0.47 & 75/2.83 & 15/1.57 & 106/4.03 \\
SegLoc~\cite{pietrantoni2023segloc} & \cmark & - & 24/0.26 & 36/0.52 & 11/0.34 & 17/0.46  \\
GSFF Privacy~\cite{pietrantoni2025gaussian} & \cmark & \xmark & 24/0.39 & 26/0.49 & 5/0.27 & 13/0.48 \\

Random Lines~\cite{speciale2019queries} & \cmark & \cmark& 11/0.2 & 16/0.3 & 4/0.2 & 7/0.2 \\
DCL (ours) & \xmark & \cmark&  25/0.4 & 38/0.7 & 10/0.4 & 16/0.5 \\

\bottomrule
\end{tabular}
\end{adjustbox}
\vspace{-2mm}
\caption{
Localization results on the Cambridge dataset \cite{kendall2015posenet}. 
\textbf{P.V.} stands for Privacy Vulnerability under recent attacks~\cite{chelani2024obfuscation, anonymous2025vulnerability, pittaluga2019revealing}, and \textbf{R.T.} denotes real-time performance.
Median position error (cm) ($\downarrow$) and median rotation error ($^\circ$) ($\downarrow$) are reported.
}
\label{tab:cambridge}
\vspace{-4mm}
\end{table}
\begin{table}[t!]
\centering
\renewcommand{\arraystretch}{1.2}
\setlength{\tabcolsep}{2pt}
\scriptsize
\begin{tabular}{l|c|ccc|ccc}
& & \multicolumn{3}{c|}{\textbf{Aachen Day}} & \multicolumn{3}{c}{\textbf{Aachen Night}} \\
\toprule
\textbf{Method} & \textbf{\makecell{P.V.}} & 0.25m,2° & 0.5m,5° & 5m,10° & 0.25m,2° & 0.5m,5° & 5m,10° \\
\midrule
HLoc(SP+NN)& N/A  & 86.5&93.7&96.8&70.4&86.7&94.9 \\
ACE x 50 & N/A & 6.9 & 17.2 & 50.0 & 0.0 & 1.0 & 5.1 \\
GLACE & N/A & 8.6 & 20.8 & 64.0 & 1.0 & 1.0 & 17.3 \\
\midrule
Random Lines & \cmark & 79.9&87.1&94.2&54.1&63.3&85.7 \\
DCL (ours) & \xmark & 41.0&57.9&72.7&13.3&21.4&38.8 \\
\bottomrule
\end{tabular}
\vspace{-2mm}
\caption{Localization performance on the Aachen Day \& Night dataset~\cite{sattler2018aachen}. 
Recall~(\%) at different thresholds is reported.}
\label{tab:aachen}
\vspace{-4mm}
\end{table}
\begin{table*}[h!]
\centering
\renewcommand{\arraystretch}{1.2}
\setlength{\tabcolsep}{3pt}
\scriptsize
\begin{adjustbox}{max width=\textwidth}
\begin{tabular}{ll|cc||cccccccc}
\toprule
 & \textbf{Method} & \textbf{\makecell{P.V.}} & \textbf{\makecell{R.T.}} & \textbf{Chess} & \textbf{Fire} & \textbf{Heads} & \textbf{Office} & \textbf{Pumpkin} & \textbf{Redkitchen} & \textbf{Stairs} \\
\specialrule{.1em}{.05em}{.05em}

\multirow{4}{*}{\parbox[c]{1.8cm}{\centering \textbf{SBM}}} 
& HLoc(SP+SG)\cite{sarlin2019coarse,detone2018superpoint,sarlin2020superglue} & N/A  & \cmark & 0.3/0.10/100 & 0.6/0.25/100 & 0.4/0.27/100 & 0.7/0.20/100 & 0.6/0.13/100 & 0.5/0.14/97 & 2.1/0.66/73 \\
& HLoc(SP+NN)\cite{sarlin2019coarse,detone2018superpoint} & N/A & \cmark & 0.3/0.11/100 & 0.6/0.25/100 & 0.5/0.27/97 & 0.8/0.21/99 & 0.7/0.15/98 & 0.5/0.14/95 & 2.6/0.74/66 \\
& DSAC* \cite{brachmann2017dsac}        & N/A & \cmark & 0.5/0.17/100 & 0.8/0.28/99  & 0.5/0.34/100 & 1.2/0.34/98  & 1.2/0.28/99  & 0.7/0.21/97  & 2.7/0.78/92 \\
& ACE \cite{brachmann2023ace}                & N/A & \cmark & 0.5/0.18/      -     & 0.8/0.33/      -     & 0.5/0.33/      -     & 1.0/0.29/      -     & 1.0/0.22/      -     & 0.8/0.2/      -      & 2.9/0.81/      - \\

\midrule

\multirow{5}{*}{\parbox[c]{1.8cm}{\centering \textbf{PPM}}}
& GoMatch ~\cite{zhou2022geometry} & \cmark & \cmark & 4/1.65/      - & 13/3.86/      - & 9/5.17/      - & 11/2.48/      - & 16/3.32/      - & 13/2.84/      - & 89/21.12/      - \\
& DGC-GNN ~\cite{wang2024dgc}& \cmark & \cmark & 3/1.41/      - & 5/1.81/      - & 4/3.13/      - & 7/1.66/      - & 8/2.03/      - & 8/2.14/      - & 83/21.53/      - \\
& GSFF Privacy~\cite{pietrantoni2025gaussian} & \cmark &\xmark & 0.8/0.28/96 & 0.8/0.33/94 & 1.0/0.67/90 & 1.5/0.51/90 & 2.0/0.50/78 & 1.2/0.33/85 & 28.2/0.98/29 \\
& Random Lines~\cite{speciale2019queries}& \cmark & \cmark & 0.5/0.15/100 & 0.9/0.35/99 & 0.7/0.40/96 & 1.0/0.29/98 & 1.0/0.22/96 & 0.7/0.20/94 & 3.7/1.06/58 \\
& DCL (ours) & \xmark & \cmark & 1.0/0.36/97 & 2.0/0.76/81 & 1.3/0.77/83 & 2.0/0.61/83 & 2.1/0.48/80 & 1.4/0.42/84 & 26.3/6.59/13 \\

\bottomrule
\end{tabular}
\end{adjustbox}
\vspace{-2mm}
\caption{Evaluation of localization methods on the 7Scenes dataset~\cite{shotton2013scene}. We report median position error (cm.) (↓)/ median rotation error (°) (↓)/ recall at 5cm/5° (\%) (↑). Results are presented for Structure-Based Methods (SBM) and Privacy-Preserving Methods (PPM). 
}
\vspace{-4mm}
\label{tab:7scenes}
\end{table*}

\subsection{Localization Performance}
\label{sec:localization_results}

We evaluated the localization performance in terms of accuracy and real-time speed as shown in Table~\ref{tab:cambridge},\ref{tab:aachen},\ref{tab:7scenes}.
For comparison, we categorized cloud-based localization methods into structure-based (SBM) \cite{sarlin2019coarse, brachmann2023ace, GLACE2024CVPR} and privacy-preserving (PPM) for image queries (e.g., descriptor-free \cite{zhou2022geometry, wang2024dgc} and segmentation-based \cite{pietrantoni2023segloc, pietrantoni2025gaussian} approaches). 
The method's vulnerability to \emph{cloud-based} privacy attacks is indicated by the \textbf{P.V.} status (Privacy Vulnerability), and further details are in the supplementary material.
    

As shown in Table~\ref{tab:cambridge},\ref{tab:7scenes}, DCL demonstrates a practical balance between localization accuracy and real-time performance across various environments. On the indoor (7Scenes) and outdoor (Cambridge) datasets, DCL consistently achieves better accuracy than descriptor-free approaches (GoMatch~\cite{zhou2022geometry}, DGC-GNN~\cite{wang2024dgc}). While its accuracy is slightly lower than the segmentation-based methods (SegLoc~\cite{pietrantoni2023segloc}, GSFF Privacy~\cite{pietrantoni2025gaussian}), DCL offers an advantage in speed, achieving real-time inference approximately 4ms on 7Scenes and 6ms on Cambridge, highlighting its practicality. In contrast, GSFF Privacy reports much slower runtime of up to 45 seconds~\cite{pietrantoni2025gaussian} on the Cambridge dataset.
On the large-scale Aachen Day-Night dataset (Table~\ref{tab:aachen}), while the recall of DCL is lower than HLoc \cite{sarlin2019coarse} or Random Lines \cite{speciale2019queries}, it maintains superior accuracy against the recent learning-based methods such as ACE \cite{brachmann2023ace} and GLACE \cite{GLACE2024CVPR} in both Day and challenging Night scenarios.

We note that DCL is the only method that demonstrates robust resilience against the recent privacy attacks~\cite{chelani2024obfuscation, anonymous2025vulnerability, pittaluga2019revealing}, while maintaining practical runtime and  localization performance across diverse environments.

\begin{table}[t]
\centering
\renewcommand{\arraystretch}{1.2}
\setlength{\tabcolsep}{6pt}
\footnotesize
\begin{tabular}{ccc}
\toprule
\textbf{Distance between anchors} & \textbf{7Scenes} & \textbf{Cambridge} \\
\midrule
$H$  & \textbf{5.14} / 1.41 & \textbf{22.50} / \textbf{0.50} \\
$2H$   & \textbf{5.14} / 1.42 & 27.75 / 0.73 \\
$3H$  & 6.66 / \textbf{1.37} & 32.50 / 0.75 \\
\bottomrule
\end{tabular}
\vspace{-2mm}
\caption{
Ablation study of anchor point distance on localization accuracy. We report the average of median position (cm) / rotation (°) errors across the 7Scenes and Cambridge datasets. $H$ denotes the height of the image query.
}
\vspace{-4mm}
\label{tab:ablation}
\end{table}

\subsection{Ablation Study and Discussion}
\label{sec:ablation}
\paragraph{Ablation study of anchor locations.}
We conducted an ablation study to analyze the effect of the anchor distance on localization performance.
We compared our default DCL configuration (\ie $\mathbf{a}_1$$=$$(W/2, 0)$, $\mathbf{a}_2$$=$$(W/2, H)$) against two variants on the same midline with increased distance (2H and 3H).
Table~\ref{tab:ablation} shows that the default setting demonstrated the best accuracy, with performance degrading as the anchor distance increased. This suggests that farther anchors make the obfuscated lines more parallel, which adversely affects pose estimation.
        
\noindent\textbf{Resilience against server-side attack.} 
Now, we discuss the resilience of DCL against a more challenging server-side adversary~\cite{speciale2019queries}.
In this scenario, we assume the malicious server has additional information of 2D-3D correspondences after pose estimation and attempts to recover private content by leveraging the projections of inlier 3D keypoints.
As the keypoints from the private content (\eg people) tend to appear in the query but are not part of the server's 3D map, we refer to these features as \emph{outliers}.


To simulate this, we performed the geometry-recovery attack~\cite{chelani2024obfuscation}.
As shown in Fig.~\ref{fig:deformable_matching}, the attack fails to reconstruct meaningful outlier information.
Furthermore, we evaluated a stronger \textit{iterative} attack that strictly filters outliers by incorporating only points with dense inlier neighborhoods.
Our supplementary material demonstrates that even under this filtering, DCL induces error accumulation rather than accurate recovery.
Detailed algorithm configurations for both attacks, along with visual results for the iterative attack, are provided in the supplementary material.

\noindent\textbf{Limitation and future work.} 
A potential limitation occurs in rare degenerate cases where all query keypoints fall into a single region (Sec.~\ref{sec:localization}).
While our choice of a vertical midline is intended to reduce this risk by providing a better keypoint distribution, we note that this scenario was exceedingly rare in real (only 4/17,000 images in 7Scenes). 
Our future work includes research into adaptive partitions based on keypoints from structures (e.g., buildings) that are highly likely to match with the map to prevent such cases.



    

\begin{figure}[t]
    \centering
    \includegraphics[width=0.35\linewidth]{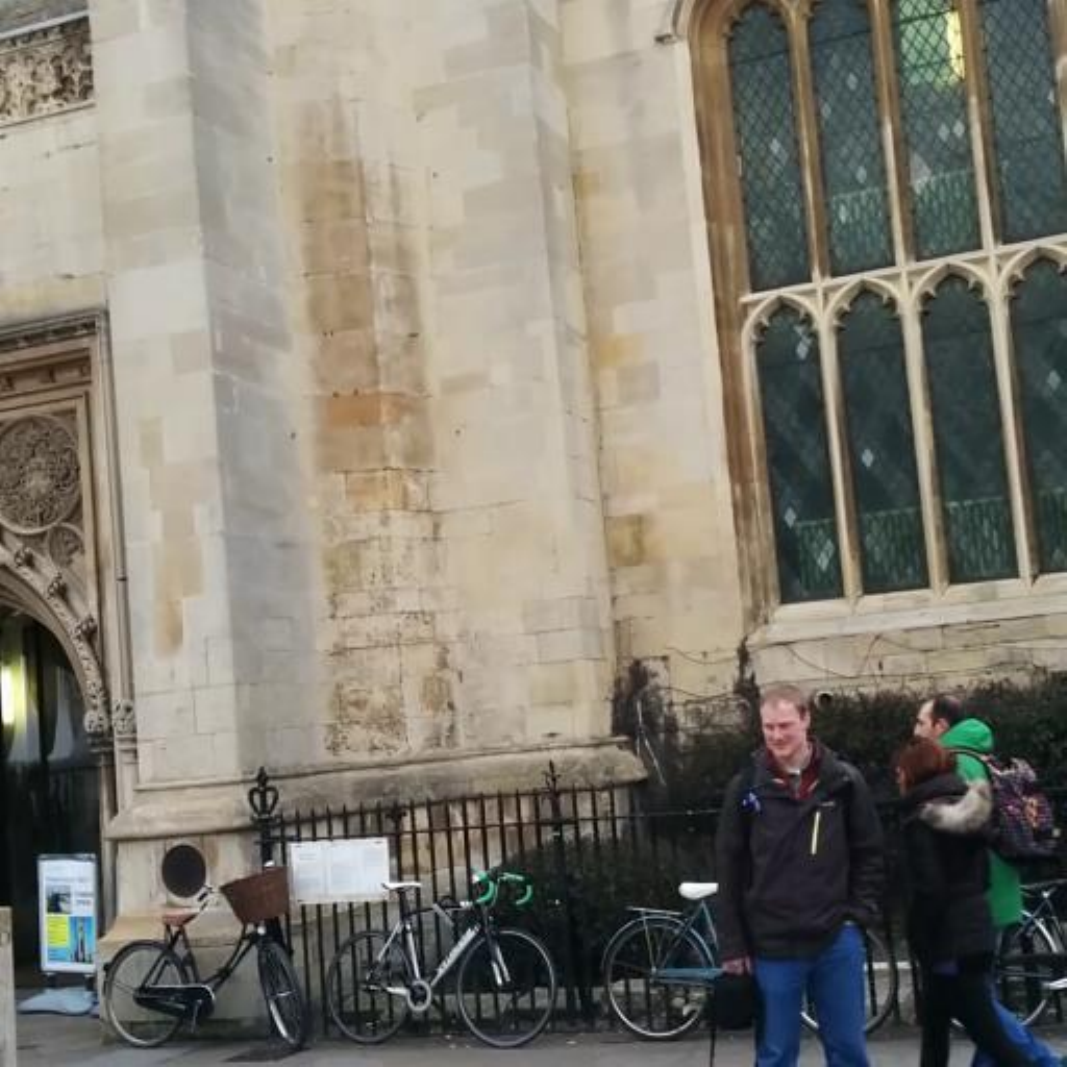}
    \includegraphics[width=0.35\linewidth]{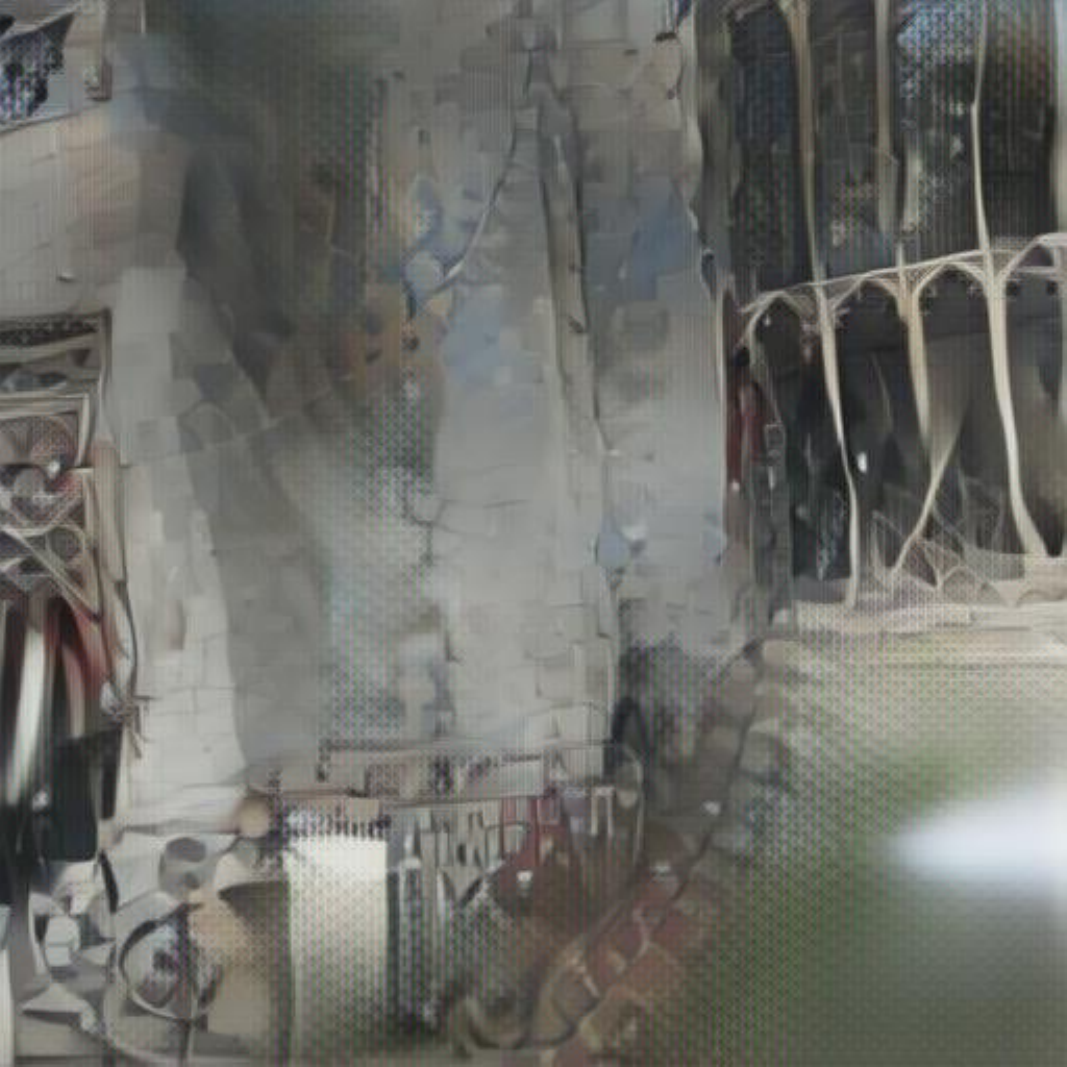}
    \vspace{-2mm}
    \caption{Example of server-side attack, where inlier feature positions are known. (Left) original image (Right) privacy-related contents (\eg people) are not revealed after the server-side attack.}
    \label{fig:deformable_matching}
    \vspace{-5mm}
\end{figure}

\vspace{-1mm}
\section{Conclusion}
\label{sec:conclusion}
We addressed a recent vulnerability in PPIQ where geometric obfuscation can be defeated by the neighborhood-based recovery attack. We proposed Dual Convergent Lines (DCL), a novel keypoint lifting strategy designed to be resilient to this class of attack. The core idea of our method is a dual-anchor geometry that makes the attacker's point-recovery optimization an ill-posed problem. We provided a theoretical analysis showing how this line distribution leads to two distinct failure modes.
Experiments demonstrate that DCL provides strong robustness against these privacy attacks, rendering recovered images unrecognizable. Since this scheme is still based on 2D lines, DCL achieves this while remaining compatible with efficient minimal solver, achieving practical localization performance.

\vspace{-4mm}
\paragraph{Acknowledgement} 
This work was supported by National Research Foundation of Korea(NRF) grants funded by the Korea government(MSIT) (No.RS-2025-16068784, 50\%; No.RS-2026-25486241, 25\%) and Institute of Information \& communications Technology Planning \& Evaluation(IITP) grant funded by the Korea government(MSIT) (No.RS-2020-II201373,  Artificial Intelligence Graduate School Program(Hanyang University), 25\%).

{
    \small
    \bibliographystyle{ieeenat_fullname}
    \bibliography{main}
}

\definecolor{cvprblue}{rgb}{0.21,0.49,0.74}


\def\httilde{\mbox{\tt\raisebox{-.5ex}{\symbol{126}}}}

\def\eqref#1{(\ref{eq:#1})}
\def\eqlabel#1{\label{eq:#1}}
\def\figref#1{\ref{fig:#1}}
\def\figlabel#1{\label{fig:#1}}
\def\pparagraph#1{\par{\bf #1}~~}

\def\x{{\mathbf x}}
\def\L{{\cal L}}

\newcommand{\hlc}[2][yellow]{{%
    \colorlet{foo}{#1}%
    \sethlcolor{foo}\hl{#2}}%
}
\def\xx#1{\textcolor{red}{xx: #1}}

\def\NoNumber#1{{\def\alglinenumber##1{}\State #1}\addtocounter{ALG@line}{-1}}

\definecolor{darkgreen}{RGB}{0,150,0}

\def\eqref#1{(\ref{eq:#1})}
\def\eqlabel#1{\label{eq:#1}}
\def\figref#1{\ref{fig:#1}}
\def\figlabel#1{\label{fig:#1}}
\def\pparagraph#1{\par{\bf #1}~~}
\def\httilde{\mbox{\tt\raisebox{-.5ex}{\symbol{126}}}}

\def\xcomment#1{\textcolor[rgb]{.3,.3,.1}{\text{$/\!\!/$ {\em #1}}}}
\def\comment#1{\kern-1cm\xcomment{#1}}
\def\eqcomment#1{\kern-1cm\xcomment{#1}}


\def\xx#1{\textcolor{red}{xx: #1}}

\def\m#1{\ensuremath{\mathtt{#1}}}
\def\mt#1{\ensuremath{\mathtt{\tilde{#1}}}}
\def\v#1{\ensuremath{\mathbf{#1}}}

\def\colspace#1{\mathrm{col}(#1)}
\def\nullspace#1{\mathrm{null}(#1)}
\def\localmin{\widehat{\min} }
\def\localargmin{{\arg\widehat{\min}}}

\def\symmx#1{{\mathbb S}^{#1}}
\def\Rmx#1#2{{\mathbb R}^{{#1}\times{#2}}}

\def\real{\mathbb{R}}
\def\tr{^\top}
\def\trinv{^{-\top}}
\def\inv#1{#1^\mathsf{-1}}
\def\pinv#1{#1^\mathsf{\dagger}}
\def\mupinvof#1{{{#1}^{-\lambda}}}
\def\expm{\mathrm{expm}}
\def\logm{\mathrm{logm}}

\def\hadamard{\odot}
\def\kron{\otimes}
\def\vec{\operatorname{vec}}
\def\unvec{\operatorname{unvec}}
\def\sym{\operatorname{sym}}

\def\norm#1{\left\lVert#1\right\rVert}
\def\fro#1{\norm{#1}_F}
\def\l2#1{\norm{#1}_2}
\def\nuclear#1{\norm{#1}_*}

\def\err{f}
\def\erru{f^*}

\def\vDu{{\v\Delta \vu}}
\def\vDv{{\v\Delta \vv}}
\def\vDx{{\v\Delta \v x}}
\def\vDy{{\v\Delta \v y}}

\def\vx{\v x}
\def\vhx{\hat{\v x}}
\def\vtx{\hat {\v x}}
\def\vy{\v y}
\def\vf{\v f}
\def\vp{\v p}
\def\vdx{\v {\Delta x}}
\def\vg{\v g}
\def\vgu{\vg^*}
\def\vgutr{\vg^{*\top}}
\def\vgup{\vgu_p}
\def\vgur{\vgu_r}
\def\vw{\v w}
\def\vm{\v m}
\def\vr{\v r}
\def\vru{\vr^*}
\def\vu{\v u}
\def\vup{\vu_p}
\def\vuperp{\vu_\perp}
\def\vuv{\v u^*}
\def\vv{\v v}
\def\vvu{\v v^*}
\def\hatvvu{{\hat{\v v}}^*}
\def\vx{\v x}
\def\vs{\v s}
\def\verr{{\boldsymbol\varepsilon}}
\def\verru{{\boldsymbol\varepsilon}^*}
\def\vdu{\partial \vu}
\def\vdv{\partial \vv}
\def\vdvu{\partial \vvu}
\def\dell{{\boldsymbol\delta}}

\def\set#1{\mathcal{#1}}

\def\mJu{\mJ_\vu}
\def\mJuu{\mJ_\vu^*}
\def\mJv{\m J_\vv}
\def\tmJv{\tilde{\m J}_\vv}
\def\mPv{\m P_\vv}
\def\mQv{\m Q_\vv}

\def\mA {\m A}
\def\tildemJu{\tilde\mJ^*}
\def\mJutr{\mJ^{*\top}}
\def\mU{\m U}
\def\mUperp{\mU_{\perp}}
\def\mUp{\mU_p}
\def\mUr{\mU_r}
\def\mUu{\mU^*}
\def\mVu{\mV^*}
\def\mVutr{\mV^{*\top}}
\def\mW{\m W}
\def\mM{\m M}
\def\mS{\m S}
\def\dvdu{\frac{d\vvu}{d\vu}}
\def\mdU{\partial\mU}
\def\mdV{\partial\mV}
\def\mdVu{\partial\mVu}

\def\cU{{\cal U}}
\def\cV{{\cal V}}

\def\twiddle#1{{\tilde{#1}}}

\maketitlesupplementary
\appendix

In this supplementary material, we first provide the detailed mathematical derivation for Proposition~\ref{prop:weighted_average_supp} in Sec.~\ref{sec:proof_prop1}. 
We then detail the architectures of the inversion models used in our experiments in Sec.~\ref{sec:inversion_impl}. 
Following this, we present extended qualitative evaluations demonstrating the robustness of DCL against inversion attacks in Sec.~\ref{sec:qual_results}. 
Subsequently, we clarify the system architecture of our privacy-preserving localization framework in Sec.~\ref{sec:loc_details}. 
In Sec.~\ref{sec:discussion}, we discuss the degree of privacy spectrum across various paradigms.
In Sec.~\ref{sec:impl_details} and Sec.~\ref{sec:add_results}, we provide implementation details and additional experimental analyses, including geometric degeneracy and ablation studies.
Finally, we detail the baseline configurations and practical limitations of iterative server-side attacks in Sec.~\ref{sec:supp_server_attack}.

\section{Mathematical Proofs of Proposition 1}
\label{sec:proof_prop1}

We provide the detailed mathematical proof for Proposition~\ref{prop:weighted_average_supp}, which establishes that the recovered keypoint parameter $t_i^*$ is a weighted average of individual intersection parameters $\{t^*_{i,j}\}$.

The attack's cost function (in the main paper Eq. (3)), seeks to find the point $\mathbf{x}_i^*=\mathbf{\hat x}_i$ on the target line $\mathbf{l}_i$ that minimizes the sum of squared point-to-line distances:
\begin{equation}
    f(\mathbf{\hat x}_i) = \sum_{\mathbf{l}_j \in  \mathcal{N}(\mathbf{l}_i)} d(\mathbf{l}_j, \mathbf{\hat{x}}_i)^2.
\end{equation}
We parameterize the point $\mathbf{\hat x}_i$ along the line $\mathbf{l}_i$ using the scalar parameter $t_i$, where the true keypoint $\mathbf{x}_i$ is assumed to be at $t_i=0$. We assume that $\mathbf{v}_i$ is a normalized direction vector (i.e., $\|\mathbf{v}_i\|=1$):
\begin{equation}
    \mathbf{\hat x}_i(t_i) = \mathbf{x}_i + t_i \mathbf{v}_i.
    \label{eq:line_param_supple}
\end{equation}
By substituting Eq.~\eqref{line_param_supple} into the cost function, $f(\mathbf{\hat x}_i)$ becomes a function of $t_i$, which we denote as $f(\mathbf{\hat x}_i(t_i))$.

\begin{proposition}\label{prop:weighted_average_supp}
For the inter-anchor scenario, the minimized cost function $f(\mathbf{\hat{x}}_i(t_i))$ is convex quadratic. The recovered point parameter $t_i^*$ is the weighted average of the individual parameters $\{t_{i,j}^*\}$ of the intersection points between the target line $\mathbf{l}_i$ and each neighbor line $\mathbf{l}_j$:
\begin{equation}
    t_i^* = \frac{\sum_j (w_{i,j} t_{i,j}^*)}{\sum_j w_{i,j}}.
    \label{eq:t_star_weighted_supple}
\end{equation}
\end{proposition}

\begin{proof}
The point-to-line distance $d(\mathbf{l}_j, \mathbf{\hat{x}}_i)$ is calculated using the Euclidean distance in 2D. For convenience in the derivation below, we adopt an abuse of notation where the squared distance $d(\mathbf{l}_j, \mathbf{\hat{x}}_i)^2$ is expressed using the squared norm of a 3D cross product. This is valid by treating all 2D vectors ($\mathbf{x} \in \mathbb{R}^2$, including $\mathbf{v}_i$, $\mathbf{v}_j$, and $\mathbf{a}_2$) as 3D vectors with a zero z-component (i.e., $\mathbf{x} \rightarrow (\mathbf{x}, 0)$). The direction vector $\mathbf{v}_j$ is also assumed to be normalized ($\|\mathbf{v}_j\|=1$).

We consider the contribution of a single neighbor line $\mathbf{l}_j$, which passes through an anchor $\mathbf{a}_2$ with direction $\mathbf{v}_j$. The squared point-to-line distance, $f_j(\mathbf{\hat x}_i(t_i)) = d(\mathbf{l}_j,\mathbf{\hat x}_i(t_i) )^2$, is calculated using the cross product:
\begin{align}
    f_j(\mathbf{\hat x}_i(t_i)) &= \| (\mathbf{\hat x}_i(t_i) - \mathbf{a}_2) \times \mathbf{v}_j \|^2 \nonumber \\
    \intertext{Substituting $\mathbf{\hat x}_i(t_i) = \mathbf{x}_i + t_i \mathbf{v}_i$:}
    f_j(\mathbf{\hat x}_i(t_i)) &= \| ( (\mathbf{x}_i - \mathbf{a}_2) + t_i \mathbf{v}_i ) \times \mathbf{v}_j \|^2 \nonumber \\
    \intertext{Using the distributive property of the cross product $\mathbf{v}_i \times \mathbf{v}_j$ and expanding the squared norm, we obtain the quadratic form $f_j(\mathbf{\hat x}_i(t_i)) = A_{i,j} t_i^2 + B_{i,j} t_i + C_{i,j}$, where:}
    A_{i,j} &= \| \mathbf{v}_i \times \mathbf{v}_j \|^2 \label{eq:A_ij_supple} \\
    B_{i,j} &= 2 \left( (\mathbf{x}_i - \mathbf{a}_2) \times \mathbf{v}_j \right) \cdot \left( \mathbf{v}_i \times \mathbf{v}_j \right) \\
    C_{i,j} &= \| (\mathbf{x}_i - \mathbf{a}_2) \times \mathbf{v}_j \|^2.
\end{align}
The total cost $f(\mathbf{\hat x}_i(t_i)) = \sum_j f_j(\mathbf{\hat x}_i(t_i)) = (\sum_j A_{i,j}) t_i^2 + (\sum_j B_{i,j}) t_i + (\sum_j C_{i,j})$ is also a quadratic and convex function. Since $A_{i,j} = \| \mathbf{v}_i \times \mathbf{v}_j \|^2 \geq 0$ for all $j$, the coefficient of the quadratic term, $\sum_j A_{i,j}$, is non-negative. 
The parameter $t_i^*$ that minimizes $f(\mathbf{\hat x}_i(t_i))$ satisfies $\frac{df}{dt_i}(t_i^*) = 0$:

\begin{equation}
    t_i^* = -\frac{\sum_j B_{i,j}}{2 \sum_j A_{i,j}}.
\end{equation}
Since the individual optimum $t^*_{i,j}$ minimizes $f_j(\mathbf{\hat x}_i(t_i))$, we have $t^*_{i,j} = -B_{i,j} / (2A_{i,j})$, which implies $B_{i,j} = -2 A_{i,j} t^*_{i,j}$. Substituting this relationship yields the weighted average form:
\begin{equation}
    t_i^* = \frac{\sum_j (A_{i,j} t^*_{i,j})}{\sum_j A_{i,j}}.
\end{equation}
By setting the weights $w_{i,j} = A_{i,j}$ (Eq.~\eqref{A_ij_supple}), we complete the proof of the proposition.
\end{proof}

\begin{corollary}[\textbf{Corollary 1.1 from Main Paper}]
The weight $w_{i,j}$, which is defined as $A_{i,j}$ in Proposition~\ref{prop:weighted_average_supp} (Eq.~\eqref{A_ij_supple}), for each offset contribution $t_{i,j}^*$ is determined solely by the angle $\theta_{i,j}$ between the target line $\mathbf{l}_i$ and its neighboring line $\mathbf{l}_j$. This relationship is established by applying the geometric definition of the cross product to the term $A_{i,j}$:
\begin{equation}
    w_{i,j} = A_{i,j} = \| \mathbf{v}_i \times \mathbf{v}_j \|^2 = \sin^2(\theta_{i,j}).
\end{equation}
\end{corollary}

\section{Inversion Model Configuration and Setup}
\label{sec:inversion_impl}
We present inversion experiments for two representative descriptors: SIFT~\cite{lowe2004sift} and SuperPoint~\cite{detone2018superpoint}. We selectively applied an inversion model for each descriptor. 
For SIFT, we utilized InvSfM~\cite{pittaluga2019revealing}, a widely adopted baseline in Privacy-Preserving Visual Localization with publicly available pre-trained weights. Designed to reconstruct images from 3D SfM point clouds, InvSfM employs a three-stage cascaded architecture: VisibNet for visibility estimation, CoarseNet for initial reconstruction, and RefineNet for final enhancement. Although originally trained on projected 3D point clouds, the model is capable of performing inversion using 2D query feature maps extracted from single images.

For SuperPoint, since no compatible pre-trained weights exist, we trained a new model using the U-Net-based inversion architecture adopted in~\cite{ng2022ninjadesc,dangwal2021analysis}. Unlike InvSfM, which utilizes a complex cascaded structure often aided by auxiliary depth or RGB inputs, we employ a simplified 2D U-Net trained on SuperPoint keypoints and descriptors~\cite{detone2018superpoint}. This strictly adheres to the cloud-based localization scenario where raw images are not transmitted to the server.


\section{Extended Evaluation}
\label{sec:qual_results}
\paragraph{Robustness against SuperPoint inversion.}
We provide uncurated qualitative results for a stress-test using SuperPoint features with an expanded neighborhood (K=100), where DCL's robustness continues to hold (see Fig.~\ref{fig:SP_K100_privacy_qualitative}). Additionally, Fig.~\ref{fig:SP_zoom_out_qualitative_results} visualizes the inversion network's output. This image shows the result of feeding the recovered keypoint locations, which are captured from a ×2 zoomed-out, at the K=20 neighborhood setting, into the network. This zoomed-out view is necessary as most points of DCL are recovered far outside the original image boundary. This visualization provides clear empirical evidence of our method's robustness against the inversion attack.

\paragraph{Robustness against SIFT inversion.}
We also provide qualitative results to visually validate DCL's robustness against the geometry-recovery attack~\cite{chelani2024obfuscation}, conducted in the oracle setting. Fig.~\ref{fig:sift_privacy_qualitative} shows the inversion results for SIFT features (K=20), corresponding to the quantitative metrics (Table~\ref{tab:sift_privacy_quantitative}). These results confirm that while Random Lines~\cite{speciale2019queries} and Coordinate Permutation~\cite{pan2023permut} yield recognizable inversions, DCL demonstrates consistent robustness against the attack by rendering the output unrecognizable.

\renewcommand{\arraystretch}{1.2}

%





\begin{table}[t]
\centering
\scriptsize

\begin{tabular}{c|c|c|c|c|c}
\hline
        &         & Feature  & Random & Coordinate & DCL\\
Dataset & Metrics &   points & lines~\cite{speciale2019queries} & permut.~\cite{pan2023permut} & (ours)\\
\hline

\multirow{3}{*}{7-scenes} 
& PSNR(\textcolor{blue}{↓}) & 15.38 & 14.30 & 13.83 & \textbf{11.02} \\
& SSIM(\textcolor{blue}{↓}) & 0.604 & 0.531 & 0.509 & \textbf{0.481} \\
& LPIPS(\textcolor{darkgreen}{↑})& 0.540 & 0.595 & 0.614 & \textbf{0.720} \\
\hline

\multirow{3}{*}{Cambridge} 
& PSNR(\textcolor{blue}{↓}) & 18.59 & 17.00 & 16.29 & \textbf{9.765} \\
& SSIM(\textcolor{blue}{↓}) & 0.674 & 0.556 & 0.505 & \textbf{0.308} \\
& LPIPS(\textcolor{darkgreen}{↑})& 0.445 & 0.512 & 0.549 & \textbf{0.801} \\
\hline

\multirow{3}{*}{Aachen} 
& PSNR(\textcolor{blue}{↓}) & 18.43 & 17.16 & 16.44 & \textbf{9.73} \\
& SSIM(\textcolor{blue}{↓}) & 0.602 & 0.488 & 0.425 & \textbf{0.248}\\
& LPIPS(\textcolor{darkgreen}{↑})& 0.406 & 0.470 & 0.520 & \textbf{0.806} \\
\hline

\end{tabular}
\vspace{-2mm}
\caption[Result-inversion-attacks]
{Mean quality metrics for images reconstructed from SIFT~\cite{lowe2004sift} features using the InvSfM~\cite{pittaluga2019revealing} network.}
\label{tab:sift_privacy_quantitative}
\vspace{-5mm}
\end{table}

\section{Privacy Preserving Localization Framework}
\label{sec:loc_details}
To mitigate potential privacy risks, DCL adopts a privacy-preserving(P.P.) architecture where the raw query image is never transmitted to the server. Consequently, the extraction of all visual features is performed on the client-side. The client computes the image retrieval features, keypoint descriptors, and DCL features, and transmits only these compact representations to the server.

Upon receiving the features, the server executes the standard localization pipeline following HLoc~\cite{sarlin2019coarse}. Specifically, the server performs image retrieval against the database using the received global features, followed by feature matching and 6-DoF pose estimation.

Importantly, we employ this hierarchical pipeline, which includes an image retrieval step, primarily to ensure a fair comparison with the non-privacy-preserving baseline~\cite{sarlin2019coarse}. Furthermore, our proposed geometry obfuscation method is agnostic to the localization pipeline. It can be readily applied to architectures that omit the image retrieval stage, such as direct 2D-3D matching against a global map.

\section{Discussions on Privacy Preserving Paradigms}
\label{sec:discussion}
Privacy in image-query obfuscation is not binary, but depends on the adversary model and the amount of recoverable information. 

\paragraph{Learning-based:} Recently, learning-based localization methods (\eg Scene Coordinate Regression (SCR)~\cite{brachmann2023ace, GLACE2024CVPR, brachmann2017dsac, brachmann2021dsacstar}), which directly map image pixels to 3D world coordinates,  have achieved impressive localization performance.
While this implicit map representation could offer privacy in a client-based (map distribution) setting, it poses a high risk for cloud-based visual localization. 
This approach typically requires uploading the raw query image (or query features) for inference\begingroup\renewcommand\thefootnote{\fnsymbol{footnote}}\footnotemark[2]\endgroup, 
making it unsuitable as a Privacy-Preserving Image Query (PPIQ) solution.

\paragraph{Descriptor-free:} A common limitation of descriptor-free approaches is their inability to conceal keypoint positions, leaving the 2D layout vulnerable to image content inference.
As shown by the original InvSfM work (cf. Fig.6 in~\cite{pittaluga2019revealing}), a fair-well image can be reconstructed from keypoint locations and color information even without descriptors, which mainly add textural details.

\paragraph{Segmentation-based obfuscation:} A recent diffusion-based attack~\cite{anonymous2025vulnerability} shows that segmentation-based methods are not immune to leakage. However, they still provide a higher degree of privacy than geometric obfuscation methods~\cite{speciale2019queries,pan2023permut} by more effectively obscuring fine details.

\paragraph{Prior geometry-based obfuscation:} Prior geometry-based schemes~\cite{speciale2019queries,pan2023permut} may still retain a degree of privacy depending on the recovery conditions, although the evaluation in our main paper focuses on the worst-case upper-bound setting of~\cite{chelani2024obfuscation}.

\begingroup
\renewcommand\thefootnote{\fnsymbol{footnote}}
\footnotetext[2]{A notable example is Niantic's Lightship VPS, a cloud-based service that sends raw camera frames to the cloud~\cite{niantic_vps_docs}. It uses ACE~\cite{brachmann2023ace}, a production-scale SCR relocalizer, in its operations~\cite{niantic_ace_blog}.}
\endgroup    

\begin{table}[t!]
\centering
\renewcommand{\arraystretch}{1.2}
\setlength{\tabcolsep}{1pt}
\scriptsize
\begin{tabular}{c|ccc|ccc}
\toprule
Distance between anchors & \multicolumn{3}{c|}{\textbf{Aachen Day}} 
& \multicolumn{3}{c}{\textbf{Aachen Night}} \\
\cmidrule(lr){2-4} \cmidrule(lr){5-7}
& 0.25m,2° & 0.5m,5° & 5m,10° 
  & 0.25m,2° & 0.5m,5° & 5m,10° \\
\midrule

$H$  & \textbf{41.0}&\textbf{57.9}&\textbf{72.7}&13.3&\textbf{21.4}&\textbf{38.8} \\
$2H$ & 34.1&53.2&71.5&\textbf{14.3}&20.4&37.8 \\
$3H$ & 26.8&49.5&70.3&12.2&24.5&37.8 \\
\bottomrule
\end{tabular}
\caption{Ablation study of anchor point distance on localization accuracy on the Aachen Day \& Night dataset. 
We report the recall at thresholds 0.25m/2°, 0.5m/5°, and 5m/10° (\%). $H$ denotes the height of the image query.}
\label{tab:aachen_supp}
\end{table}

\section{Implementation Details}
\label{sec:impl_details}
\paragraph{Dataset and scene selection.}
Consistent with prior works~\cite{zhou2022geometry, pietrantoni2025gaussian}, our evaluation excludes the Street and Great Court scenes from the Cambridge dataset~\cite{kendall2015posenet} because of their excessive geometric outliers.

\paragraph{Pose estimation parameters.}
For the RANSAC-based pose estimation, we configure the inlier reprojection error thresholds on a per-dataset basis. The 2D keypoint thresholds ($\epsilon_{pt}$) are set to 4$px$ for 7Scenes, 18$px$ for the Cambridge dataset, and 8$px$ for Aachen. Consequently, the inlier thresholds for 2D lines ($\epsilon_{line}$) are defined as their corresponding point-based values divided by $\sqrt{2}$ (i.e., $\epsilon_{pt}/\sqrt{2}$), resulting in $4/\sqrt{2}$$px$, $18/\sqrt{2}$$px$, and $8/\sqrt{2}$$px$, respectively.


\paragraph{Baseline framework and fair comparison.}To ensure a rigorous and fair evaluation, we re-implemented the baseline methods, including HLoc~\cite{sarlin2019coarse} and Random Lines~\cite{speciale2019queries}, within our unified testing framework using PoseLib~\cite{PoseLib}.
Applying the same Lo-RANSAC loop across all methods, including DCL and the baselines~\cite{sarlin2019coarse,speciale2019queries}, guarantees an evaluation under identical conditions, enabling a direct and equitable comparison of accuracy and recall performance.

Consistent with prior works, all experiments reported on 7Scenes, Cambridge Landmarks, and Aachen utilize Structure-from-Motion (SfM) based pseudo ground truth (pGT) poses~\cite{brachmann2021limits} to ensure fair comparison across methods.

\paragraph{Exclusion of baselines.}
The pose estimation results for the coordinate permutation~\cite{pan2023permut} are omitted from our comparison. The official implementation for this method is not publicly available, and our own efforts to replicate the results reported in their paper were unsuccessful.

\section{Additional Experiments and Analysis}
\label{sec:add_results}
\paragraph{Partition boundary analysis.}
Note that while we use a vertical partition for simplicity, the scheme is fundamentally applicable to any dividing line (e.g., horizontal). 
We conducted an empirical analysis by setting the boundary as a horizontal midline that crosses the center of the image.
The experiment revealed a significant increase in degenerate cases, which grew from 4 to 39 in the 7Scenes dataset. 
This increased degeneracy stems from the inherent nature of indoor environments, which frequently possess textureless ceilings and floors. A horizontal boundary risks isolating these sparse areas into a single region, leading to geometric instability. Conversely, our vertical midline intersects the feature-rich central band, ensuring a more balanced keypoint distribution.


\paragraph{Effect of anchor length on Aachen dataset.}
Due to space constraints in the main manuscript, the ablation study on the impact of anchor length for the Aachen Day-Night dataset is presented here. Consistent with the observations on the 7Scenes and Cambridge Landmarks datasets (in the main paper), the Aachen dataset achieves the best localization performance when the anchor length is set to the full image height ($H$) as shown in Table.~\ref{tab:aachen_supp}.

\begin{figure}[t]
\tiny
\centering 
\subfloat{\includegraphics[width=1.01\linewidth]{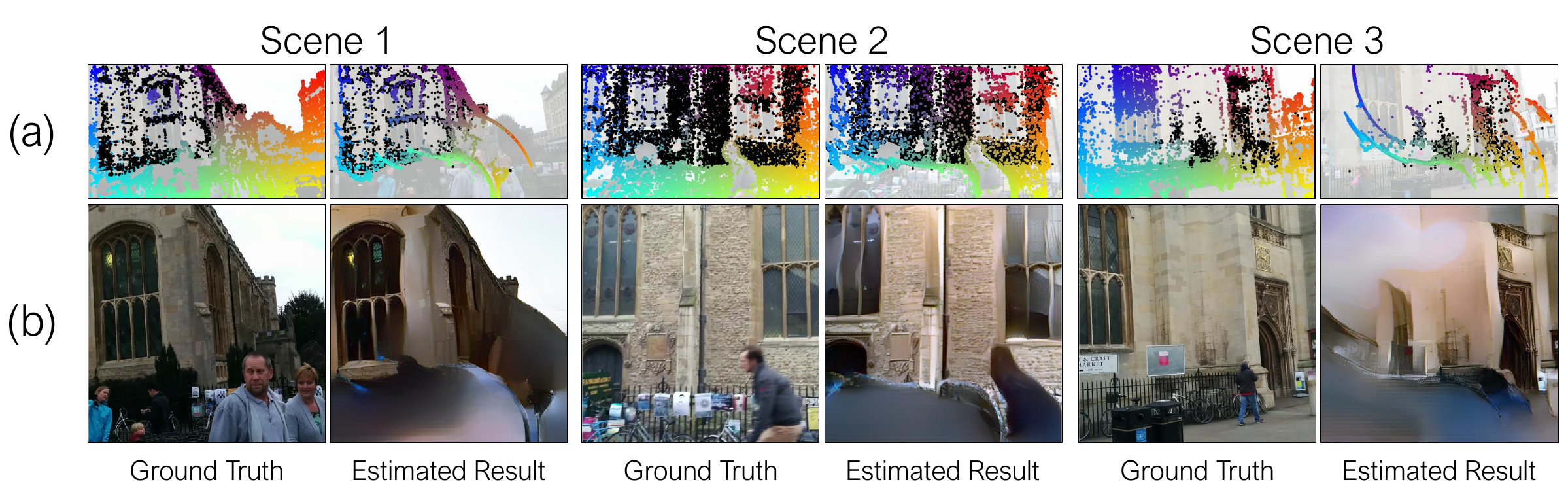}}

\vspace{-3mm}
\caption{Visualization of the iterative server-side attack on three different scenes.
In (a), black points indicate inlier keypoints from the server map, while colored points indicate ground truth (user's) query keypoints (left) or their estimated locations after running the iterative attack on DCL (right). For each scene, same colored points in (left) and (right) indicate the same keypoint.
In (b),  the original query image (left) and the image revealed from the estimated points (right) are shown for each scene. 
}
\label{fig:iterative_server_side}
\vspace{-3mm}
\end{figure}

\section{Details on Server-Side Attacks}
\label{sec:supp_server_attack}

\paragraph{Baseline server-side attack.}
As introduced in the main paper, we consider a challenging server-side adversary~\cite{speciale2019queries} that leverages 2D-3D correspondences.
To simulate the server-side attack, we first performed the geometry-recovery attack~\cite{chelani2024obfuscation}, setting K=20 nearest neighbors for each outlier. 
Note, we assumed perfect identification of these neighbors. 
Then, if the neighborhood contains projected inlier points, we used a point-to-point distance instead of point-to-line distance in Eq.~(3) in the main paper~\cite{authors26main}.
If the inilers are not in the neighborhood, the server can still proceed with recovery using only the obfuscated line neighbors as in Sect.~4.2 of the main text~\cite{authors26main}.

\paragraph{Iterative server-side attack and its practical limitations.}
To evaluate robustness against stronger adversaries, we implemented an iterative server-side attack. 
In each step, we estimate the 2D points from the 2D lines which are anticipated to comprise the biggest number of neighboring (server-side) inliers, and add them to the inlier set for the next iteration.

Since nearby obfuscated 2D lines exhibit similar directions, their estimated keypoints are usually found at similar locations provided they share a large overlap of server-side inliers.
We observe this effect as well as accumulation of errors lead to a \textit{spiral-like} drift as shown in  Fig.~\ref{fig:iterative_server_side}, preventing accurate recovery.
Quantitatively, the mean point estimation errors for the iterative attack in Scenes 1, 2, and 3 are 163, 55, and 122 pixels, respectively.
This demonstrates the structural advantage of DCL rather than an outlier artifact of the recovery algorithm~\cite{chelani2024obfuscation}. 

Furthermore, it is important to note that this attack assumes 100\% accurate neighborhood estimation between server-side inliers and \textit{outliers} defined as keypoints from dynamic, privacy-sensitive objects.
Since \cite{chelani2024obfuscation} relies on local pattern matching, forming accurate inlier-outlier neighborhoods can be inherently challenging due to several factors, including i) inliers being potentially geometrically far from outliers, ii) the transient nature of privacy-sensitive objects (\eg motion or occlusions) changing local patterns and the matching result, and 
iii) variations in SfM hyperparameters that significantly affect inlier keypoint distribution, potentially leading to poor generalization of attack models similar to~\cite{chelani2024obfuscation}.




\begin{figure*}[t]
\centering

    \setcounter{subfigure}{0}
    \subfloat{
        \includegraphics[width=0.19\linewidth]{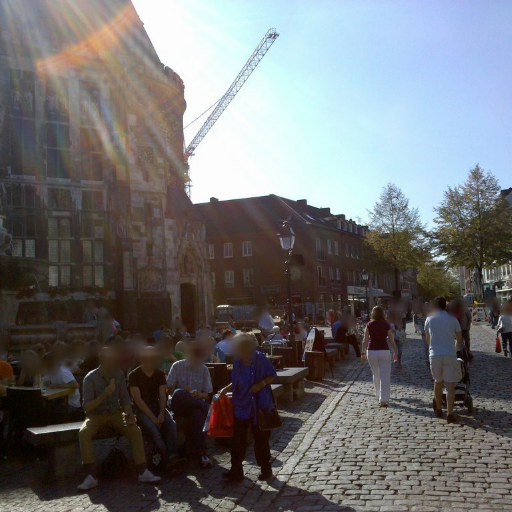}
    }
    \subfloat{
        \includegraphics[width=0.19\linewidth]{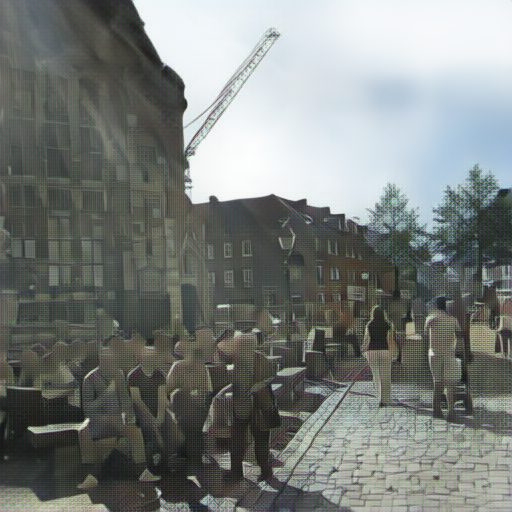}
    }
    \subfloat{
        \includegraphics[width=0.19\linewidth]{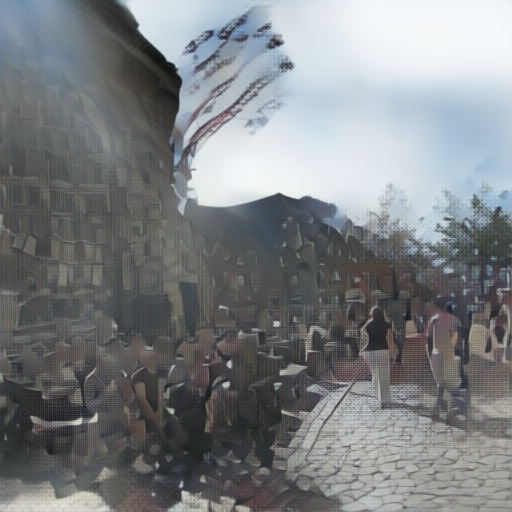}
    }
    \subfloat{
        \includegraphics[width=0.19\linewidth]{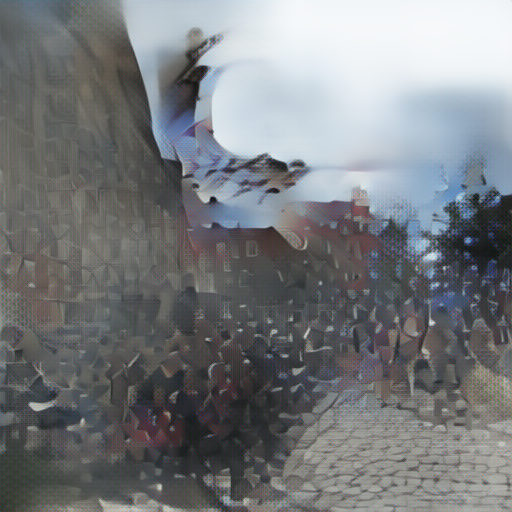}
    }
    \subfloat{
        \includegraphics[width=0.19\linewidth]{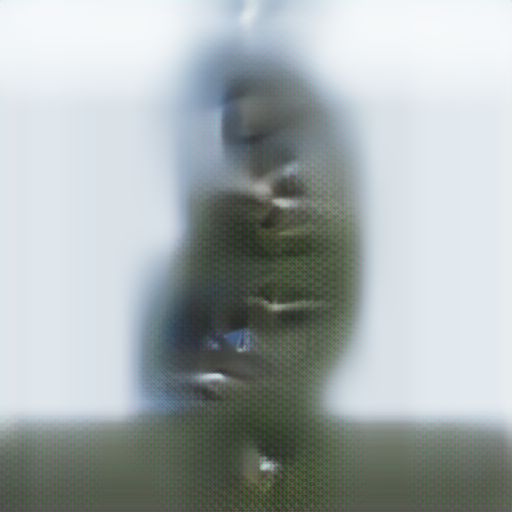}
    }
    \\
    \vspace{1mm}
    
    \setcounter{subfigure}{0}
    \subfloat{
        \includegraphics[width=0.19\linewidth]{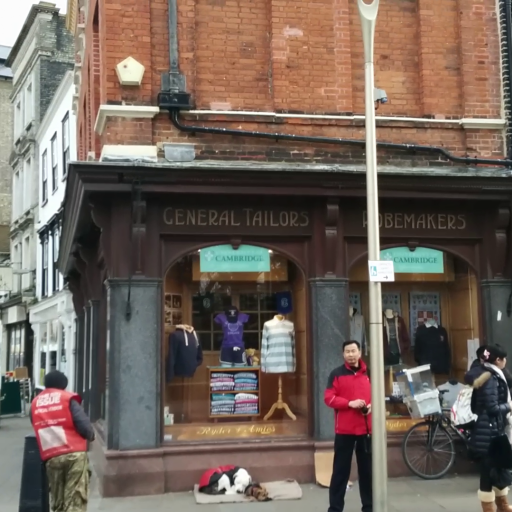}
    }
    \subfloat{
        \includegraphics[width=0.19\linewidth]{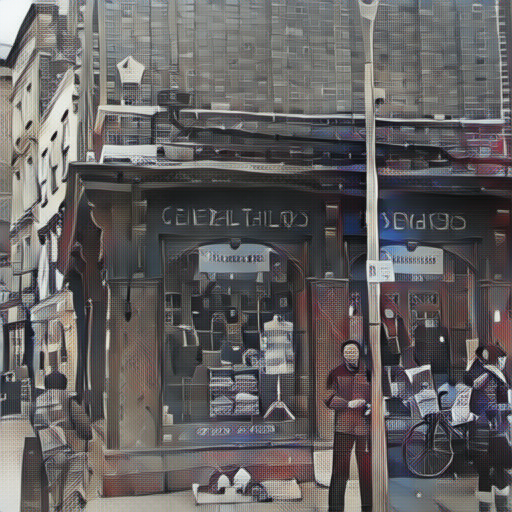}
    }
    \subfloat{
        \includegraphics[width=0.19\linewidth]{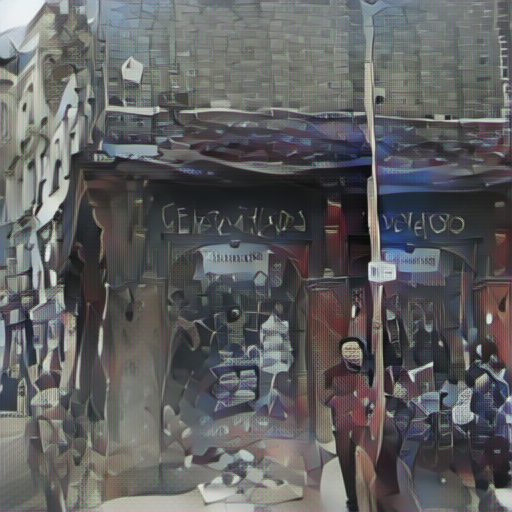}
    }
    \subfloat{
        \includegraphics[width=0.19\linewidth]{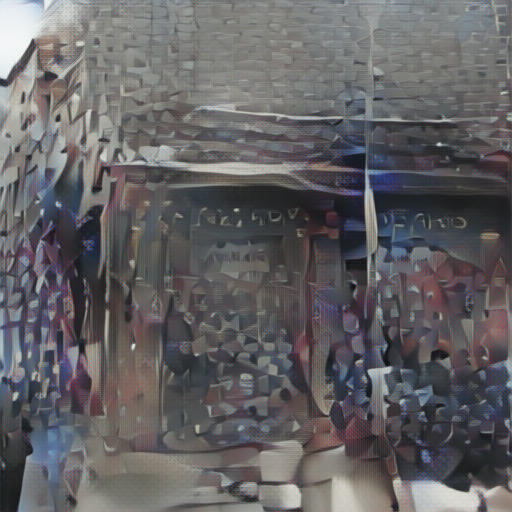}
    }
    \subfloat{
        \includegraphics[width=0.19\linewidth]{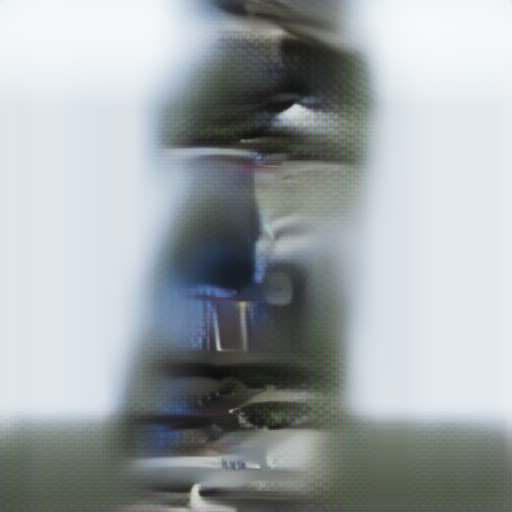}
    }
    \\
    \vspace{1mm}
    
    \setcounter{subfigure}{0}
    \subfloat{
        \includegraphics[width=0.19\linewidth]{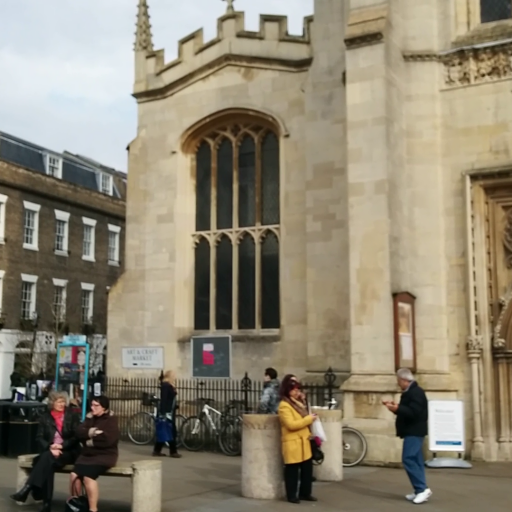}
    }
    \subfloat{
        \includegraphics[width=0.19\linewidth]{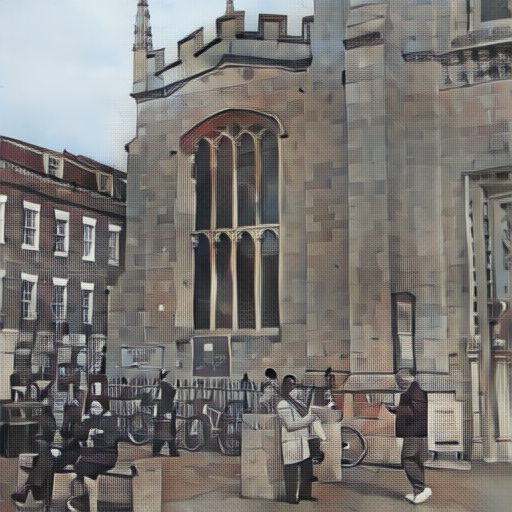}
    }
    \subfloat{
        \includegraphics[width=0.19\linewidth]{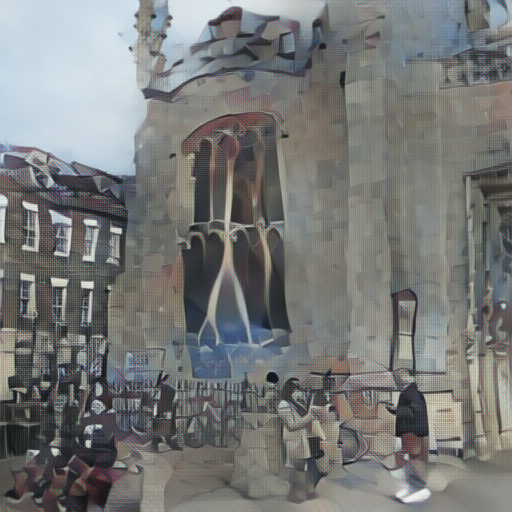}
    }
    \subfloat{
        \includegraphics[width=0.19\linewidth]{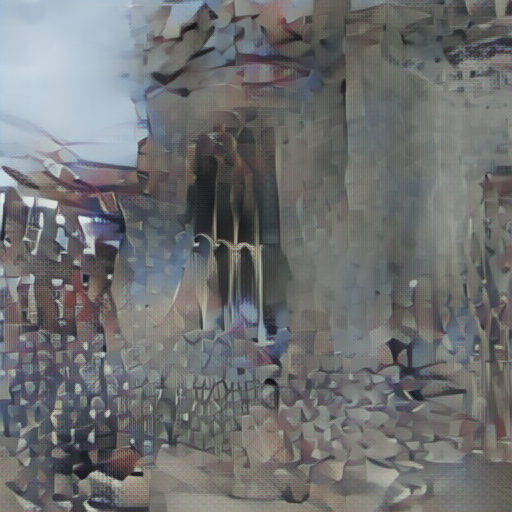}
    }
    \subfloat{
        \includegraphics[width=0.19\linewidth]{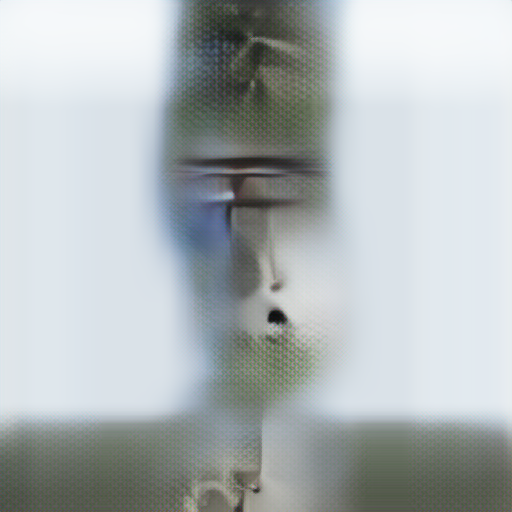}
    }
    \\
    \vspace{1mm}
    
    \setcounter{subfigure}{0}
    \subfloat{
        \includegraphics[width=0.19\linewidth]{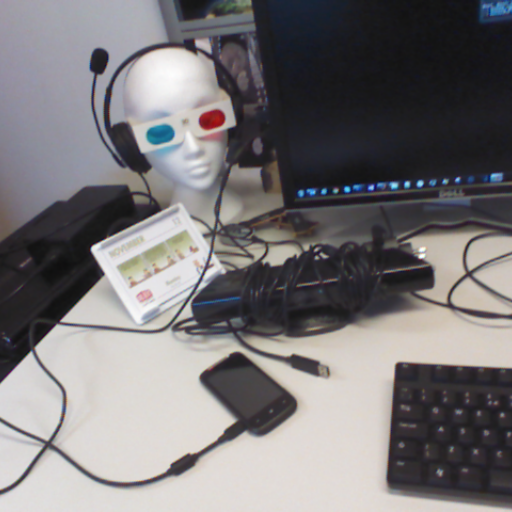}
    }
    \subfloat{
        \includegraphics[width=0.19\linewidth]{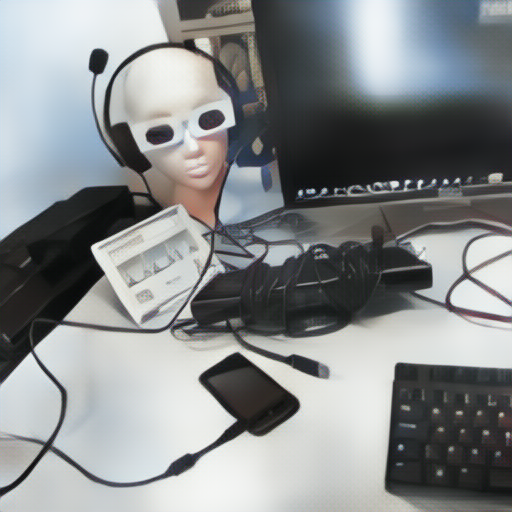}
    }
    \subfloat{
        \includegraphics[width=0.19\linewidth]{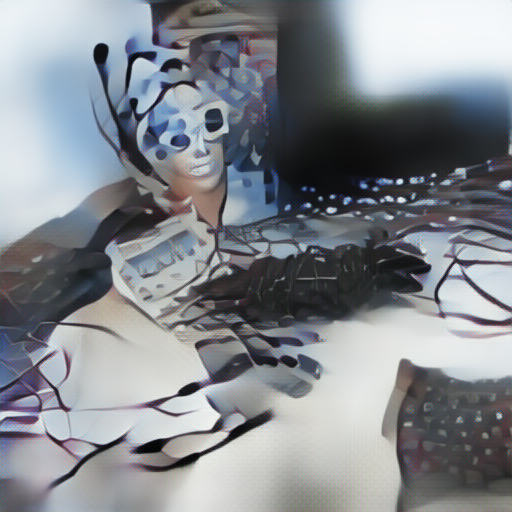}
    }
    \subfloat{
        \includegraphics[width=0.19\linewidth]{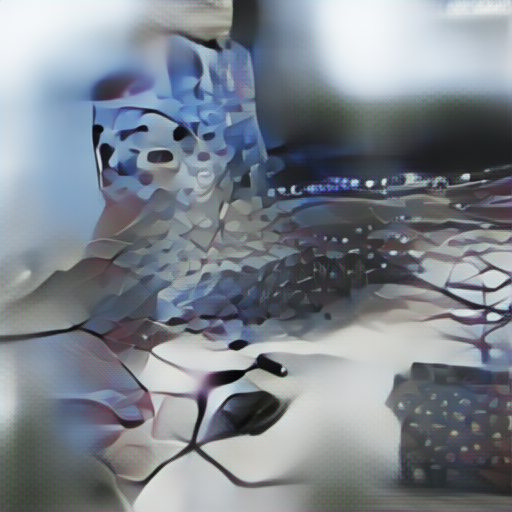}
    }
    \subfloat{
        \includegraphics[width=0.19\linewidth]{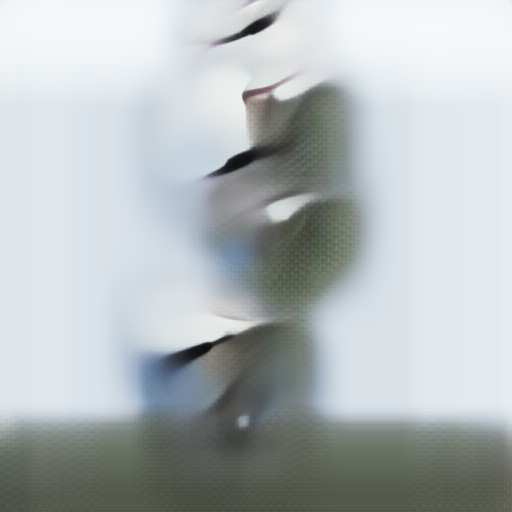}
    }
    \\
    \vspace{1mm}

    \setcounter{subfigure}{0}
    \subfloat[][Original Image]{
        \includegraphics[width=0.19\linewidth]{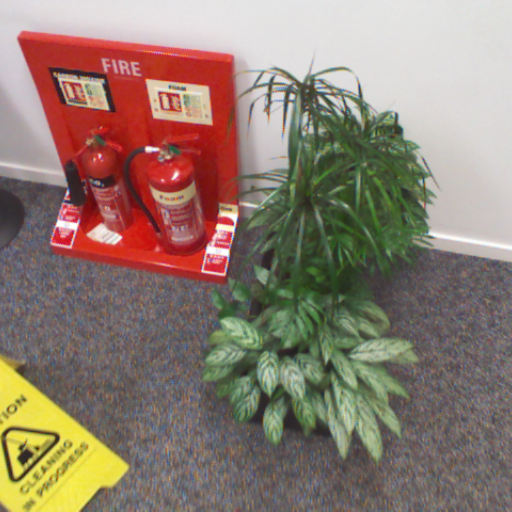}
    }
    \subfloat[][Feature Points~\cite{lowe2004sift}]{
        \includegraphics[width=0.19\linewidth]{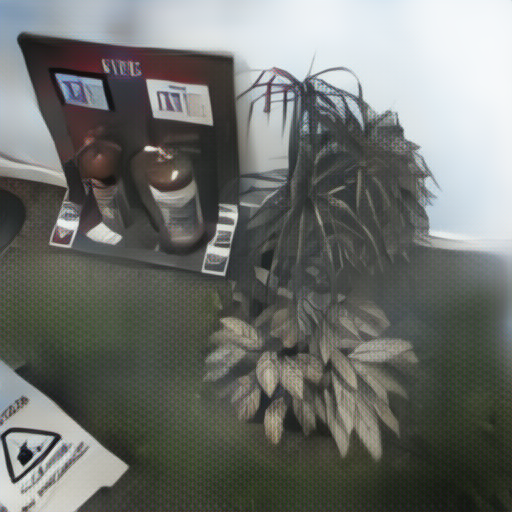}
    }
    \subfloat[][Random Lines~\cite{speciale2019queries}\figlabel{SP_k100_ulc_inv2d}]{
        \includegraphics[width=0.19\linewidth]{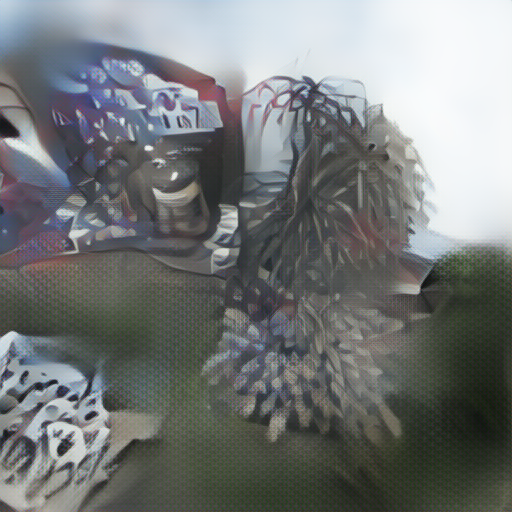}
    }
    \subfloat[][Coord. Perm.~\cite{pan2023permut}\figlabel{SP_k100_cp_inv2d}]{
        \includegraphics[width=0.19\linewidth]{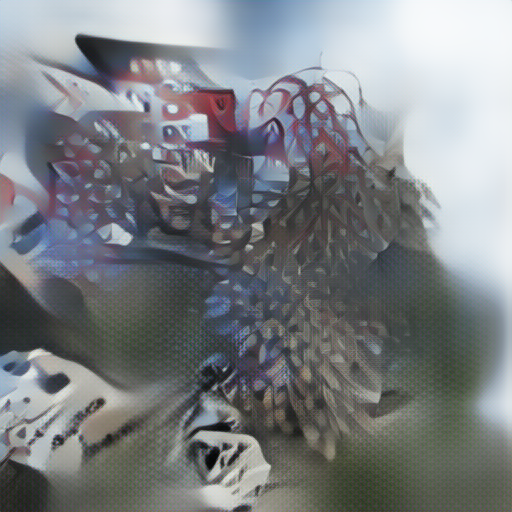}
    }
    \subfloat[][\textbf{DCL (ours)}\figlabel{SP_k100_dcl_inv2d}]{
        \includegraphics[width=0.19\linewidth]{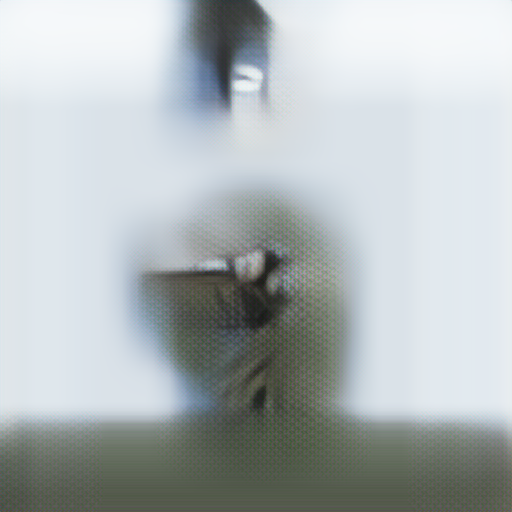}
    }

    \vspace{-2mm}
    \caption{Qualitative inversion results using \textbf{SuperPoint} features against the geometry-recovery attack~\cite{chelani2024obfuscation} in an \textbf{expanded neighborhood} ($K=100$) setting. The reconstructed images are organized by dataset: Row 1 (Aachen), Rows 2--3 (Cambridge), and Rows 4--5 (7Scenes). Feature positions are recovered via various methods: (a) Feature Points, (b) Random Lines~\cite{speciale2019queries}, (c) Coordinate Permutation~\cite{pan2023permut}, and (d) Dual Convergent Lines (ours).}

    \figlabel{SP_k100_feat_inv}
    \label{fig:SP_K100_privacy_qualitative}
\end{figure*}


\begin{figure*}[t]
\centering

    \setcounter{subfigure}{0}
    \subfloat{
        \includegraphics[width=0.19\linewidth]{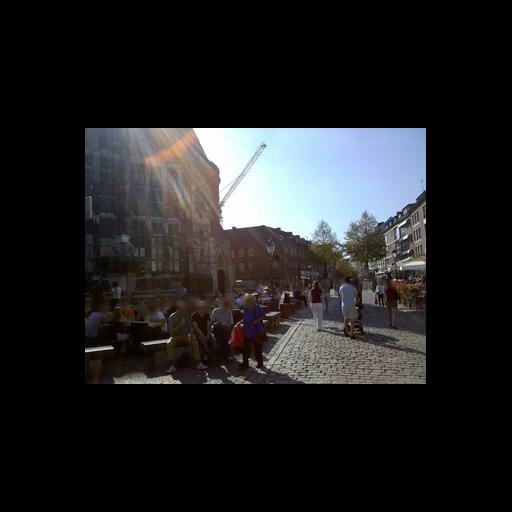}
    }
    \subfloat{
        \includegraphics[width=0.19\linewidth]{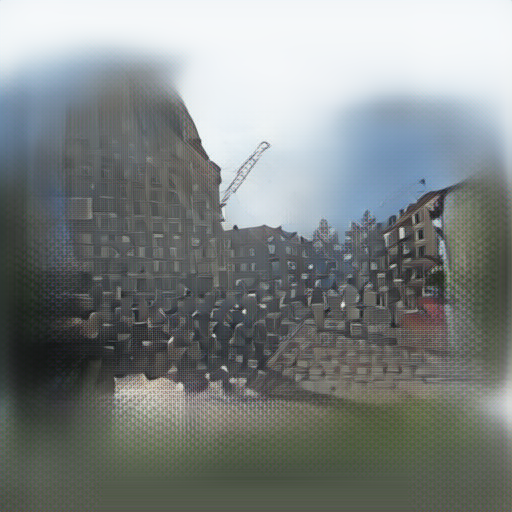}
    }
    \subfloat{
        \includegraphics[width=0.19\linewidth]{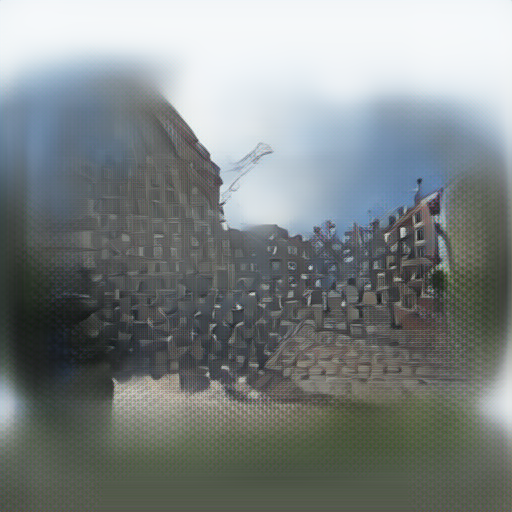}
    }
    \subfloat{
        \includegraphics[width=0.19\linewidth]{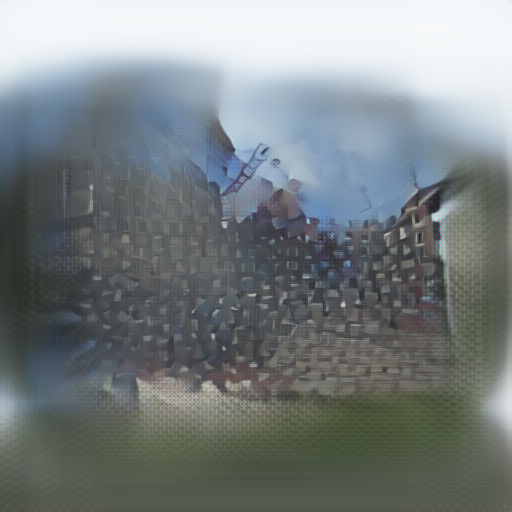}
    }
    \subfloat{
        \includegraphics[width=0.19\linewidth]{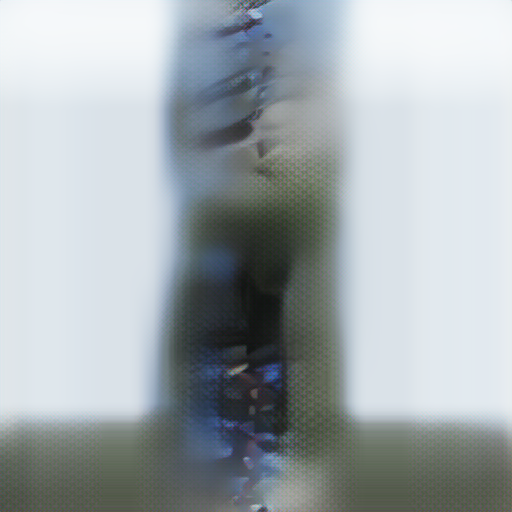}
    }
    \\
    \vspace{1mm}
    
    \setcounter{subfigure}{0}
    \subfloat{
        \includegraphics[width=0.19\linewidth]{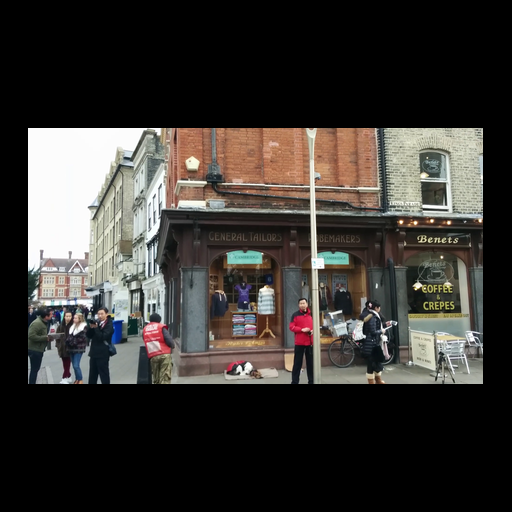}
    }
    \subfloat{
        \includegraphics[width=0.19\linewidth]{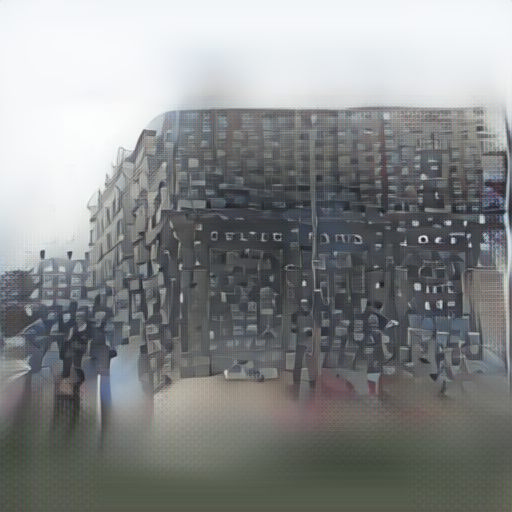}
    }
    \subfloat{
        \includegraphics[width=0.19\linewidth]{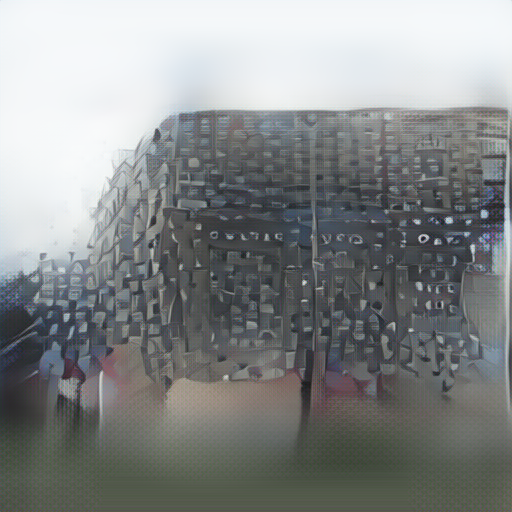}
    }
    \subfloat{
        \includegraphics[width=0.19\linewidth]{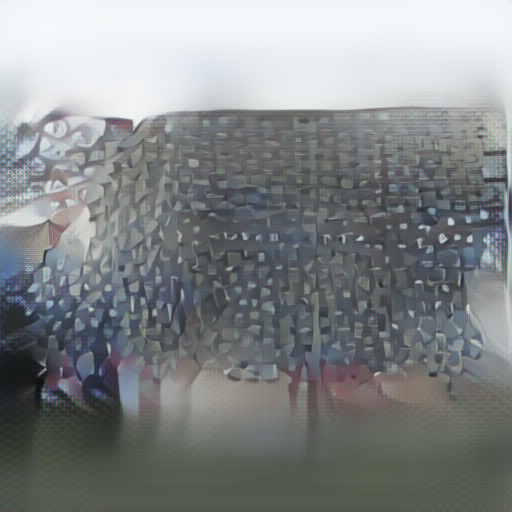}
    }
    \subfloat{
        \includegraphics[width=0.19\linewidth]{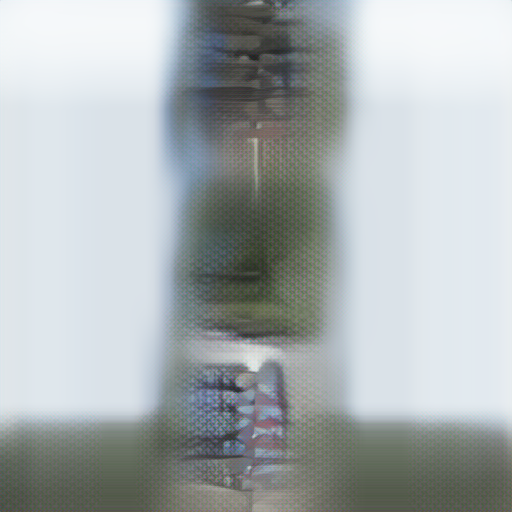}
    }
    \\
    \vspace{1mm}
    
    \setcounter{subfigure}{0}
    \subfloat{
        \includegraphics[width=0.19\linewidth]{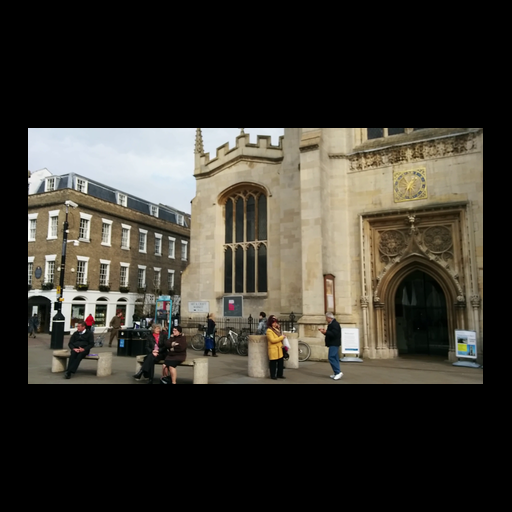}
    }
    \subfloat{
        \includegraphics[width=0.19\linewidth]{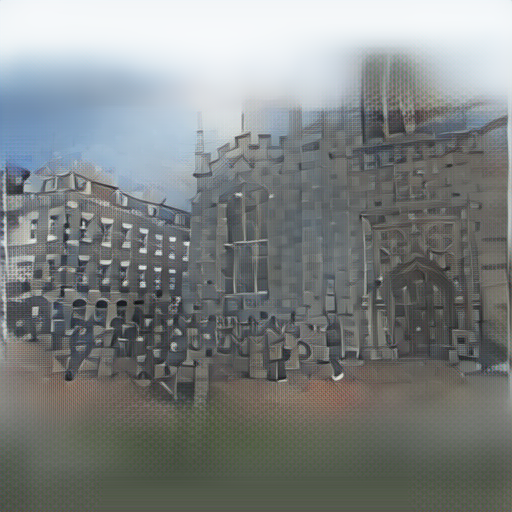}
    }
    \subfloat{
        \includegraphics[width=0.19\linewidth]{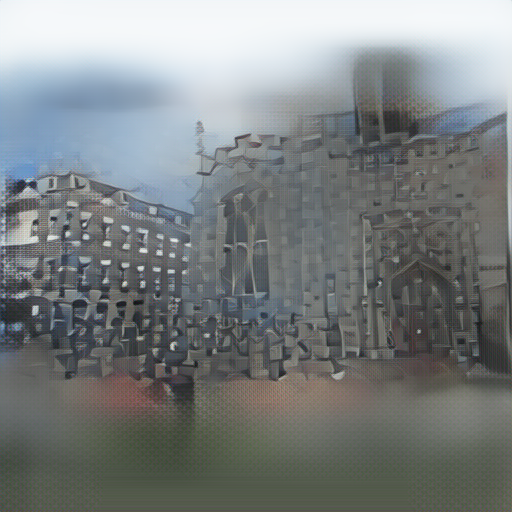}
    }
    \subfloat{
        \includegraphics[width=0.19\linewidth]{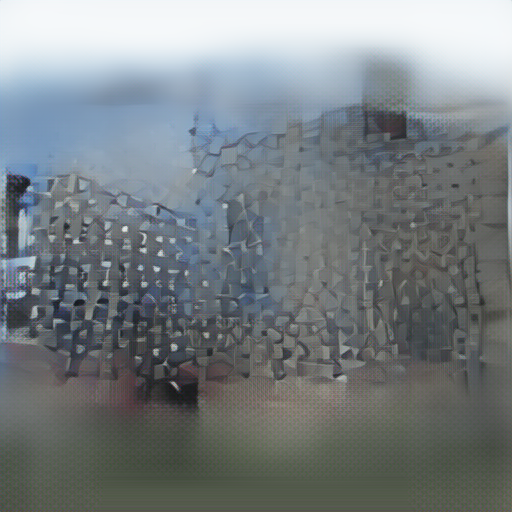}
    }
    \subfloat{
        \includegraphics[width=0.19\linewidth]{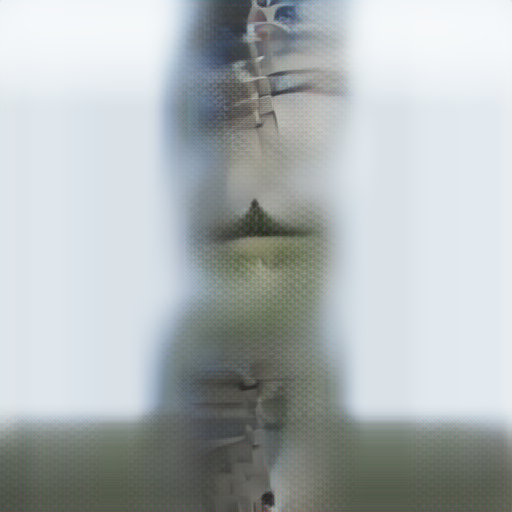}
    }
    \\
    \vspace{1mm}
    
    \setcounter{subfigure}{0}
    \subfloat{
        \includegraphics[width=0.19\linewidth]{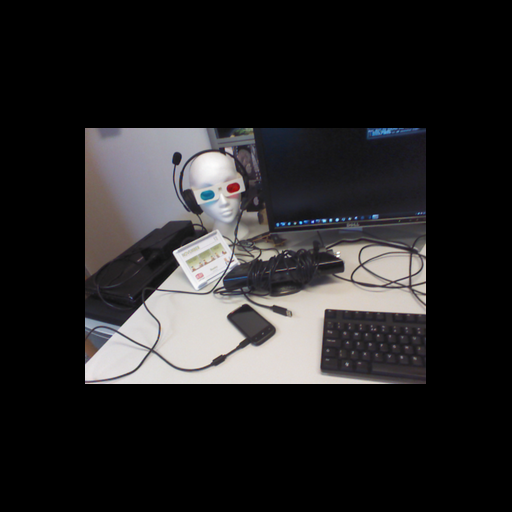}
    }
    \subfloat{
        \includegraphics[width=0.19\linewidth]{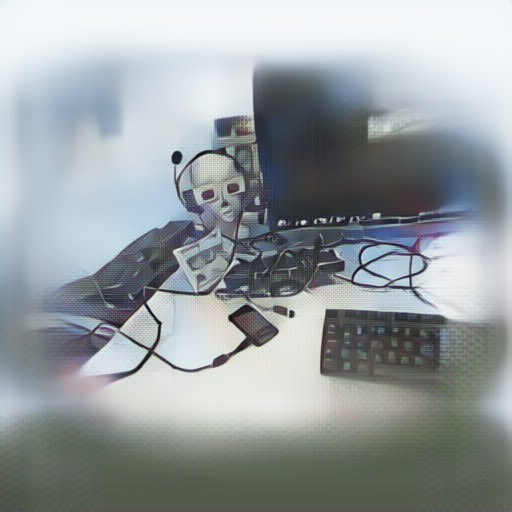}
    }
    \subfloat{
        \includegraphics[width=0.19\linewidth]{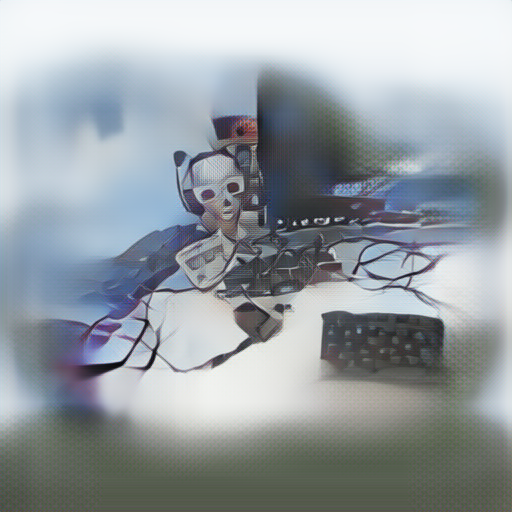}
    }
    \subfloat{
        \includegraphics[width=0.19\linewidth]{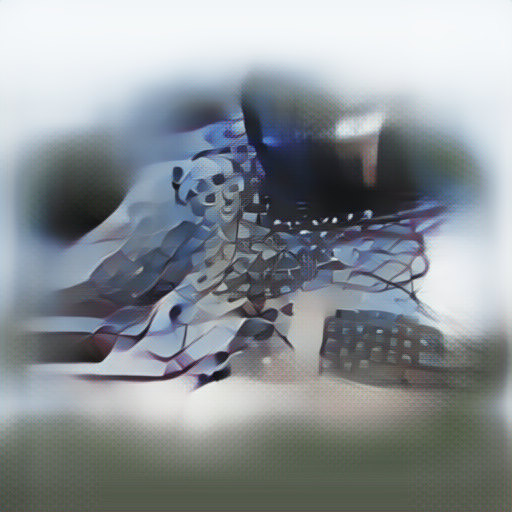}
    }
    \subfloat{
        \includegraphics[width=0.19\linewidth]{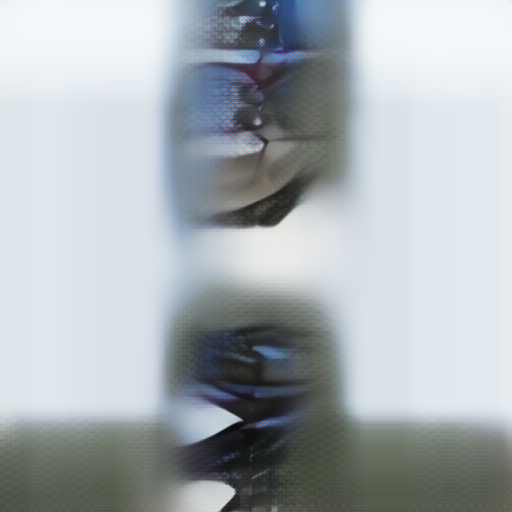}
    }
    \\
    \vspace{1mm}

    \setcounter{subfigure}{0}
    \subfloat[][Original Image]{
        \includegraphics[width=0.19\linewidth]{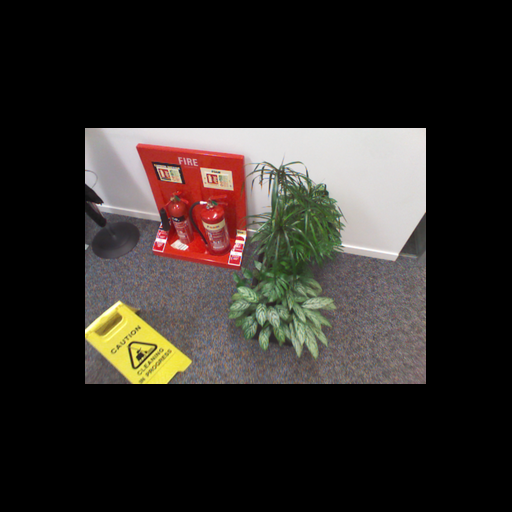}
    }
    \subfloat[][Feature Points~\cite{lowe2004sift}]{
        \includegraphics[width=0.19\linewidth]{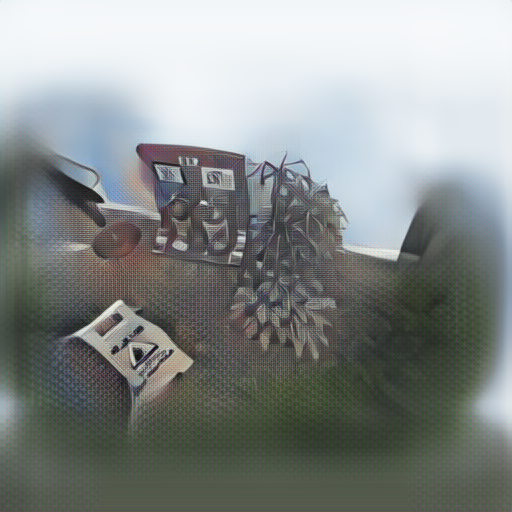}
    }
    \subfloat[][Random Lines~\cite{speciale2019queries}\figlabel{SP_zoomout_ulc_inv2d}]{
        \includegraphics[width=0.19\linewidth]{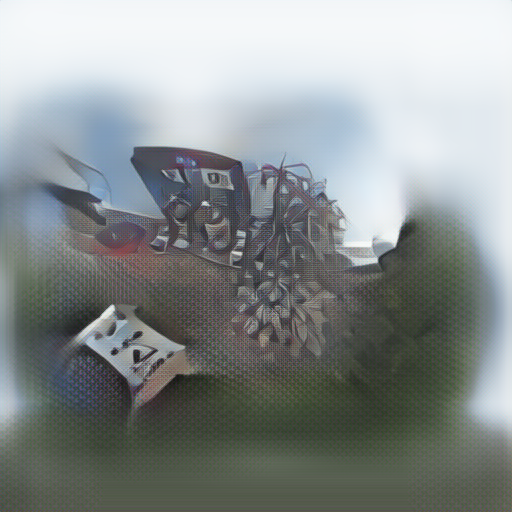}
    }
    \subfloat[][Coord. Perm.~\cite{pan2023permut}\figlabel{SP_zoomout_cp_inv2d}]{
        \includegraphics[width=0.19\linewidth]{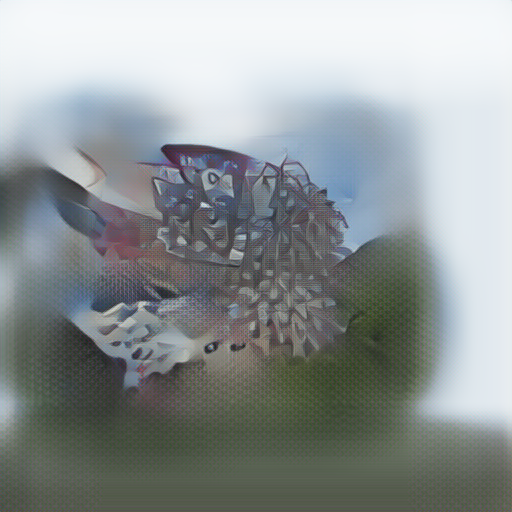}
    }
    \subfloat[][\textbf{DCL (ours)}\figlabel{SP_zoomout_dcl_inv2d}]{
        \includegraphics[width=0.19\linewidth]{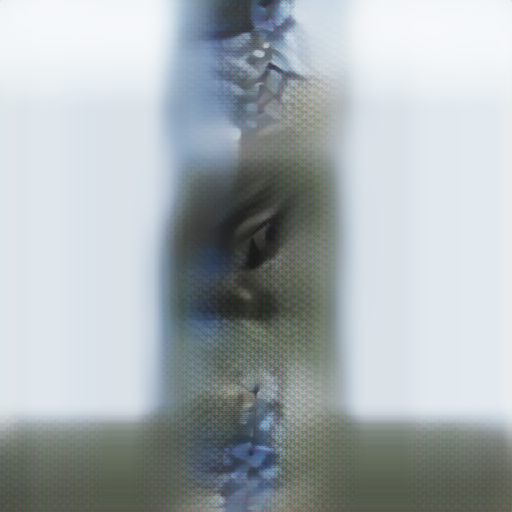}
    }

    \vspace{-2mm}
    \caption{Visualization of the inversion network's output using \textbf{SuperPoint} features after the geometry-recovery attack~\cite{chelani2024obfuscation}. The results are shown at the K=20 neighborhood setting in a \textbf{$\times 2$ zoomed-out view} because most recovered points fall outside the original image boundary. The reconstructed images are organized by dataset: Row 1 (Aachen), Rows 2--3 (Cambridge), and Rows 4--5 (7Scenes). This visualization provides clear empirical evidence of the effect of recovered keypoint displacement.}

    \figlabel{SP_zoomout_feat_inv}
    \label{fig:SP_zoom_out_qualitative_results}
\end{figure*}


\begin{figure*}[t]
\centering

    \setcounter{subfigure}{0}
    \subfloat{
        \includegraphics[width=0.19\linewidth]{figs/supple/inv_src/day_milestone_2011-10-01_14-01-18_320.png}
    }
    \subfloat{
        \includegraphics[width=0.19\linewidth]{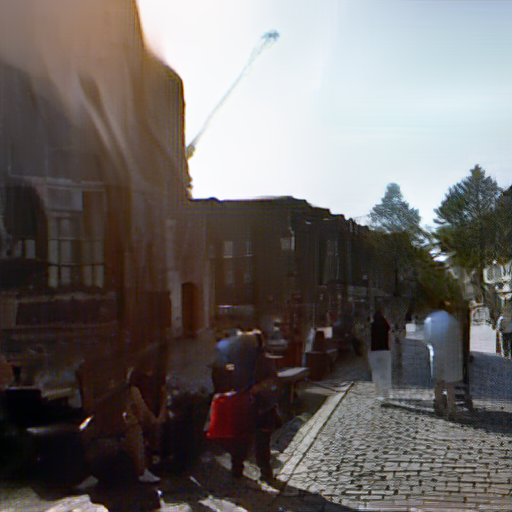}
    }
    \subfloat{
        \includegraphics[width=0.19\linewidth]{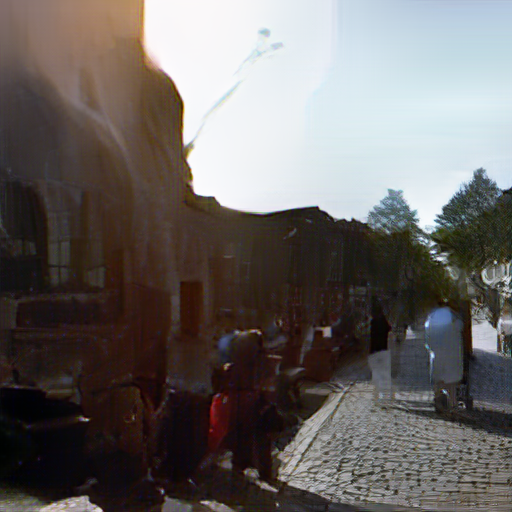}
    }
    \subfloat{
        \includegraphics[width=0.19\linewidth]{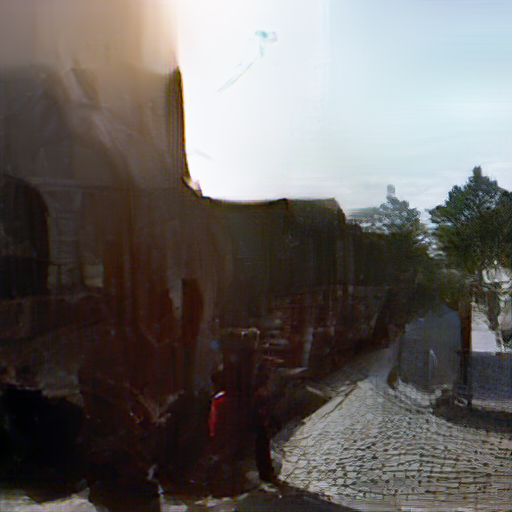}
    }
    \subfloat{
        \includegraphics[width=0.19\linewidth]{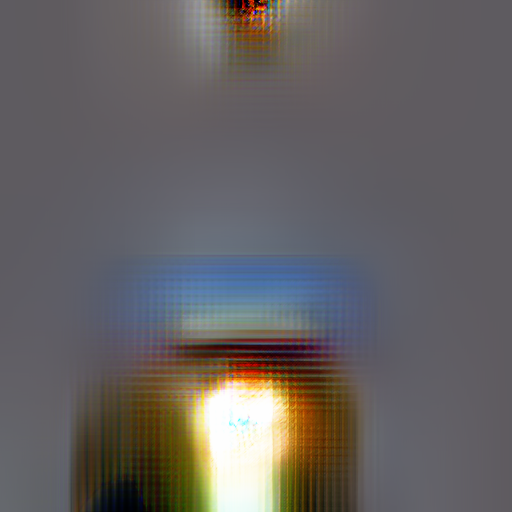}
    }
    \\
    \vspace{1mm}
    
    \setcounter{subfigure}{0}
    \subfloat{
        \includegraphics[width=0.19\linewidth]{figs/supple/inv_src/seq1_frame00005.png}
    }
    \subfloat{
        \includegraphics[width=0.19\linewidth]{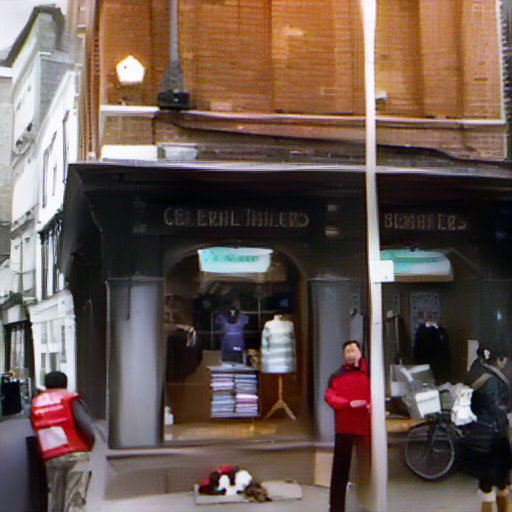}
    }
    \subfloat{
        \includegraphics[width=0.19\linewidth]{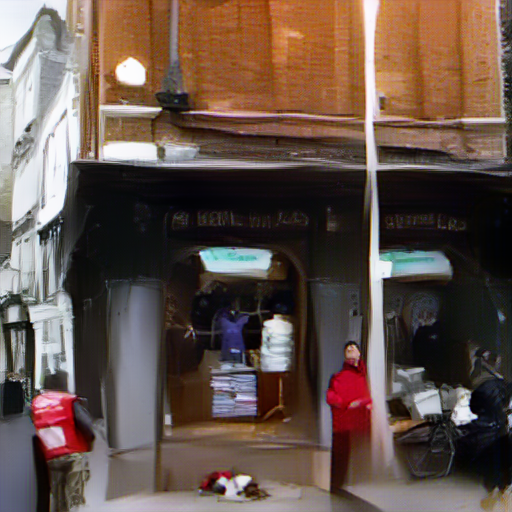}
    }
    \subfloat{
        \includegraphics[width=0.19\linewidth]{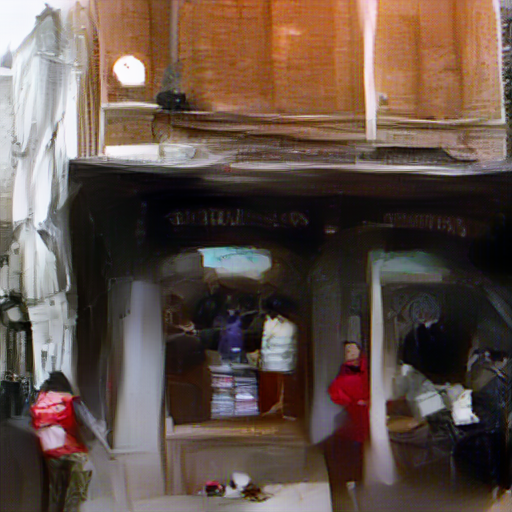}
    }
    \subfloat{
        \includegraphics[width=0.19\linewidth]{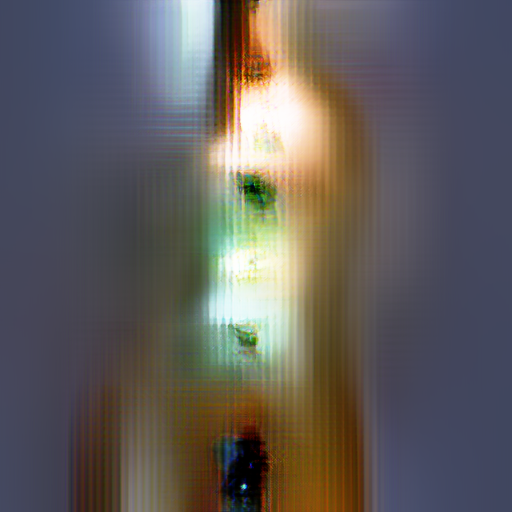}
    }
    \\
    \vspace{1mm}
    
    \setcounter{subfigure}{0}
    \subfloat{
        \includegraphics[width=0.19\linewidth]{figs/supple/inv_src/seq3_frame00050.png}
    }
    \subfloat{
        \includegraphics[width=0.19\linewidth]{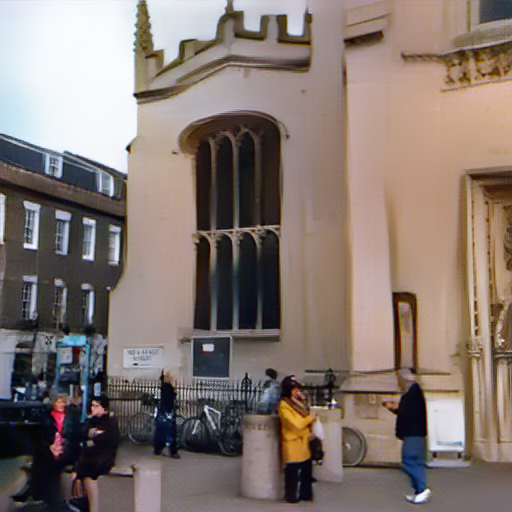}
    }
    \subfloat{
        \includegraphics[width=0.19\linewidth]{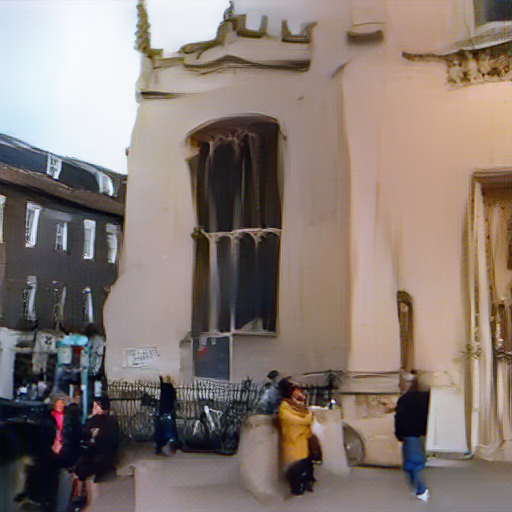}
    }
    \subfloat{
        \includegraphics[width=0.19\linewidth]{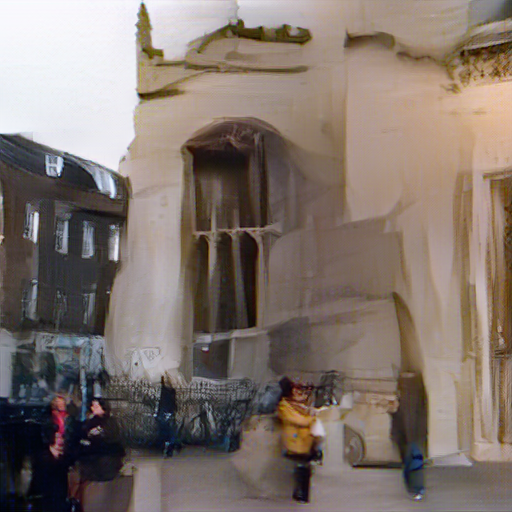}
    }
    \subfloat{
        \includegraphics[width=0.19\linewidth]{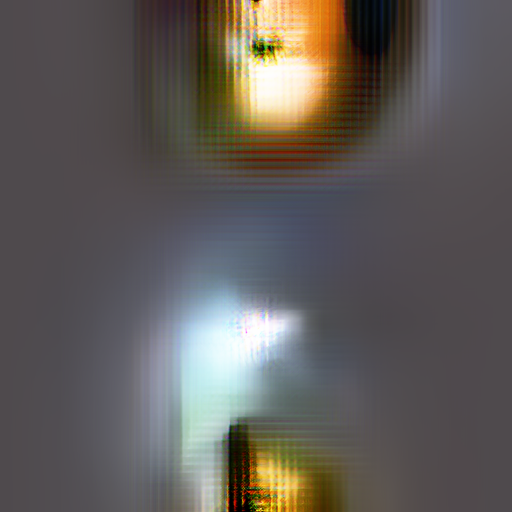}
    }
    \\
    \vspace{1mm}
    
    \setcounter{subfigure}{0}
    \subfloat{
        \includegraphics[width=0.19\linewidth]{figs/supple/inv_src/seq-01_frame-000003.color.png}
    }
    \subfloat{
        \includegraphics[width=0.19\linewidth]{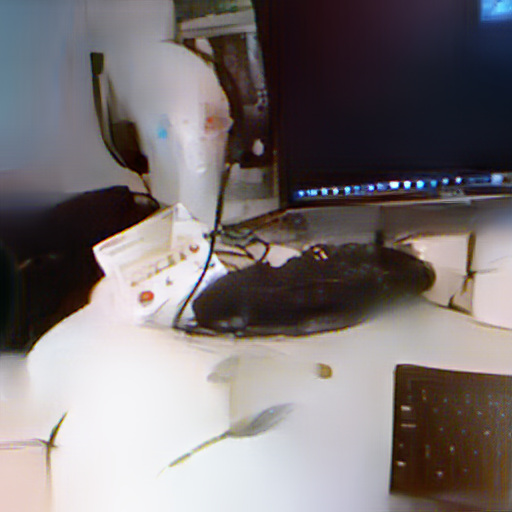}
    }
    \subfloat{
        \includegraphics[width=0.19\linewidth]{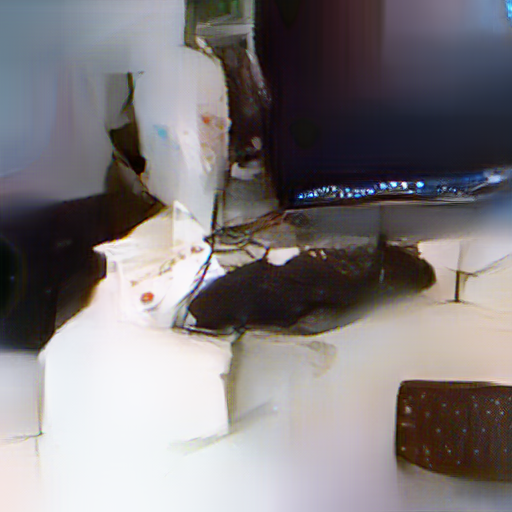}
    }
    \subfloat{
        \includegraphics[width=0.19\linewidth]{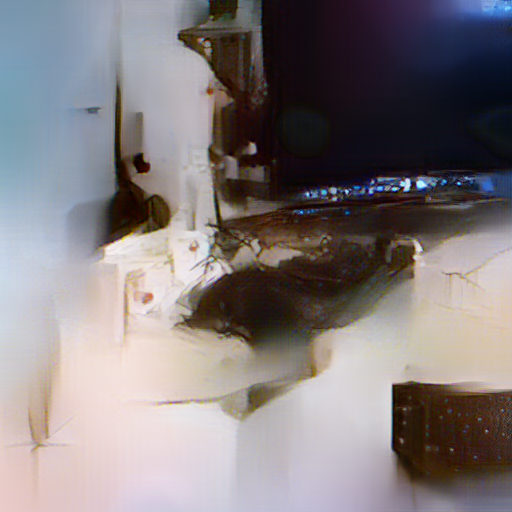}
    }
    \subfloat{
        \includegraphics[width=0.19\linewidth]{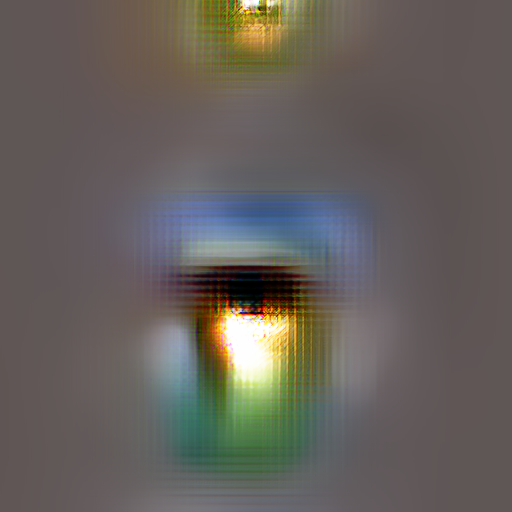}
    }
    \\
    \vspace{1mm}

    \setcounter{subfigure}{0}
    \subfloat[][Original Image]{
        \includegraphics[width=0.19\linewidth]{figs/supple/inv_src/seq-03_frame-000573.color.png}
    }
    \subfloat[][Feature Points~\cite{lowe2004sift}]{
        \includegraphics[width=0.19\linewidth]{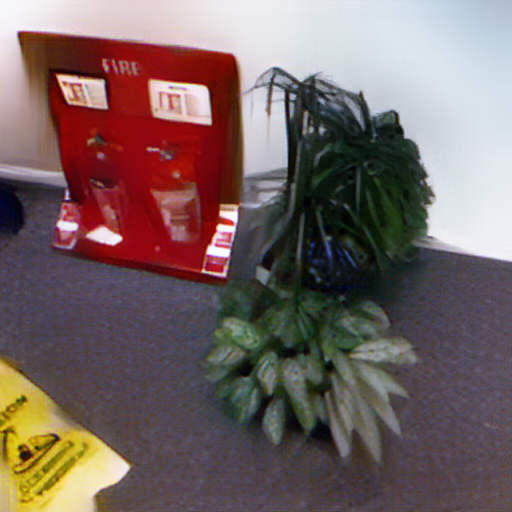}
    }
    \subfloat[][Random Lines~\cite{speciale2019queries}\figlabel{sift_ulc_inv2d}]{
        \includegraphics[width=0.19\linewidth]{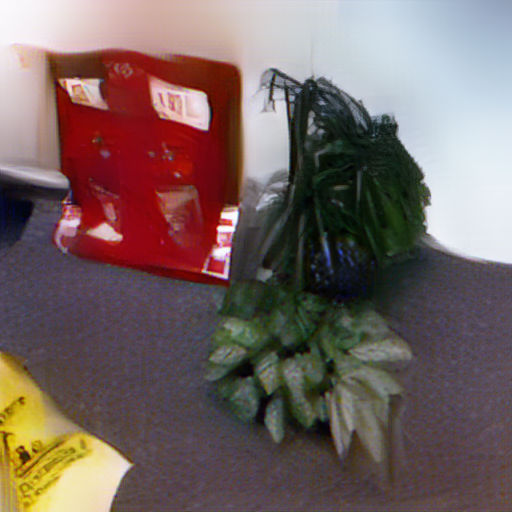}
    }
    \subfloat[][Coord. Perm.~\cite{pan2023permut}\figlabel{sift_cp_inv2d}]{
        \includegraphics[width=0.19\linewidth]{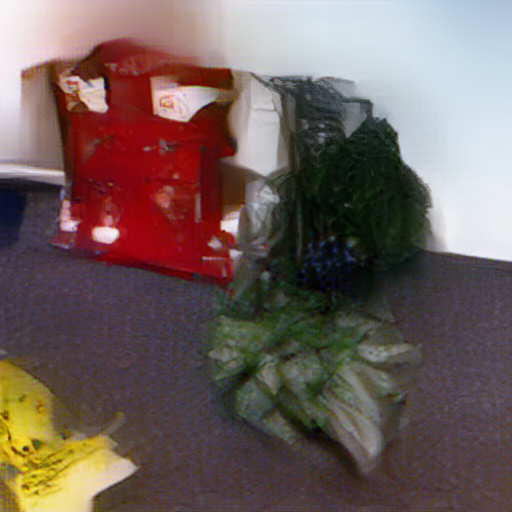}
    }
    \subfloat[][\textbf{DCL (ours)}\figlabel{sift_dcl_inv2d}]{
        \includegraphics[width=0.19\linewidth]{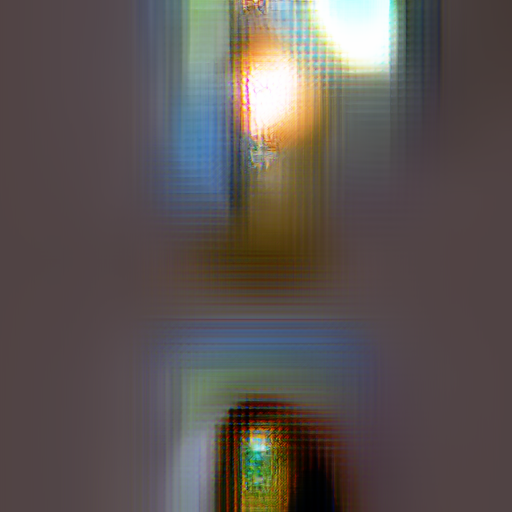}
    }

    \vspace{-2mm}
    \caption{Inversion results generated by the \textbf{InvSfM}~\cite{pittaluga2019revealing} model using \textbf{SIFT} features, conducted in the \textbf{oracle setting} with $K=20$. The reconstructed images are organized by dataset: Row 1 (Aachen), Rows 2--3 (Cambridge), and Rows 4--5 (7Scenes). Feature positions are recovered via various methods: (a) Feature Points, (b) Random Lines~\cite{speciale2019queries}, (c) Coordinate Permutation~\cite{pan2023permut}, and (d) Dual Convergent Lines (ours).}
    
    \figlabel{sift_feat_inv}
    \label{fig:sift_privacy_qualitative}
\end{figure*}


\end{document}